\documentclass{article}

\usepackage{microtype}
\usepackage[final]{graphicx}
\usepackage{subfigure}
\usepackage{booktabs} 
\usepackage{enumitem}
\usepackage{caption}
\usepackage{diagbox}

\usepackage{hyperref}

\usepackage[accepted]{icml2024}

\usepackage{amsmath}
\usepackage{amssymb}
\usepackage{mathtools}
\usepackage{amsthm}
\usepackage{xspace}
\usepackage[]{xcolor}
\usepackage{thm-restate}
\usepackage{subcaption}
\usepackage[capitalize,noabbrev]{cleveref}

\theoremstyle{plain}
\newtheorem{theorem}{Theorem}[section]

\newtheorem{lemma}[theorem]{Lemma}

\theoremstyle{definition}

\theoremstyle{remark}

\DeclareMathOperator*{\argmin}{arg\,min}

\usepackage[textsize=tiny]{todonotes}

\icmltitlerunning{Relational DNN Verification With Cross Executional Bound Refinement}

\begin{document}

\twocolumn[
\icmltitle{Relational DNN Verification With Cross Executional Bound Refinement}



\icmlsetsymbol{equal}{*}

\begin{icmlauthorlist}
\icmlauthor{Debangshu Banerjee}{yyy}
\icmlauthor{Gagandeep Singh}{yyy,comp}
\end{icmlauthorlist}

\icmlaffiliation{yyy}{Department of Computer Science, University of Illinois Urbana-Champaign, USA}
\icmlaffiliation{comp}{VMware Research, USA}

\icmlcorrespondingauthor{Debangshu Banerjee}{db21@illinois.edu}


\icmlkeywords{Machine Learning, ICML}

\vskip 0.3in]



\printAffiliationsAndNotice{\icmlEqualContribution} 

\newcommand{\Tool}{RACoon\xspace}
\newcommand{\Relu}{ReLU}
\newcommand{\inspecset}{\Phi}
\newcommand{\crossinspecset}{\Phi^{\delta}}
\newcommand{\crossinspec}{\phi^{\delta}_{i,j}}
\newcommand{\inspec}{\phi_{in}}
\newcommand{\inprop}{\phi}
\newcommand{\outprop}{\psi}
\newcommand{\outpropparam}[1]{\psi^{#1}}
\newcommand{\inreg}{\phi_{t}}
\newcommand{\inregparam}[1]{\phi^{#1}_{t}}
\newcommand{\outspecreg}{\psi_{t}}
\newcommand{\inspecsetreg}{\Phi_{t}}
\newcommand{\outspecset}{\Psi}
\newcommand{\acti}[1]{\sigma(#1)}
\newcommand{\actilower}[1]{\sigma_l(#1)}
\newcommand{\actiupper}[1]{\sigma_u(#1)}
\newcommand{\netlower}[1]{N_l(#1)}
\newcommand{\netupper}[1]{N_u(#1)}
\newcommand{\vect}[1]{\mathbf{#1}}
\newcommand{\linapprx}{\vect{L}}
\newcommand{\linapprxparam}[1]{\vect{L}_{#1}}
\newcommand{\linapprxalpha}[1]{\vect{L}_{#1}(\pmb{\alpha}_{#1})}
\newcommand{\linapprxalphaini}[1]{\vect{L}_{#1}(\pmb{\alpha}^{0}_{#1})}
\newcommand{\biasalpha}[1]{\vect{b}_{#1}(\pmb{\alpha}_{#1})}
\newcommand{\biasalphaini}[1]{\vect{b}_{#1}(\pmb{\alpha}^{0}_{#1})}
\newcommand{\constalpha}[1]{a_{#1}(\pmb{\alpha}_{#1})}
\newcommand{\param}[1]{\pmb{\alpha}_{#1}}
\newcommand{\paramini}[1]{\pmb{\alpha}^{0}_{#1}}
\newcommand{\dual}{G}
\newcommand{\naivedual}{\overline{G}}
\newcommand{\opt}[1]{#1^{*}}
\newcommand{\optappx}[1]{#1^{*}_{appx}}
\newcommand{\optappxparam}[2]{#1^{{appx}}_{#2}}
\newcommand{\optparams}{\pmb{\lambda}}
\newcommand{\Output}{\vect{M}_{0}}
\newcommand{\apprxOutput}{\vect{M}}
\newcommand{\linapprxoptalpha}[1]{\vect{L}_{#1}(\pmb{\opt{\alpha}}_{#1})}
\newcommand{\biasoptalpha}[1]{\vect{b}_{#1}(\pmb{\opt{\alpha}}_{#1})}
\newcommand{\linapprxoptalphamulti}[2]{\vect{L}_{#1}(\pmb{\opt{\alpha}}_{#2})}
\newcommand{\biasoptalphamulti}[2]{\vect{b}_{#1}(\pmb{\opt{\alpha}}_{#2})}
\newcommand{\constoptalpha}[1]{a_{#1}(\pmb{\opt{\alpha}}_{#1})}
\newcommand{\ol}[1]{\opt{#1}}
\definecolor{mgreen}{rgb}{0.0, 0.7, 0.1}
\newcommand{\linearmap}{\mathcal{L}}
\newcommand{\lirpa}{auto\_LiRPA\xspace}

\begin{abstract}
We focus on verifying relational properties defined over deep neural networks (DNNs) such as robustness against universal adversarial perturbations (UAP), certified worst-case hamming distance for binary string classifications, etc. Precise verification of these properties requires reasoning about multiple executions of the same DNN. However, most of the existing works in DNN verification only handle properties defined over single executions and as a result, are imprecise for relational properties. Though few recent works for relational DNN verification, capture linear dependencies between the inputs of multiple executions, they do not leverage dependencies between the outputs of hidden layers producing imprecise results. We develop a scalable relational verifier \Tool that utilizes cross-execution dependencies at all layers of the DNN gaining substantial precision over SOTA baselines on a wide range of datasets, networks, and relational properties. 
\end{abstract}
\section{Introduction}
\label{sec:intro}
Deep neural networks (DNNs) have gained widespread prominence across various domains, including safety-critical areas like autonomous driving~\cite{bojarski2016end_DUP} or medical diagnosis~\cite{AMATO201347}, etc. Especially in these domains, the decisions made by these DNNs hold significant importance, where errors can lead to severe consequences. However, due to the black-box nature and highly nonlinear behavior of DNNs, reasoning about them is challenging. Despite notable efforts in identifying and mitigating DNN vulnerabilities \cite{adversarial,pgdTraining, uap, potdevin2019empirical, wu2023toward, sotoudeh2020abstract}, these methods cannot guarantee safety. Consequently, significant research has been dedicated to formally verifying the safety properties of DNNs.
\noindent Despite advancements, current DNN verification techniques can not handle relational properties prevalent in practical scenarios. Most of the existing efforts focus on verifying the absence of input-specific adversarial examples within the local neighborhood of test inputs. 
However, recent studies \cite{music:19} highlight the impracticality of attacks targeting individual inputs. In practical attack scenarios \cite{rafa23, sticker:19, music:19}, there is a trend towards developing universal adversarial perturbations (UAPs) \cite{uap} designed to affect a significant portion of inputs from the training distribution.
Since the same adversarial perturbation is applied to multiple inputs, the executions on different perturbed inputs are related, and exploiting the relationship between different executions is important for designing precise relational verifiers. Existing DNN verifiers working on individual executions lack these capabilities and as a result, lose precision. Beyond UAP verification, other relevant relational properties include measuring the worst-case hamming distance for binary string classification and bounding the worst-case absolute difference between the original number and the number classified using a digit classifier where inputs perturbed with common perturbation \cite{nonLinSpec}.


\noindent\textbf{Key challenges:} For precise relational verification, we need scalable algorithms to track the relationship between DNN's outputs across multiple executions. Although it is possible to exactly encode DNN executions with piecewise linear activation functions (e.g. \Relu) over input regions specified by linear inequalities as MILP (Mixed Integer Linear Program), the corresponding MILP optimization problem is computationally expensive. For example, MILP encoding of $k$ executions of a DNN with $n_{r}$ ReLU activations in the worst case introduces $O(n_{r} \times k)$ integer variables. Considering the cost of MILP optimization grows exponentially with the number of integer variables, even verifying small DNNs w.r.t a relational property defined over $k$ execution with MILP is practically infeasible. For scalability, \cite{certifair} completely ignores the dependencies across executions and reduces relational verification over $k$ executions into $k$ individual verification problems solving them independently. SOTA relational verifier \cite{iclrUap} first obtains provably correct linear approximations of the DNN with existing non-relational verifier \cite{auto_lirpa} without tracking any cross-execution dependencies then adds linear constraints at the input layer capturing linear dependencies between inputs used in different executions. In this case, ignoring cross-execution dependencies while computing provably correct linear approximations of the DNN for each execution leads to the loss of precision (as confirmed by our experiments in Section~\ref{sec:expEval}). This necessitates developing scalable algorithms for obtaining precise approximations of DNN outputs over multiple executions that benefit from cross-execution dependencies.

\textbf{Our contributions: } We make the following contributions to improve the precision of relational DNN verification:
\begin{itemize}[noitemsep, nolistsep,leftmargin=*]
\item In contrast to the SOTA baselines, we compute a provably correct parametric linear approximation of the DNN for each execution using parametric bounds of activation functions (e.g. \Relu) as done in existing works \cite{alphaCrown, convex_barr}. Instead of learning the parameters for each execution independently as done in \cite{alphaCrown}, we refine the parametric bounds corresponding to multiple executions together. In this case, the bound refinement at the hidden layer takes into account the cross-execution dependencies so that the learned bounds are tailored for verifying the specific \textit{relational} property.
\item For scalable cross-executional bound refinement, we (a) formulate a linear programming-based relaxation of the relational property, (b) find a provably correct differentiable closed form of the corresponding Dual function that preserves dependencies between parameters from different executions while being suitable for scalable differentiable optimization techniques, (c) using the differentiable closed form refine the parametric bound with scalable differential optimization methods (e.g. gradient descent).
\item We develop \Tool (\textbf{R}elational DNN \textbf{A}nalyzer with \textbf{C}ross-Excuti\textbf{o}nal B\textbf{o}und Refi\textbf{n}ement) that formulates efficiently optimizable MILP instance with cross-executional bound refinement for precise relational verification.
\item We perform extensive experiments on popular datasets, multiple DNNs (standard and robustly trained), and multiple relational properties showcasing that \Tool significantly outperforms the current SOTA baseline.\footnote{Code at  https://github.com/Debangshu-Banerjee/RACoon} 
\end{itemize}

\vspace{-5mm}
\section{Related Works}
\vspace{-2mm}
\textbf{Non-relational DNN verifiers:}
DNN verifiers are broadly categorized into three main categories - (i) sound but incomplete verifiers which may not always prove property even if it holds~\cite{gehr2018ai2,DeepZ, deeppoly, singh2019beyond, crown, auto_lirpa, alphaCrown}, (ii) complete verifiers that can always prove the property if it holds ~\cite{wang2018neurify, gehr2018ai2, bunel2020branch, bunel2020efficient, DBLP:conf/cav/BakTHJ20, ehlers2017formal, ferrari2022complete, fromherz2021fast, wang2021beta, depalma2021scaling, anderson2020strong, zhang2022general} and (iii) verifiers with probabilistic guarantees \cite{randSmooth, lini_random}. 

\textbf{Relational DNN verifier: } Existing DNN relational verifiers can be grouped into two main categories - (i) verifiers for properties (UAP, fairness, etc.) defined over multiple executions of the same DNN, \cite{iclrUap, certifair}, (ii) verifiers for properties (local DNN equivalence \cite{reluDiff}) defined over multiple executions of different DNNs on the same input \cite{reluDiff, neuroDiff}.
For relational properties defined over multiple executions of the same DNN the existing verifiers \cite{certifair} reduce the verification problem into $L_{\infty}$ robustness problem by constructing product DNN with multiple copies of the same DNN. However, the relational verifier in \cite{certifair} treats all $k$ executions of the DNN as independent and loses precision as a result of this. The SOTA DNN relational verifier \cite{iclrUap} (referred to as I/O formulation in the rest of the paper) although tracks the relationship between inputs used in multiple executions at the input layer, does not track the relationship between the inputs fed to the subsequent hidden layers and can only achieve a limited improvement over the baseline verifiers that treat all executions independently as shown in our experiments. There exist, probabilistic verifiers, \cite{prob_relational1, prob_relational2} based on randomized smoothing \cite{randSmooth} for verifying relational properties. However, these works can only give probabilistic guarantees on smoothed models which have high inference costs.
Similar to \cite{certifair, iclrUap}, in this work, we focus on deterministic scalable incomplete relational verifiers that can serve as a building block for BaB (Branch and Bound) based complete verifiers \cite{wang2021beta} with popular branching strategies like input splitting \cite{anderson2020strong}, ReLU splitting \cite{wang2021beta}, etc.
We leave combining \Tool with branching strategies as future work.
In this work, we consider DNNs with ReLU activation.
\vspace{-4mm}
\section{Preliminaries}
\label{sec:prelims}
\vspace{-2mm}
We provide the necessary background on approaches for non-relational DNN verification, DNN safety properties that can be encoded as relational properties, and existing works on parametric bound refinement for individual executions.

\textbf{Non-relational DNN verification: } 
For individual execution, DNN verification involves proving that the network outputs $\vect{y}=N(\vect{x} + \pmb{\delta})$ corresponding to all perturbations $\vect{x} + \pmb{\delta}$ of an input $\vect{x}$ specified by $\inprop$, satisfy a logical specification $\outprop$. For common safety properties like local DNN robustness, the output specification ($\outprop$) is expressed as linear inequality (or conjunction of linear inequalities) over DNN output $\vect{y} \in \mathbb{R}^{n_l}$. e.g.  $\outprop(\vect{y}) = (\vect{c}^{T} \vect{y} \geq 0)$ where $\vect{c} \in \mathbb{R}^{n_{l}}$. In general, given a DNN $N: \mathbb{R}^{n_0} \rightarrow \mathbb{R}^{n_l}$ and a property specified by $(\inprop, \outprop)$, scalable sound but incomplete verifiers compute a linear approximation specified by $\linapprx \in \mathbb{R}^{n_0}, b \in \mathbb{R}$ such that for any input $\vect{x} \in \inreg \subseteq \mathbb{R}^{n_0}$ satisfying $\phi$ the following condition holds $\linapprx^T \vect{x} + b \leq \vect{c}^TN(\vect{x})$. To show $\vect{c}^TN(\vect{x}) \geq 0 $ for all $\vect{x} \in \inreg$ DNN verifiers prove for all $\vect{x} \in \inreg$,  $\linapprx^T \vect{x} + b \geq 0$ holds.

\textbf{DNN relational properties: } For a DNN $N: \mathbb{R}^{n_0} \rightarrow \mathbb{R}^{n_l}$, relational properties defined over $k$ executions of $N$ are specified by the tuple $(\inspecset, \outspecset)$ where the input specification $\inspecset: \mathbb{R}^{n_0 \times k} \rightarrow \{true, false\}$ encodes the input region $\Phi_{t} \subseteq \mathbb{R}^{n_0 \times k}$ encompassing all potential inputs corresponding to each of the $k$ executions of $N$ and the output specification $\outspecset: \mathbb{R}^{n_l \times k} \rightarrow \{true, false\}$ specifies the safety property we expect the outputs of all $k$ executions of $N$ to satisfy. Formally, in DNN relational verification, given $N$, an input specification $\inspecset$ and an output specification $\outspecset$ we require to prove whether $\forall \vect{x^{*}_1}, \dots, \vect{x^{*}_k} \in \mathbb{R}^{n_0}. \inspecset(\vect{x^{*}_1}, \dots, \vect{x^{*}_k}) \implies \outspecset(N(\vect{x^{*}_1}), \dots N(\vect{x^{*}_k})) $ or provide a counterexample otherwise. Here, $\vect{x^{*}_1}, \dots, \vect{x^{*}_k}$ are the inputs to the $k$ executions of $N$ and $N(\vect{x^{*}_1}), \dots, N(\vect{x^{*}_k})$ are the corresponding outputs. 
Commonly, the input region $\inregparam{i}$ for the $i$-th execution is a $L_{\infty}$ region around a fixed point $\vect{x_{i}} \in \mathbb{R}^{n_0}$ defined as $\inregparam{i} = \{\vect{x^{*}_{i}} \in \mathbb{R}^{n_0}\;|\;\|\vect{x^{*}_{i}} - \vect{x_{i}}\|_{\infty} \leq \epsilon\}$ while the corresponding output specification $\outpropparam{i}(N(\vect{x^{*}_i})) = \bigwedge_{j=1}^{m}(\vect{c_{i,j}}^{T}N(\vect{x^{*}_i}) \geq 0)$. Subsequently, 
$\inspecset(\vect{x^{*}_1,\dots, \vect{x^{*}_k}}) = \bigwedge_{i=1}^k (\vect{x^{*}_i} \in \inreg^{i}) \bigwedge \crossinspecset(\vect{x^{*}_1}, \dots, \vect{x^{*}_k})$ where $\crossinspecset(\vect{x^{*}_1}, \dots, \vect{x^{*}_k})$ encodes the relationship between the inputs used in different execution and $\outspecset(N(\vect{x^{*}_1}), \dots, N(\vect{x^{*}_k})) = \bigwedge_{i=1}^k \outpropparam{i}(N(\vect{x^{*}_i}))$. 
Next, we describe relational properties that can encode interesting DNN safety configurations over multiple executions.

\textbf{UAP verification: }Given a DNN $N$, in a UAP attack, the adversary tries to find an adversarial perturbation with a bounded $L_{\infty}$ norm that maximizes the misclassification rate of $N$ when the same adversarial perturbation is applied to all inputs drawn from the input distribution. Conversely, the UAP verification problem finds the provably correct worst-case accuracy of $N$ in the presence of a UAP adversary (referred to as UAP accuracy in the rest of the paper). \cite{iclrUap} showed that it is possible to statistically estimate (Theorem 2 in \cite{iclrUap}) UAP accuracy of $N$ w.r.t input distribution provided we can characterize the UAP accuracy of $N$ on $k$ randomly selected images e.g. the $k$-UAP problem. For the rest of the paper, we focus on the $k$-UAP verification problem as improving the precision of $k$-UAP verification directly improves UAP accuracy on the input distribution (see Appendix~\ref{appenexp:uapDistribution}). 
The $k$-UAP verification problem fundamentally differs from the commonly considered local $L_{\infty}$ robustness verification where the adversary can perturb each input independently. Since the adversarial perturbation is common across a set of inputs, the UAP verification problem requires a relational verifier that can exploit the dependency between perturbed inputs. We provide the input specification $\inspecset$ and the output specification $\outspecset$ of the UAP verification problem in Appendix~\ref{appendix:uapFormulation}. 

\textbf{Worst case hamming distance: } The hamming distance between two strings with the same length is the number of substitutions needed to turn one string into the other \cite{hamming}. Given a DNN $N$, a binary string (a list of images of binary digits), we want to formally verify the worst-case bounds on the hamming distance between the original binary string and binary string recognized by $N$ where a common perturbation can perturb each image of the binary digits. Common perturbations are a natural consequence of faulty input devices that uniformly distort the inputs already considered in verification problems in \cite{contextually_relevant_pertub}. The input specification $\inspecset$ and the output specification $\outspecset$ are in Appendix~\ref{appendix:hammingFormulation}. Beyond hamming distance and $k$-UAP, \Tool is a general framework capable of formally analyzing the worst-case performance of algorithms that rely on multiple DNN executions \cite{nonLinSpec}. For example, the absolute difference between the original and the number recognized by a digit classifier.

\textbf{Parametric bound refinement:} Common DNN verifiers \cite{crown, deeppoly} handle non-linear activations $\acti{x}$ in DNN by computing linear lower bound $\actilower{x}$ and upper bound $\actiupper{x}$ that contain all possible outputs of the activation w.r.t the input region $\phi_t$ i.e. for all possible input values $x$, $\actilower{x} \leq \acti{x} \leq \actiupper{x}$ holds. Common DNN verifiers including the SOTA relational verifier \cite{iclrUap} also compute the linear bounds $\actilower{x}$ and $\actiupper{x}$ statically without accounting for the property it is verifying. Recent works such as \cite{alphaCrown}, instead of static linear bounds, use parametric linear bounds and refine the parameters with scalable differential optimization techniques to facilitate verification of the property $(\inprop, \outprop)$. For example, for $\Relu(x)$, the parametric lower bound is $\Relu(x) \geq \alpha \times x$ where the parameter $\alpha \in [0, 1]$ decides the slope of the lower bound. Since for any $\alpha \in [0, 1]$, $\alpha \times x$ is a valid lower bound of $\Relu(x)$ it is possible to optimize over $\alpha$ while ensuring mathematical correctness. Alternatively, \cite{convex_barr} showed that optimizing $\alpha$ parameters is equivalent to optimizing the dual variables in the LP relaxed verification problem \cite{WongK18}.
However, existing works can only optimize the $\alpha$ parameters w.r.t individual executions independently making these methods sub-optimal for relational verification. The key challenge here is to develop techniques for \textbf{jointly} optimizing $\alpha$ parameters over multiple DNN executions while leveraging their inter-dependencies. 
\vspace{-5mm}
\section{Cross Executional Bound Refinement}
Before delving into the details, first, we describe why it is essential to leverage cross-execution dependencies for relational verification. For illustrative purposes, we start with the $k$-UAP verification problem on a pair of executions i.e. $k=2$.
Note that bound refinement for worst-case hamming distance can be handled similarly.
For 2-UAP, given a pair of unperturbed input $\vect{x_1}, \vect{x_2} \in \mathbb{R}^{n_0}$  first we want to prove whether there exists an adversarial perturbation $\pmb{\delta} \in \mathbb{R}^{n_0}$ with bounded $L_{\infty}$ norm $\|\pmb{\delta}\|_{\infty} \leq \epsilon$ such that $N$ misclassifies both $(\vect{x_1} + \pmb{\delta})$ and $(\vect{x_2} + \pmb{\delta})$. Now, consider the scenario where both $\vect{x_1}$ and $\vect{x_2}$ have valid adversarial perturbations $\pmb{\delta_1}$ and $\pmb{\delta_2}$ but no \textit{common} perturbation say $\pmb{\delta}$ that works for both $\vect{x_1}$ and $\vect{x_2}$. In this case, non-relational verification that does not account for cross-execution dependencies can never prove the absence of a common perturbation given that both $\vect{x_1}, \vect{x_2}$  have valid adversarial perturbations. This highlights the necessity of utilizing cross-execution dependencies. 
Next, we detail three key steps for computing a provably correct parametric linear approximation of $N$ over multiple executions. So that 
the parameters from different executions are \textit{jointly} optimized together to facilitate relational verification. 
Note that the SOTA relational verifier \cite{iclrUap} statically computes linear approximations of $N$ independently without leveraging any dependencies. 

\textbf{LP formulation: }Let, $N$ correctly classify $(\vect{x_1} + \pmb{\delta})$ if $\vect{c_1}^T N(\vect{x_1} + \pmb{\delta}) \geq 0$ and $(\vect{x_2} + \pmb{\delta})$ if $\vect{c_2}^T N(\vect{x_2} + \pmb{\delta}) \geq 0$ where $\vect{c_1}, \vect{c_2} \in \mathbb{R}^{n_{l}}$. Then $N$ does not have a common adversarial perturbation iff for all $\|\pmb{\delta}\|_{\infty} \leq \epsilon$ the outputs $\vect{y_1} = N(\vect{x_1} + \pmb{\delta})$ and $\vect{y_2} = N(\vect{x_2} + \pmb{\delta})$ satisfy $\outspecset(\vect{y_1}, \vect{y_2}) = (\vect{c_1}^T\vect{y_1} \geq 0) \vee (\vect{c_2}^T\vect{y_2} \geq 0)$. Any linear approximations specified with $\vect{L_1}, \vect{L_2} \in \mathbb{R}^{n_0}$ and $b_{1}, b_{2} \in \mathbb{R}$ of $N$ satisfying $\vect{L_1}^{T}(\vect{x_1} + \pmb{\delta}) + b_1 \leq \vect{c_1}^T\vect{y_1}$ and $\vect{L_2}^{T}(\vect{x_2} + \pmb{\delta}) + b_2 \leq \vect{c_2}^T\vect{y_2}$ for all $\pmb{\delta}$ with $\|\pmb{\delta}\|_{\infty} \leq \epsilon$ allow us to verify the absence of common adversarial perturbation with the following LP (linear programming) formulation. 
\begin{align}
& min\;\; t \;\;\text{ s.t. $\|\pmb{\delta}\|_{\infty} \leq \epsilon$} \nonumber \\
\label{eq:lpForm}
&\text{$ \vect{L_1}^{T}(\vect{x_1} + \pmb{\delta}) + b_1 \leq t $,  $ \vect{L_2}^{T}(\vect{x_2} + \pmb{\delta}) + b_2 \leq t $} 
\end{align}
Let $\opt{t}$ be the optimal solution of the LP formulation. Then $\opt{t} \geq 0$ proves the absence of a common perturbation. For fixed linear approximations $\{(\vect{L_1}, b_1), (\vect{L_2}, b_2)\}$ of $N$, the LP formulation is exact i.e. it always proves the absence of common adversarial perturbation if it can be proved with $\{(\vect{L_1}, b_1), (\vect{L_2}, b_2)\}$ (see Theorem~\ref{thm:lpoverPairExact}).
This ensures that we do not lose any precision with the LP formulation and the LP formulation is more precise than any non-relational verifier using the same $\{(\vect{L_1}, b_1), (\vect{L_2}, b_2)\}$. 
\begin{restatable}{theorem}{ExactLP}
\label{thm:lpoverPairExact}
$\vee_{i=1}^{2} (\vect{L_i}^{T} (\vect{x_i} + \pmb{\delta}) + b_i \geq 0)$ holds for all $\pmb{\delta} \in \mathbb{R}^{n_0}$ with $\|\pmb{\delta}\|_{\infty} \leq \epsilon$ if and only if $\opt{t} \geq 0$.
\end{restatable}
\vspace{-3mm}
\textbf{Proof: }The proof follows from Appendix Theorem~\ref{thm:pairExact}.

However, the LP formulation only works with fixed $\{(\vect{L_1}, b_1), (\vect{L_2}, b_2)\}$ and as a result, is not suitable for handling parametric linear approximations that can then be optimized to improve the relational verifier's precision. Instead, we use the equivalent Lagrangian Dual \cite{convexOpt} which retains the benefits of the LP formulation while facilitating joint optimation of parameters from multiple executions as detailed below.

\textbf{Dual with parametric linear approximations: } Let, for a list of parametric activation bounds specified by a parameter list $\pmb{\alpha} = [\alpha_1, \dots, \alpha_m]$ we denote corresponding parametric linear approximation of $N$ with the coefficient $\linapprxalpha{}$ and bias $\biasalpha{}$. First, for 2-UAP, we obtain $(\linapprxalpha{1}, \biasalpha{1})$ and $(\linapprxalpha{2}, \biasalpha{2})$ corresponding to the pair of executions using existing works \cite{alphaCrown}. For $i \in \{1, 2\}$, $\|\pmb{\delta}\|_{\infty} \leq \epsilon$ and $\vect{l}_{i} \preceq
\param{i} \preceq \vect{u}_{i}$ the parametric linear bounds satisfy $\linapprxalpha{i}^T(\vect{x_i} + \pmb{\delta}) + \biasalpha{i} \leq \vect{c_i}^T \vect{y_i}$ where $\vect{l}_i, \vect{u}_i$ are constant vectors defining valid range of the parameters $\param{i}$. For fixed $\param{i}$ the Lagrangian Dual of the LP formulation in Eq.~\ref{eq:lpForm} is as follows where $\lambda_1, \lambda_2 \in [0, 1]$ with $\lambda_1 + \lambda_2 = 1$ are the Lagrange multipliers relating linear approximations from different executions (details in Appendix~\ref{appendix:duallpPair}).
\begin{align*}
\max\limits_{0 \leq \lambda_i \leq 1}\min\limits_{\|\pmb{\delta}\|_{\infty} \leq \epsilon} \sum\nolimits_{i=1}^{2} \lambda_i \times \left(\linapprxalpha{i}^T(\vect{x_i} + \pmb{\delta}) + \biasalpha{i}\right)
\end{align*}
Let, for fixed $\param{1}, \param{2}$ the optimal solution of the dual formulation be $\opt{t}(\param{1}, \param{2})$. Then we can prove the absence of common perturbation provided the maximum value of $\opt{t}(\param{1}, \param{2})$ optimized over $\param{1}, \param{2}$ is $\geq 0$. This reduces the problem to the following: $\max \opt{t}(\param{1}, \param{2}) \text{ s.t. } \vect{l}_{1} \preceq \param{1} \preceq \vect{u}_{1} \;\; \vect{l}_{2} \preceq \param{2} \preceq \vect{u}_{2}$. However, the optimization problem involves a max-min formulation and the number of parameters in $\param{1}, \param{2}$ in the worst-case scales linearly with the number of activation nodes in $N$. This makes it hard to apply gradient descent-based techniques typically used for optimization \cite{alphaCrown}. Instead, we reduce the max-min formulation to a simpler maximization problem by finding an optimizable closed form of the inner minimization problem.

\textbf{Deriving optimizable closed form :} We want to characterize the closed form $\dual (\optparams) = \min\limits_{\|\pmb{\delta}\|_{\infty} \leq \epsilon} \sum\nolimits_{i=1}^{2} \lambda_i \times \left(\linapprxalpha{i}^T(\vect{x_i} + \pmb{\delta}) + \biasalpha{i}\right)$  where $\optparams = (\param{1}, \param{2}, \lambda_1, \lambda_2)$ and use it for formulating the maximization problem. Note, $\dual(\optparams)$ is related to the dual function from optimization literature \cite{convexOpt}. Naively, it is possible to solve the inner minimization problem for two different executions separately and then optimize them over $\overline{\param{}} = (\param{1}, \param{2})$ using $\naivedual(\overline{\param{}}) = max(\naivedual_1(\param{1}), \naivedual_2(\param{2}))$ as shown below. However, $\naivedual(\overline{\param{}})$ produces a suboptimal result since it ignores cross-execution dependencies and misses out on the benefits of jointly optimizing $(\param{1}, \param{2})$.
\begin{align}
\label{eq:dualNaive}
\naivedual_i(\param{i}) = \min\limits_{\|\pmb{\delta}\|_{\infty} \leq \epsilon} \linapprxalpha{i}^T(\vect{x_i} + \pmb{\delta}) + \biasalpha{i}
\end{align}
Since $\|\pmb{\delta}\|_{\infty}$ is bounded by $\epsilon$, it is possible to \textit{exactly} compute the closed form of $\dual(\optparams)$ as shown below where for $j \in [n_0]$,   $\linapprxalpha{i}[j] \in \mathbb{R}$ denotes the $j$-th component of $\linapprxalpha{i} \in \mathbb{R}^{n_0}$ and $\constalpha{i} = \linapprxalpha{i}^T\vect{x_i} + \biasalpha{i}$
\begin{align*}
\dual(\optparams) &= \min\limits_{\|\pmb{\delta}\|_{\infty} \leq \epsilon} \sum_{i=1}^{2} \lambda_i \times \left(\linapprxalpha{i}^T(\vect{x_i} + \pmb{\delta}) + \biasalpha{i}\right) \\
\dual (\optparams) &= \sum_{i=1}^{2} \lambda_i \times \constalpha{i} + \min\limits_{\|\pmb{\delta}\|_{\infty} \leq \epsilon} \sum_{i=1}^{2} \lambda_i \times \linapprxalpha{i}^T\pmb{\delta} \\
\dual (\optparams) &= \sum_{i=1}^{2} \lambda_i \times \constalpha{i} - \epsilon \times  \sum_{j=1}^{n_0}\left|\sum_{i=1}^{2} \lambda_i \times \linapprxalpha{i}[j]\right|
\end{align*}
Unlike $\naivedual(\overline{\param{}})$, $\dual(\optparams)$ relates linear approximations from two different executions using $(\lambda_{1}, \lambda_{2})$ enabling joint optimization over $(\param{1}, \param{2})$. With the closed form $\dual(\optparams)$, we can use projected gradient descent to optimize $max_{\optparams} G(\optparams)$ while ensuring the parameters in $\optparams$ satisfy the corresponding constraints. Next, we provide theoretical guarantees about the correctness and efficacy of the proposed technique. For efficacy, we show the optimal solution $\opt{t}(\dual)$ obtained with $\dual(\optparams)$ is always as good as $\opt{t}(\naivedual)$ i.e. $ \opt{t}(\dual) \geq \opt{t}(\naivedual)$ (Theorem~\ref{thm:optimalBounds}) and characterize sufficient condition where $\opt{t}(\dual)$ is strictly better i.e. $ \opt{t}(\dual) > \opt{t}(\naivedual)$ (Appendix Theorem~\ref{thm:strictBetter}). Experiments substantiating the improvement in the optimal values ($\opt{t}(\dual)$ vs. $\opt{t}(\naivedual)$) are in Section~\ref{sec:boundImprovement}.

\begin{restatable}{theorem}{optimalBounds}
\label{thm:optimalBounds}
If $\opt{t}(\dual) = \max_{\optparams} \dual(\optparams)$ and $\opt{t}(\naivedual) = \max_{\param{1}, \param{2}} \naivedual(\param{1}, \param{2})$ then $\opt{t}(\naivedual) \leq \opt{t}(\dual)$. 
\end{restatable}
\textbf{Proof: }
For any $\vect{l}_{1} \preceq \param{1} \preceq \vect{u}_{1} \;\; \vect{l}_{2} \preceq \param{2} \preceq \vect{u}_{2}$, consider $\optparams_1 = (\param{1}, \param{2}, \lambda_1=1, \lambda_2=0)$ and $\optparams_2 = (\param{1}, \param{2}, \lambda_1=0, \lambda_2=1)$, then $\dual(\optparams_1) = \min\limits_{\|\pmb{\delta}\|_{\infty} \leq \epsilon} \linapprxalpha{1}^T(\vect{x_1} + \pmb{\delta}) + \biasalpha{1}$ and $\dual(\optparams_2) = \min\limits_{\|\pmb{\delta}\|_{\infty} \leq \epsilon} \linapprxalpha{2}^T(\vect{x_2} + \pmb{\delta}) + \biasalpha{2}$. Since, $\opt{t}(\dual) \geq \dual(\optparams_1)$ and $\opt{t}(\dual) \geq \dual(\optparams_2)$ then $\opt{t}(\dual) \geq \max\limits_{1 \leq i \leq 2}\dual(\optparams_i) = \naivedual(\param{1}, \param{2})$. Hence, $\opt{t}(\dual) \geq \max_{\param{1}, \param{2}} \naivedual(\param{1}, \param{2}) = \opt{t}(\naivedual)$.

The correctness proof for bound refinement between two executions is in Appendix~\ref{appendix:correctnessPair}. Note that correctness does not necessitate the optimization technique to identify the global maximum, especially since gradient-descent-based optimizers may not always find the global maximum.

\textbf{Genralization to multiple executions: } Instead of a pair of executions considered above, we now generalize the approach to any set of $n$ executions where $n \leq k$. With parametric linear approximations $\{(\linapprx_{1}, b_1), \dots, (\linapprx_{n}, b_{n})\}$ of $N$ for all $n$ executions, we formulate the following LP to prove the absence of common adversarial perturbation that works for \textit{all} $n$ executions. The proof of exactness of the LP formulation is in Appnedix~Theorem~\ref{thm:pairExact}.
\begin{align}
& min\;\; t \;\;\text{ s.t. $\|\pmb{\delta}\|_{\infty} \leq \epsilon$} \nonumber \\
\label{eq:lpFormnExecution}
&\vect{L_i}^{T}(\vect{x_i} + \pmb{\delta}) + b_i \leq t\;\;\;\; \textbf{  
  $\forall i \in [n]$}
\end{align}
Similar to a pair of executions, we first specify the Lagrangian dual of the LP (Eq.~\ref{eq:lpFormnExecution}) by introducing $n$ lagrangian multipliers $\lambda_1, \dots, \lambda_{n}$ that satisfy for all $i \in [n]$ $\lambda_i \in [0, 1]$ and $\sum_{i=1}^{n} \lambda_{i} = 1$. Subsequently, we obtain the closed form $\dual(\optparams)$ where $\optparams = (\param{1}, \dots, \param{n}, \lambda_{1}, \dots, \lambda_{n})$ and $\constalpha{i} = \linapprxalpha{i}^T\vect{x_i} + \biasalpha{i}$ as shown below. 
\begin{align*}
\dual (\optparams) &= \sum_{i=1}^{n} \lambda_i \times \constalpha{i} - \epsilon \times  \sum_{j=1}^{n_0}\left|\sum_{i=1}^{n} \lambda_i \times \linapprxalpha{i}[j]\right|    
\end{align*}
Theoretical results regarding the correctness and efficacy of bound computation over $n$ executions are in Appendix~\ref{sec:nTheorems}.

\textbf{Genralization to a conjunction of linear inequalities: } Until now, we assume for each execution the output specification is defined as a linear inequality i.e. $\vect{c_i}^TN(\vect{x_i} + \pmb{\delta}) \geq 0$. Next, we generalize our method to any output specification for each execution defined with conjunction of $m$ linear inequalities.   
For example, if $\vect{y_i}$ denotes the output of the $i$-th execution $\vect{y_i} = N(\vect{x_i} + \pmb{\delta})$ then the output specification $\outpropparam{i}(\vect{y_i})$ is given by $\outpropparam{i}(\vect{y_i}) = \bigwedge_{j=1}^{m} (\vect{c_{i,j}}^T\vect{y_i}\geq 0)$ where $\vect{c_{i,j}} \in \mathbb{R}^{n_l}$. In this case, $\outprop(\vect{y_i})$ is satisfied iff $(\min_{1 \leq j \leq m}\vect{c_{i,j}}^T\vect{y_i}) \geq 0$. Using this observation, we first reduce this problem to subproblems with a single linear inequality (see Appendix~Theorem~\ref{lem:multipleToSingle}) and subsequently characterize the closed form $\dual(\optparams)$ for each subproblem separately. However, the number of subproblems in the worst case can be $m^{n}$ which is practically intractable for large $m$ and $n$. Hence, we greedily select which subproblems to use for bound refinement to avoid exponential blow-up in the runtime while ensuring the bound refinement remains provably correct (see Appendix~\ref{sec:reductionToSingleLinearIneq}).
Since most of the common DNN output specification can be expressed as a conjunction of linear inequalities \cite{crown} \Tool generalizes to them. Moreover, cross-excution bound refinement is not restricted to $L_{\infty}$ input specification where $\|\pmb{\delta}\|_{\infty}$ is bounded and can work for any $\|\cdot\|_{p}$ norm bounded perturbation (see Appendix~\ref{sec:lPNorm}).

Next, we utilize the cross-executional bound refinement to formulate a MILP with at most $O(k \times n_{l})$ integer variables. Similar to \cite{iclrUap} we only use integer variables to encode the output specification $\outspecset$. Since the output dimension $n_{l}$ of $N$ is usually much smaller than the number of total \Relu~nodes $n_{l} << n_{r}$ in $N$, \Tool is more scalable than the naive MILP encoding that in the worst case introduces $O(k \times n_{r})$ integer variables. 
\vspace{-4mm}
\section{\Tool Algorithm}
\label{sec:algorithm}
\vspace{-2mm}
The cross-executional bound refinement learns parameters over any set of $n$ executions. However, for a relational property defined over $k$ executions, since there are $2^{k} - 1$ non-empty subsets of executions, refining bounds for all possible subsets is impractical. Instead, we design a greedy heuristic to pick the subsets of executions so that we only use a small number of subsets for bound refinement.

\textbf{Eliminating individually verified executions: } First, we run existing non-relational verifiers \cite{crown, deeppoly} without tracking any dependencies across executions. \Tool eliminates the executions already verified with the non-relational verifier and does not consider them for subsequent steps. (lines 5 -- 9 in Algo.~\ref{alg:Racoon}) For example, for the $k$-UAP property, we do not need to consider those executions that are proved to have no adversarial perturbation $\pmb{\delta}$ such that $\|\pmb{\delta}\|_{\infty} \leq \epsilon$. For relational properties considered in this paper, we formally prove the correctness of the elimination technique in Appendix Theorem~\ref{thm:eliminationCorrect} and showcase eliminating verified executions does not lead to any loss in precision of \Tool.

\textbf{Greedy selection of unverified executions: } For each execution that remains unverified with the non-relational verifier ($\mathcal{V}$), we look at $s_{i} = \min_{1 \leq j \leq m} \vect{c_{i, j}}^T\vect{y_i}$ estimated by $\mathcal{V}$ where $\vect{y_i}= N(\vect{x_i} + \pmb{\delta})$ and $\vect{c_{i, j}} \in \mathbb{R}^{n_l}$ defines the corresponding output specification $\outpropparam{i}(\vect{y_i}) = \bigwedge_{i=1}^{m} (\vect{c_{i, j}}^T\vect{y_i} \geq 0)$. Intuitively, for unverified executions, $s_{i}$ measures the maximum violation of the output specification $\outpropparam{i}(\vect{y_i})$ and thus leads to the natural choice of picking executions with smaller violations for cross-executional refinement. We sort the executions in decreasing order of $s_{i}$ and pick the first $k_{0}$ (hyperparameter) executions on input regions $\vect{X} = \{\inregparam{1}, \dots, \inregparam{{k_0}}\}$ having smaller violations $s_{i}$ where for all $i \in [k_0]$, $\inregparam{i} = \{\vect{x_i'} + \pmb{\delta}\;|\; \vect{x_i'},\;\pmb{\delta} \in \mathbb{R}^{n_0} \wedge \|\pmb{\delta}\|_{\infty} \leq \epsilon\}$ and $\vect{x_i'}$ is the unperturbed input. (line 11 of Algo.~\ref{alg:Racoon}) In general, $k_{0}$ is a small constant i.e. $k_{0} \leq 10$. Further, we limit the subset size to $k_{1}$ (hyperparameter) and do not consider any subset of $\vect{X}$ with a size more than $k_1$ for cross-executional bound refinement. (lines 12 -- 15 in Algo.~\ref{alg:Racoon})
Overall, we consider $\sum_{i=1}^{k_1} \binom{k_0}{i}$ subsets for bound refinement.
\begin{algorithm}[tb]
   \caption{\Tool}
   \label{alg:Racoon}
\begin{algorithmic}[1]
   \STATE {\bfseries Input:} $N$, $(\inspecset, \outspecset)$, $k$, $k_{0}$, $k_{1}$, non-relational verifier $\mathcal{V}$.
   \STATE {\bfseries Output: } sound approximation of worst-case $k$-UAP accuracy or worst-case hamming distance $\apprxOutput(\inspecset, \outspecset)$.
   \STATE $I \gets \{\}$. \hfill\COMMENT{Indices of executions not verified by $\mathcal{V}$} 
      \STATE $\linearmap \gets \{\}$ \hfill\COMMENT{Map storing linear approximations}
   \FOR{\text{$i \in [k]$}} \label{algline:nonrelStart}
   \STATE $(s_i, \linapprx_{i}, b_{i}) \gets$ \text{$\mathcal{V}(\phi^{i}, \outpropparam{i}$)}.
   \IF{$\mathcal{V}$ can not verify $(\phi^{i}, \outpropparam{i})$}
   \STATE $I \gets I \cup \{i\}$;\;\; $\linearmap[i] \gets \linearmap[i] \cup (\linapprx_{i}, b_{i})$.
    \ENDIF
   \ENDFOR \label{algline:nonrelEnd}
   \STATE $I_{0} \gets \text{top-$k_0$ executions from $I$ selected based on $s_i$}$.
   \FOR{\text{$\overline{I_{0}} \subseteq I_0$}, $\overline{I_{0}} \neq \{\}$ and $|\overline{I_{0}}| \leq k_1$}
  \STATE $\linearmap_{\overline{I_{0}}} \gets $ \text{CrossExcutionalRefinement$(\overline{I_0}, \inspecset, \outspecset)$}.
  \STATE $\linearmap \gets $ \text{Populate$(\linearmap, \linearmap_{\overline{I_{0}}})$}.\hfill\COMMENT{Storing $\linearmap_{\overline{I_{0}}}$ in $\linearmap$}
   \ENDFOR
   \STATE $\mathcal{M} \gets $\text{MILPFormulation($\linearmap, \inspecset, \outspecset$, $k$, $I$)}. 
   \STATE \textbf{return }{ $\text{Optimize}(\mathcal{M})$}.
\end{algorithmic}
\end{algorithm}

\textbf{MILP formulation: } \Tool MILP formulation involves two steps. First, we deduce linear constraints between the input and output of $N$ for each unverified execution using linear approximations of $N$  either obtained through cross-executional refinement or by applying the non-relational verifier. Secondly, similar to the current SOTA baseline \cite{iclrUap} we encode the output specification $\outspecset$ as MILP objective that only introduces $O(k \times n_{l})$ integer variables. Finally, we use an off-the-shelf MILP solver \cite{gurobi2018} to optimize the MILP. 

For the $i$-th unverified execution, let $\inregparam{i} = \{\vect{x_i'} + \pmb{\delta}\;|\; \vect{x_i'},\;\pmb{\delta} \in \mathbb{R}^{n_0} \wedge \|\pmb{\delta}\|_{\infty} \leq \epsilon\}$ be the input region and for $\vect{y_i} = N(\vect{x_{i}'} + \pmb{\delta})$, $\outpropparam{i}(\vect{y_i}) = \bigwedge_{i=1}^{m} (\vect{c_{i, j}}^T\vect{y_i} \geq 0)$ be the output specification. Subsequently for each clause $(\vect{c_{i, j}}^T\vect{y_i} \geq 0)$ in $\outpropparam{i}(\vect{y_i})$ let $\{(\linapprx^{1}_{i,j}, b^{1}_{i,j}), \dots, (\linapprx^{k'}_{i,j}, b^{k'}_{i,j}) \}$ be set of linear approximations. Then for each $l \in [k']$ we add the following linear constraints where $o_{i,j}$ is a real variable. 
\begin{align*}
\linapprx^{l}_{i,j}(\vect{x_i'} + \pmb{\delta}) + b^{l}_{i,j} \leq o_{i, j}\;;\;\|\pmb{\delta}\|_{\infty} \leq \epsilon
\end{align*}
Next, similar to \cite{iclrUap} we encode output specification ($\outpropparam{i}$) as $z_{i} = (\min_{1 \leq j \leq m} o_{i, j}) \geq 0$ where $z_{i} \in \{0, 1\}$ are binary variables and $z_{i} = 1$ implies $\outpropparam{i}(y_i) = True$. Encoding of each $\outpropparam{i}$ introduces $O(m)$ binary (integer) variables. Since for $k$-UAP and worst-case hamming distance, $m = n_{l}$ the total number of integer variables is in the worst case $O(k \times n_{l})$. MILP encoding for $k$-UAP and worst-case hamming distance verification are shown in Appendix~\ref{appendix:milPFormulationKUAPHamm}. We prove the correctness of \Tool in Appendix Theorem~\ref{thm:racoonSoundness} and show it is always at least as precise as \cite{iclrUap} (Appendix Theorem~\ref{thm:racoonBetter}). Worst-case time complexity analysis of \Tool is in Appendix~\ref{appendix:WorstcaseComplexity}.

\textbf{Limitation: } Similar to other deterministic (relational or non-relational) verifiers \Tool does not scale to DNNs trained on larger datasets (e.g.  ImageNet). \Tool is sound but incomplete and for some cases, \Tool may fail to prove a property even if the property holds. However, for piecewise linear activations like ReLU, it is possible to design a ``Branch and Bound" based complete relational verifiers by combining \Tool (as bounding algorithm) with branching algorithms like ReLU splitting \cite{wang2021beta}. We leave that as future work. Note that existing complete non-relational verifiers like \cite{wang2021beta} are incomplete for relational properties since they can only verify each execution in isolation. 
\vspace{-5mm}
\section{Experimental Evaluation}
\label{sec:expEval}
\begin{figure*}[htb]
\centering
\begin{minipage}[b]{.23\textwidth}
\includegraphics[width=\textwidth]{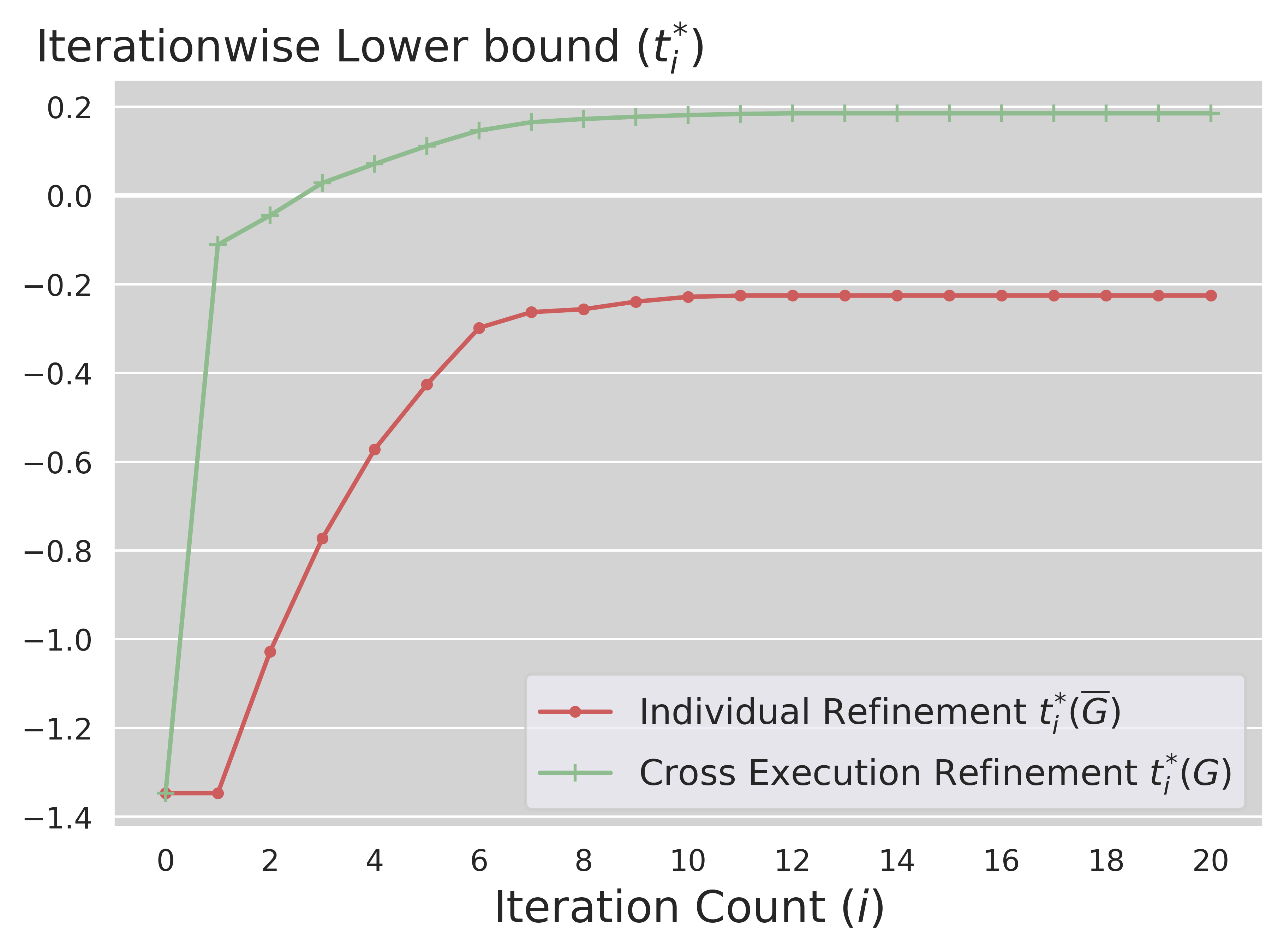}
\captionsetup{labelformat=empty}
\caption{(a) MNIST (PGD)}
\end{minipage}
\addtocounter{figure}{-1}
\begin{minipage}[b]{.23\textwidth}
\includegraphics[width=\textwidth]{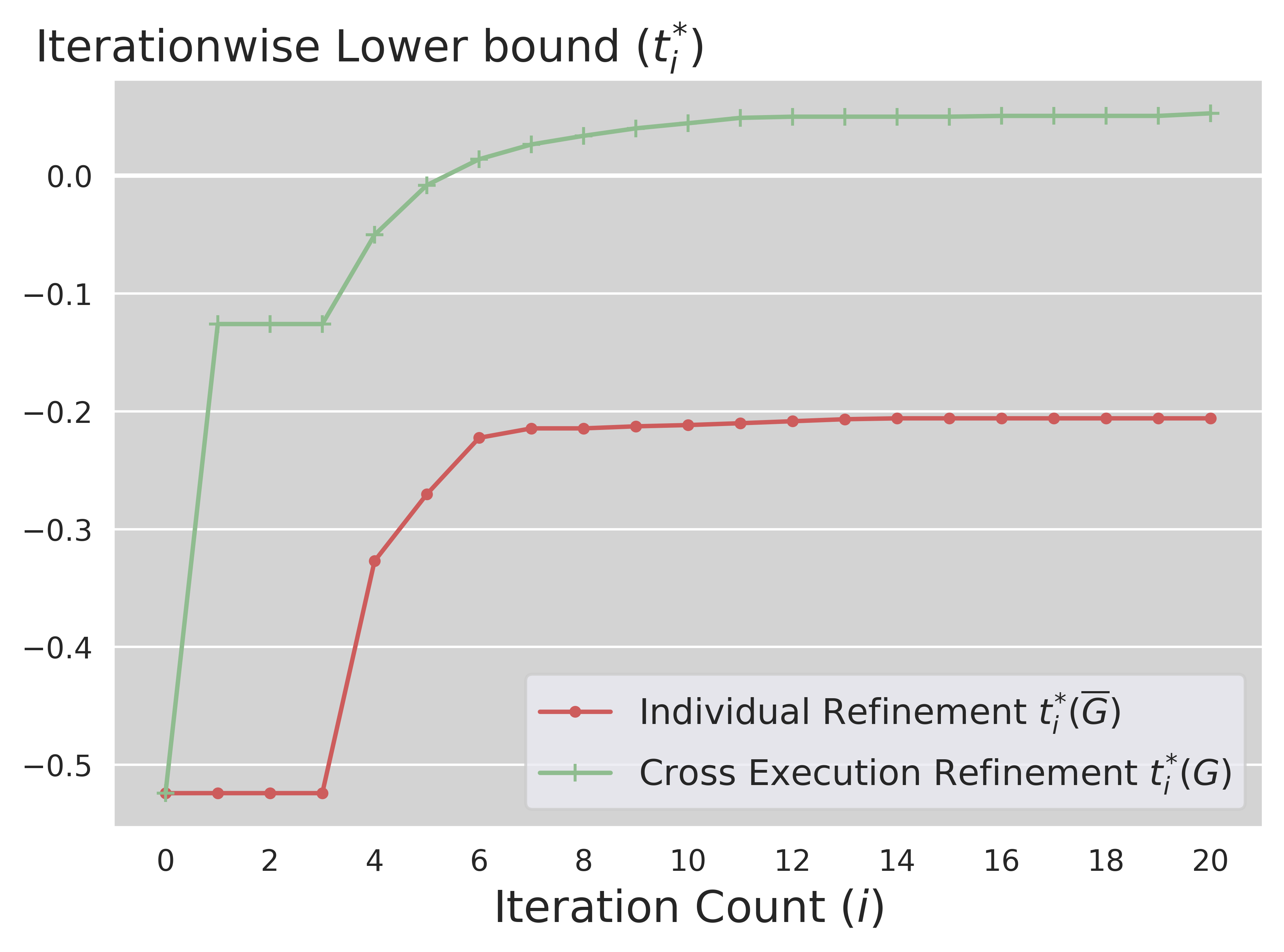}
\captionsetup{labelformat=empty}
\caption{(b) MNIST (DiffAI)}
\end{minipage}\qquad
\addtocounter{figure}{-1}
\begin{minipage}[b]{.23\textwidth}
\includegraphics[width=\textwidth]{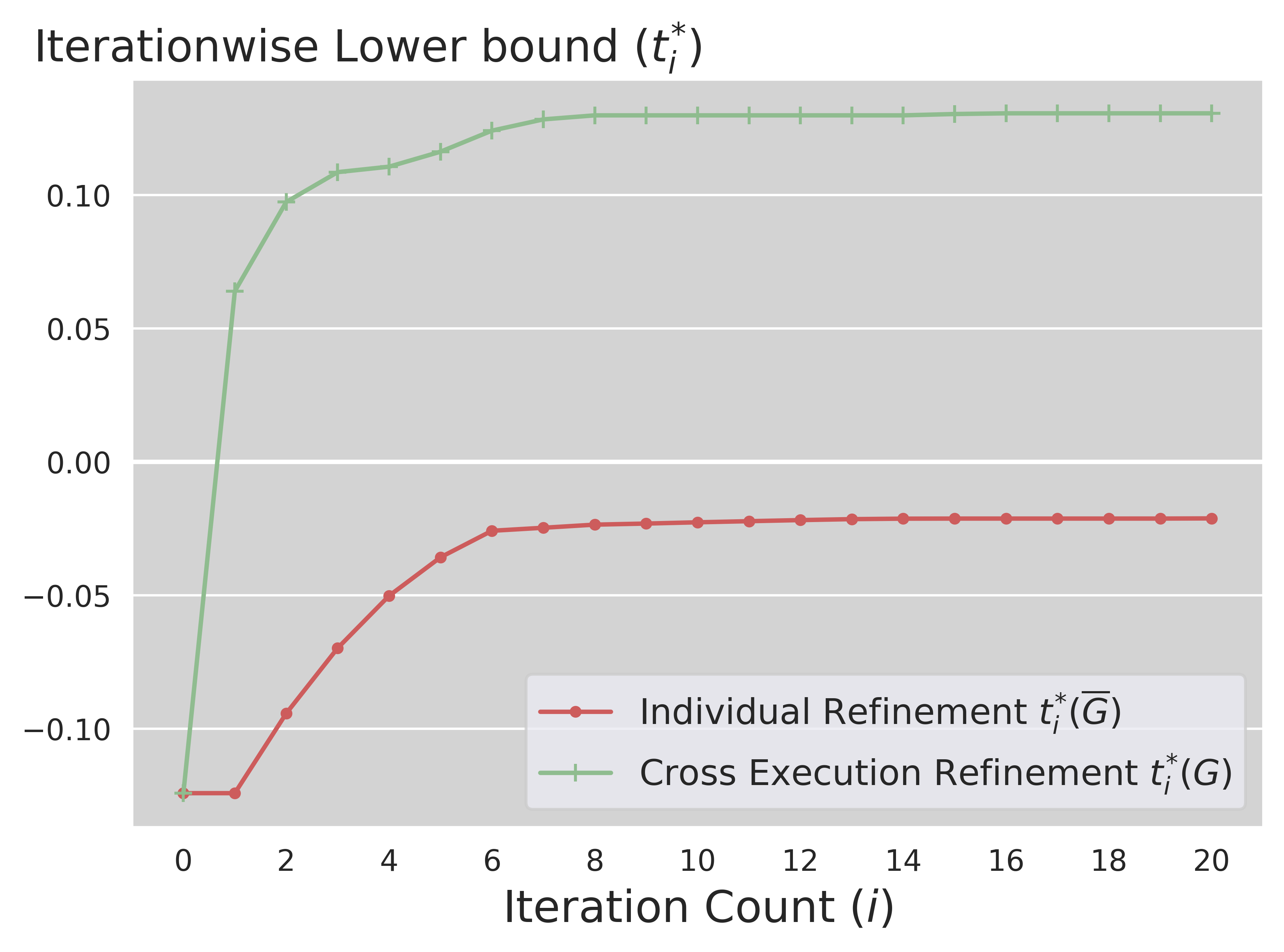}
\captionsetup{labelformat=empty}
\caption{(c) CIFAR10 (PGD)}
\end{minipage}
\addtocounter{figure}{-1}
\begin{minipage}[b]{.23\textwidth}
\includegraphics[width=\textwidth]{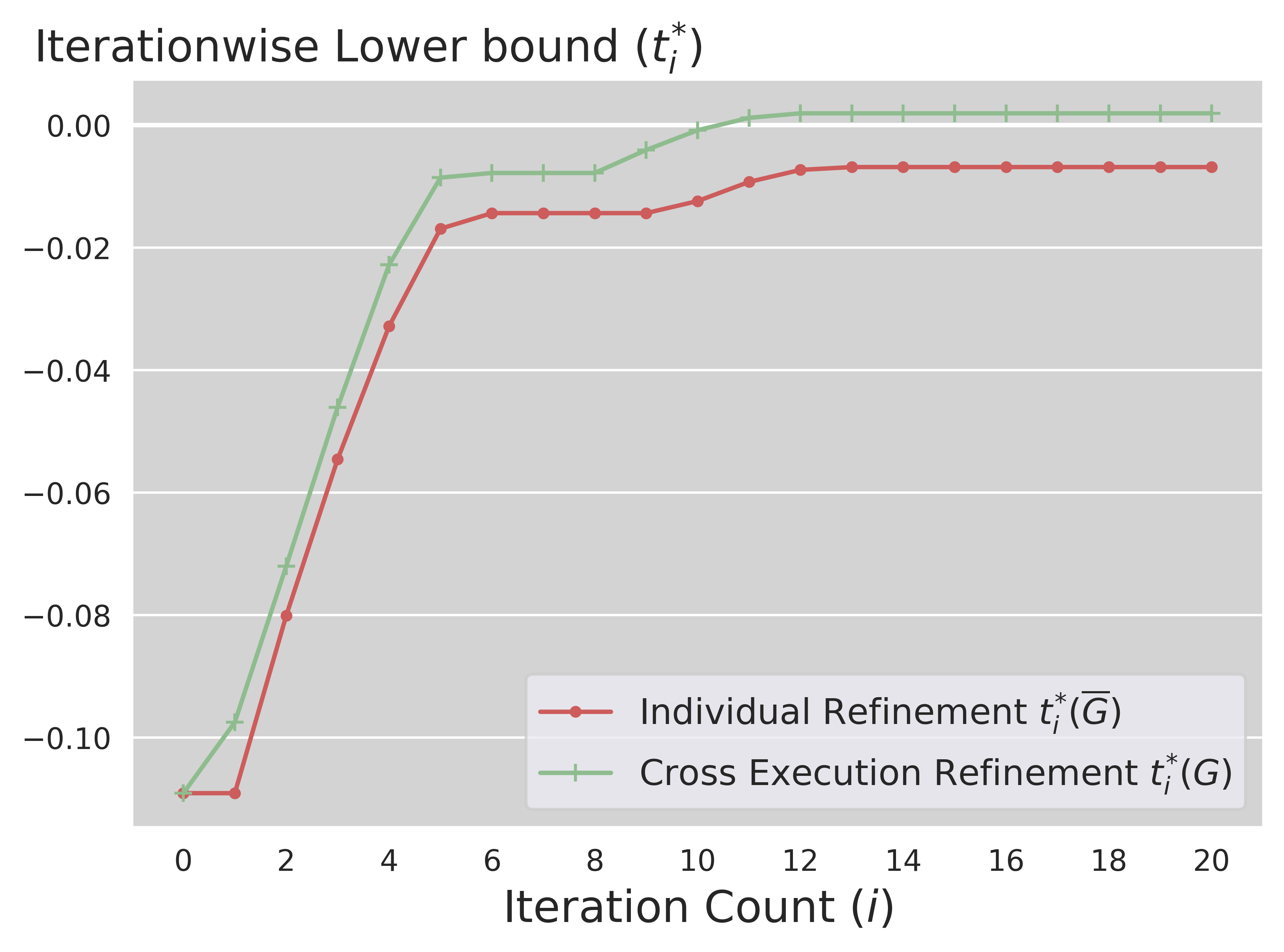}
\captionsetup{labelformat=empty}
\caption{(d) CIFAR10 (DiffAI)}
\end{minipage}
\addtocounter{figure}{-1}
\caption{Lower bound ($t$ in Eq.~\ref{eq:lpFormnExecution}) from individual vs. cross executional bound refinement over $2$ executions on ConvSmall networks.}
\vspace{-1mm}
\label{fig:bound_imporvement}
\end{figure*}
We evaluate the effectiveness of \Tool on a wide range of relational properties and a diverse set of DNNs and datasets. We consider the following relational properties: $k$-UAP, worst-case hamming distance as formally defined in Appendix~\ref{sec:FormalDefinition}. The baselines we consider are the SOTA relational verifier \cite{iclrUap} (referred to as I/O Formulation) and the non-relational verifier \cite{auto_lirpa} from the SOTA \lirpa toolbox \cite{auto_lirpa}. used by \cite{iclrUap}. We also analyze the efficacy of cross-executional bound refinement in learning parametric bounds that can facilitate relational verification (Section~\ref{sec:boundImprovement}). Note that we instantiate \Tool with the same non-relational verifier \cite{auto_lirpa} used in I/O formulation~\cite{iclrUap}.
The performance evaluation of different components of \Tool including individual bound refinement (i.e. refinement on execution set of size $1$), and individual bound refinement with MILP is in Appendix~Table~\ref{table:compareComponents}. Note that individual bound refinement uses SOTA non-relational bound refinement algorithm $\alpha$-CROWN \cite{alphaCrown}.
\vspace{-4mm}
\subsection{Experiment setup}
\vspace{-1mm}
\noindent \textbf{Networks}. We use standard convolutional architectures (ConvSmall, ConvBig, etc.) commonly seen in other neural network verification works \cite{crown, deeppoly} (see Table~\ref{table:compareoriginal}). Details of DNN architectures used in experiments are in Appendix~\ref{appenexp:archDetails}. We consider networks trained with standard training, robust training: DiffAI \cite{diffai}, CROWN-IBP \cite{zhang2019towards}, projected gradient descent (PGD) \cite{pgdTraining}, and COLT \cite{COLTRela}. We use pre-trained publically available DNNs: CROWN-IBP DNNs taken from the CROWN repository \cite{zhang2019towards} and all other DNNs are from the ERAN repository \cite{deeppoly}. 

\noindent \textbf{Implementation Details}. 
The details regarding the frameworks \Tool uses, the CPU and GPU information, and the hyperparameter ($k_0$, $k_1$) values are in Appendix~\ref{appenexp:implemenDetails}.  
\vspace{-2mm}
\subsection{Evaluating cross execution bound refinement}
\vspace{-1mm}
\label{sec:boundImprovement}
Fig.~\ref{fig:bound_imporvement} shows the values $\opt{t}_{i}(\dual)$ and $\opt{t}_{i}(\naivedual)$ after $i$-th iteration of Adam optimizer computed by cross-executional and individual refinement (using $\alpha$-CROWN) respectively over a pair of executions (i.e. $k=2$) on randomly chosen images. We used ConvSmall PGD and DiffAI DNNs trained on MNIST and CIFAR10 for this experiment. The $\epsilon$s used for MNIST PGD and DiffAI DNNs are $0.1$ and $0.12$ respectively while $\epsilon$s used for CIFAR10 PGD and DiffAI DNNs are $2.0/255$ and $6.0/255$ respectively. For each iteration $i$, $\opt{t}_{i}(\dual) > \opt{t}_{i}(\naivedual)$ shows that cross-executional refinement is more effective in learning parametric bounds that can facilitate relation verification. Since, for proving the absence of common adversarial perturbation, we need to show $\opt{t} \geq 0$, in all 4 cases in Fig.~\ref{fig:bound_imporvement} individual refinement fails to prove the absence of common adversarial perturbation while cross-executional refinement succeeds. Moreover, in all 4 cases, even the optimal solution of the LP (Eq.~\ref{eq:lpFormnExecution}) formulated with linear approximations from individual refinement remains negative. For example, for MNIST DiffAI DNN, with LP, $\opt{t}(\naivedual)$ improves to $-0.05$ from $-0.2$ but remains insufficient for proving the absence of common adversarial perturbation. This shows the importance of leveraging dependencies across executions during bound refinement.
\begin{table*}[htb]
\centering
\captionsetup{justification=centering}
\caption{\Tool Efficacy Analysis}
\label{table:compareoriginal}
\resizebox{0.98\textwidth}{!}{
\begin{tabular}{@{}c c c c c c c c c c c@{}}
\toprule
Dataset & Property & 
Network & Training &  Perturbation & \multicolumn{2}{c}{Non-relational Verifier} & \multicolumn{2}{c}{I/O Formulation} & \multicolumn{2}{c}{\Tool}  \\
\text{ } & \text{ } & 
Structure & \text{Method} \text{ } & \text{Bound ($\epsilon$) } & Avg. UAP Acc. (\%) & Avg. Time (sec.) & Avg. UAP Acc. (\%) & Avg. Time (sec.) & Avg. UAP Acc. (\%) & Avg. Time (sec.)  \\
\midrule
\text{ } & UAP &  ConvSmall & Standard & 0.08 & 38.5 & 0.01 & 48.0 & 2.65 & 54.0\;(\textcolor{mgreen}{+6.0}) & 5.20 \\
\text{ } & UAP &  ConvSmall & PGD & 0.10 & 70.5 & 0.21 & 72.0 & 0.92 & 77.0\;(\textcolor{mgreen}{+5.0}) & 4.33 \\
\text{ } &UAP &  IBPSmall & IBP & 0.13 & 74.5 & 0.02 & 75.0 & 1.01 & 89.0\;(\textcolor{mgreen}{+14.0}) & 2.01 \\
MNIST & UAP &  ConvSmall & DiffAI & 0.13 & 56.0 & 0.01 & 61.0 & 1.10 & 68.0\;(\textcolor{mgreen}{+7.0}) & 3.98 \\
\text{ } & UAP &  ConvSmall & COLT & 0.15 & 69.0 & 0.02 & 69.0 & 0.99 & 85.5\;(\textcolor{mgreen}{+16.5}) & 2.68 \\
\text{ } & UAP &  IBPMedium & IBP & 0.20 & 80.5 & 0.1 & 82.0 & 0.99 & 93.5\;(\textcolor{mgreen}{+11.5}) & 2.30 \\
\text{ } & UAP &  ConvBig & DiffAI & 0.20 & 81.5 & 1.85 & 81.5 & 2.23 & 91.5\;(\textcolor{mgreen}{+10.0}) & 7.60 \\
\midrule
\text{ } & UAP &  ConvSmall & Standard & 1.0/255 & 52.0 & 0.02 & 55.0 & 3.46 & 58.0\;(\textcolor{mgreen}{+3.0})& 7.22 \\ 
\text{ } & UAP &  ConvSmall & PGD & 3.0/255 & 21.0 & 0.01 & 26.0 & 1.57 & 29.0\;(\textcolor{mgreen}{+3.0})& 5.56 \\
\text{ } & UAP &  IBPSmall & IBP & 6.0/255 & 17.0 & 0.02 & 17.0 & 2.76 & 39.0\;(\textcolor{mgreen}{+22.0})& 6.76 \\
CIFAR10 & UAP &  ConvSmall & DiffAI & 8.0/255 & 16.0 & 0.01 & 20.0 & 2.49 & 30.0\;(\textcolor{mgreen}{+10.0})& 7.09 \\
\text{ } & UAP &  ConvSmall & COLT & 8.0/255 & 18.0 & 0.04 & 21.0 & 2.41 & 26.0\;(\textcolor{mgreen}{+5.0})& 11.02 \\
\text{ } & UAP &  IBPMedium & IBP & 3.0/255 & 46.0 & 0.15 & 50.0 & 2.13 & 71.0\;(\textcolor{mgreen}{+21.0})& 6.12 \\
\text{ } & UAP &  ConvBig & DiffAI & 3.0/255 & 17.0 & 1.33 & 20.0 & 3.42 & 25.0\;(\textcolor{mgreen}{+5.0})& 11.92 \\
\midrule
Dataset & Property & 
Network & Training &  Perturbation & \multicolumn{2}{c}{Non-relational Verifier} & \multicolumn{2}{c}{I/O Formulation} & \multicolumn{2}{c}{\Tool}  \\
\text{ } & \text{ } & 
Structure & \text{Method} \text{ } & \text{Bound ($\epsilon$) } & Avg. Hamming distance & Avg. Time (sec.) & Avg. Hamming distance & Avg. Time (sec.) & Avg. Hamming distance & Avg. Time (sec.)  \\
\midrule
\text{ } & Hamming &  ConvSmall & Standard & 0.10 & 19.0 & 0.01 & 18.0 & 2.68 & 16.0\;(\textcolor{mgreen}{-2.0})& 4.43 \\
\text{ } & Hamming &  ConvSmall & PGD & 0.12 & 17.0 & 0.01 & 16.0 & 0.99 & 14.0\;(\textcolor{mgreen}{-2.0})& 3.20 \\
\text{ } & Hamming &  ConvSmall & DiffAI & 0.15 & 16.0 & 0.01 & 16.0 & 0.98 & 14.0\;(\textcolor{mgreen}{-2.0})& 3.46 \\
MNIST & Hamming &  IBPSmall & IBP & 0.14 & 11.0 & 0.01 & 10.0 & 1.13 & 5.0\;(\textcolor{mgreen}{-5.0})& 2.56 \\
\text{ } & Hamming &  ConvSmall & COLT & 0.20 & 17.0 & 0.01 & 17.0 & 0.89 & 10.0\;(\textcolor{mgreen}{-7.0})& 1.88 \\
\text{ } & Hamming &  IBPMedium & IBP & 0.30 & 12.0 & 0.02 & 11.0 & 0.87 & 3.0\;(\textcolor{mgreen}{-8.0})& 1.75 \\
\midrule
\end{tabular}
}
\end{table*}
\vspace{-3mm}
\begin{figure*}[htb]
\centering
\begin{minipage}[b]{.23\textwidth}
\includegraphics[width=\textwidth]{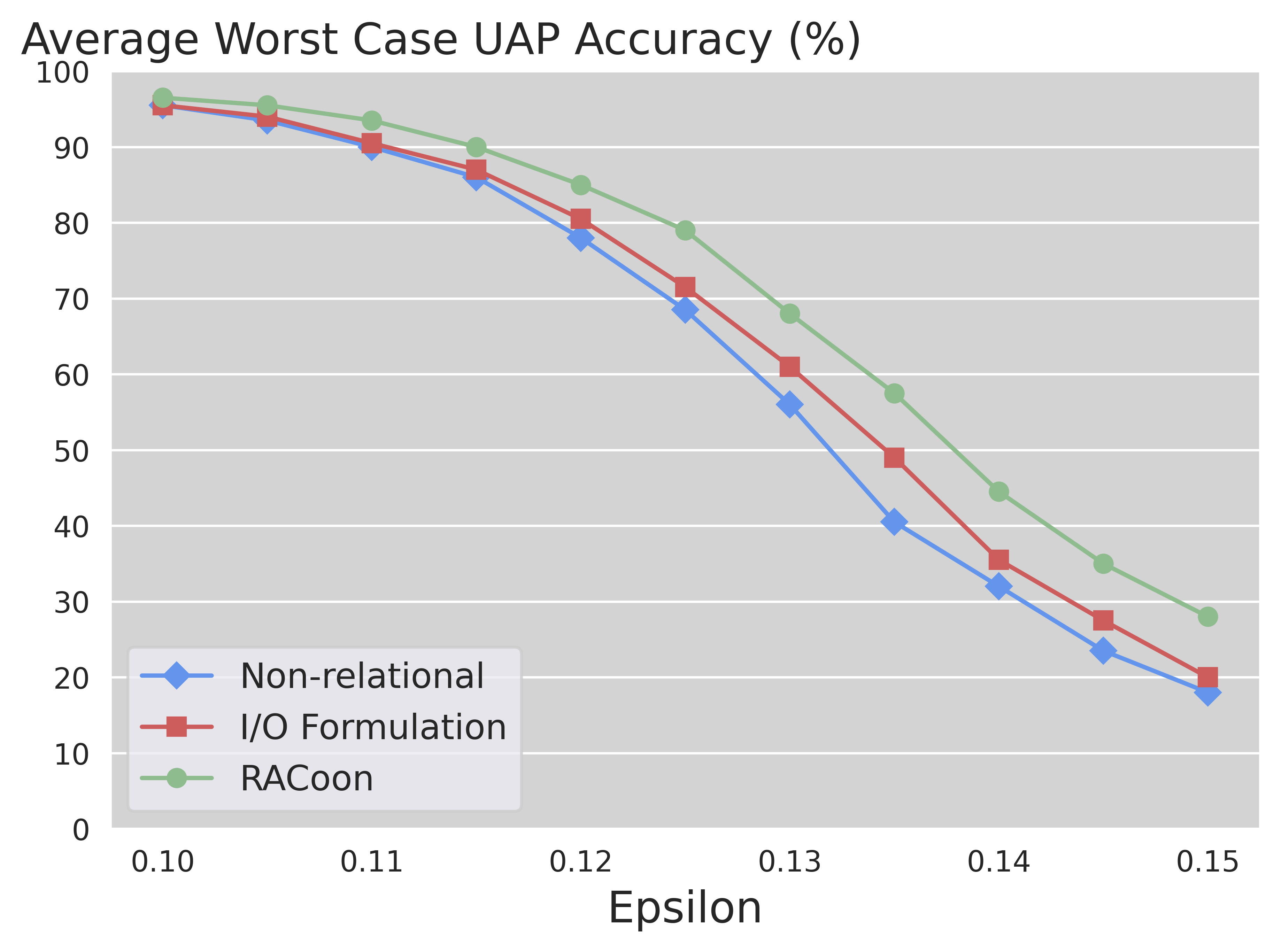}
\captionsetup{labelformat=empty}
\caption{(a) DiffAI (MNIST)}
\end{minipage}
\addtocounter{figure}{-1}
\begin{minipage}[b]{.23\textwidth}
\includegraphics[width=\textwidth]{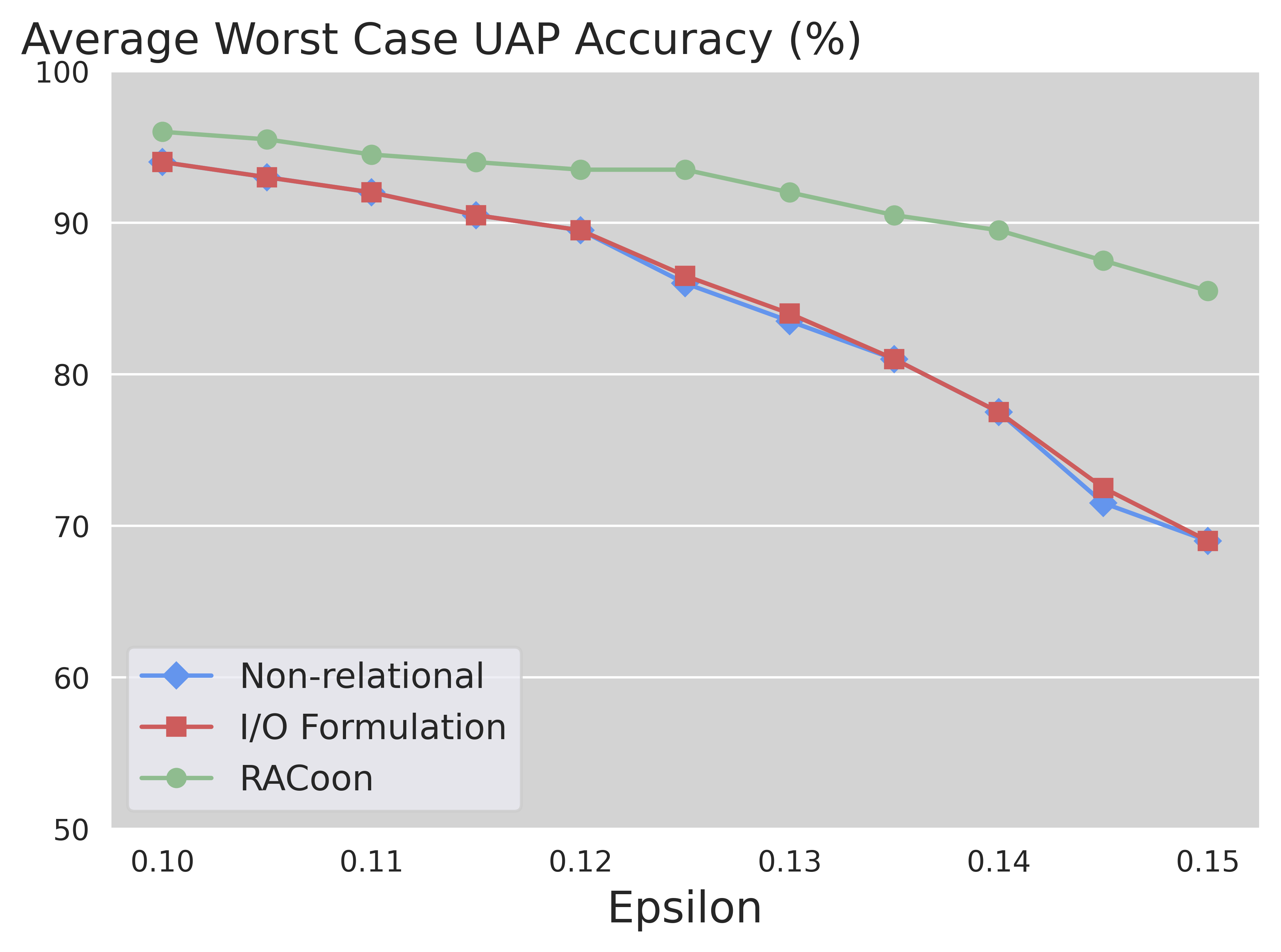}
\captionsetup{labelformat=empty}
\caption{(b) IBPSmall (MNIST)}
\end{minipage}\qquad
\addtocounter{figure}{-1}
\begin{minipage}[b]{.23\textwidth}
\includegraphics[width=\textwidth]{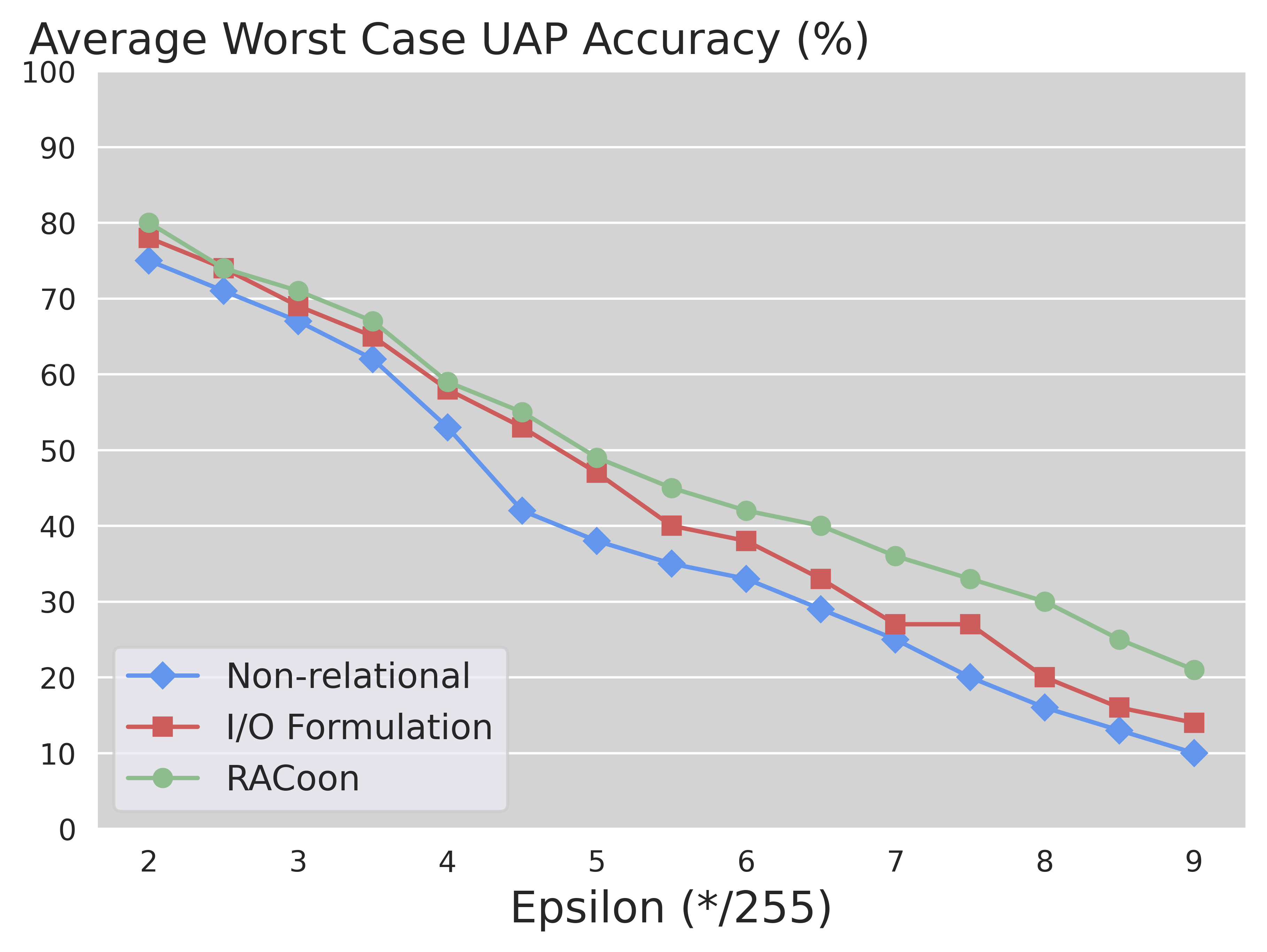}
\captionsetup{labelformat=empty}
\caption{(c) DiffAI (CIFAR10)}
\end{minipage}
\addtocounter{figure}{-1}
\begin{minipage}[b]{.23\textwidth}
\includegraphics[width=\textwidth]{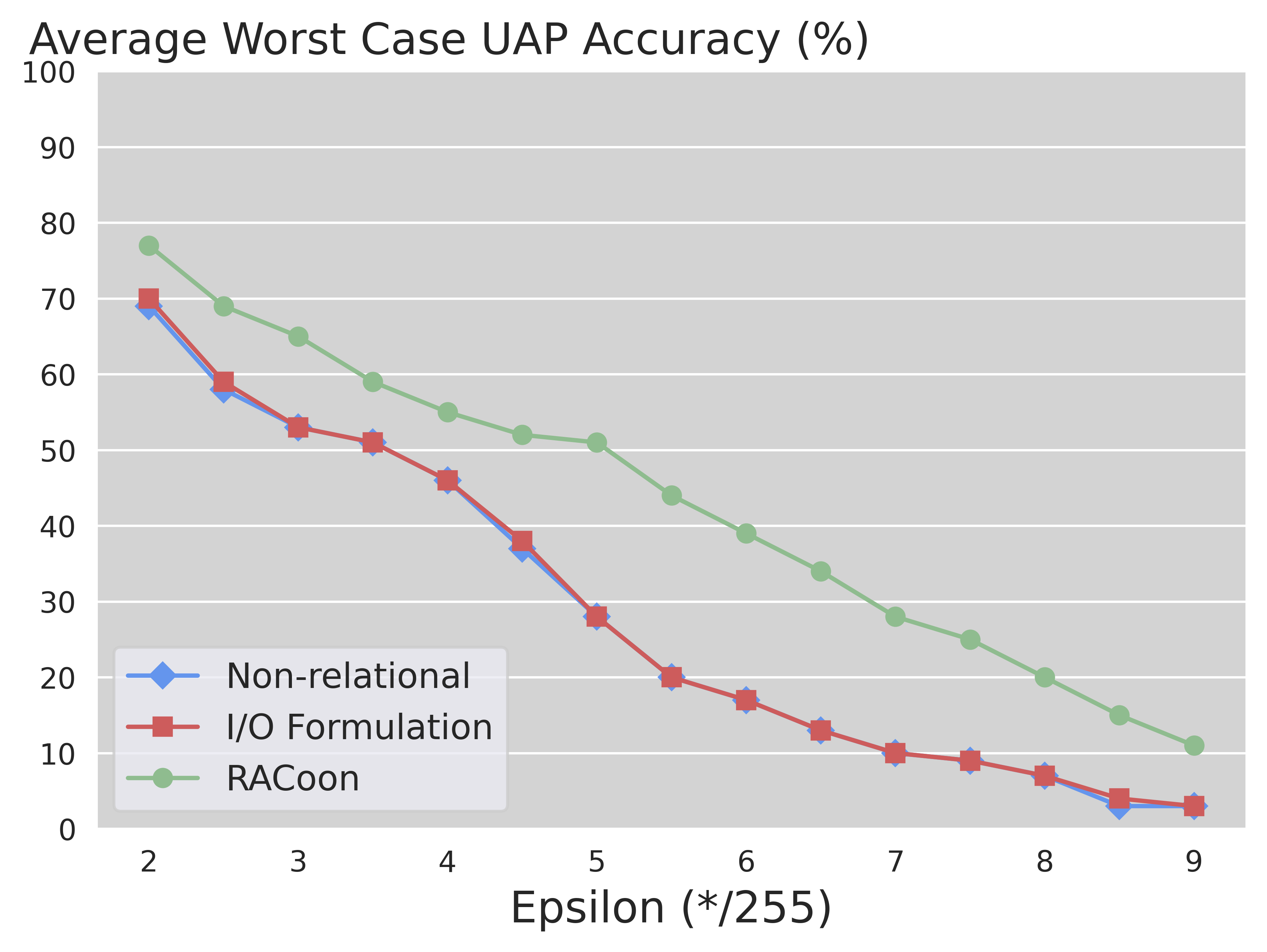}
\captionsetup{labelformat=empty}
\caption{(d) IBPSmall (CIFAR10)}
\end{minipage}
\addtocounter{figure}{-1}
\vspace{-1mm}
\caption{Average Worst case UAP accuracy for different $\epsilon$ values for ConvSmall (DiffAI) and IBPSmall DNNs.}
\label{fig:mnist_eps}
\end{figure*}

\textbf{Verification results: } For $k$-UAP, both the baselines: non-relational verifier \cite{auto_lirpa}, I/O formulation \cite{iclrUap} and \Tool computes a provably correct lower bound $\apprxOutput(\inspecset, \outspecset)$ on the worst-case UAP accuracy. In this case, larger $\apprxOutput(\inspecset, \outspecset)$ values produce a more precise lower bound tightly approximating the actual worst-case UAP accuracy. In contrast, for worst-case hamming distance $\apprxOutput(\inspecset, \outspecset)$ is a provably correct upper bound and smaller $\apprxOutput(\inspecset, \outspecset)$ values are more precise. Table~\ref{table:compareoriginal} shows the verification results on different datasets (column 1), DNN architectures (column 3) trained with different training methods (column 4) where $\epsilon$ values defining $L_{\infty}$ bound of $\pmb{\delta}$ are in column 5. The relational properties: $k$-UAP and worst-case hamming distance on MNIST DNNs use $k=20$ while $k$-UAP on CIFAR10 DNNs uses $k=10$. For each DNN and $\epsilon$, we run relational verification on $k$ randomly selected inputs and repeat the experiment $10$ times. We report worst-case UAP accuracy and worst-case hamming distance averaged over all $10$ runs. Results in Table~\ref{table:compareoriginal} substantiate that \Tool outperforms current SOTA baseline I/O formulation on all DNNs for both the relational properties. \Tool gains up to $+16.5\%$ and up to $+22\%$ improvement in the worst-case UAP accuracy (averaged over 10 runs) for MNIST and CIFAR10 DNNs respectively. Similarly, \Tool reduces the worst-case hamming distance (averaged over 10 runs) up to $8$ which is up to $40\%$ reduction for binary strings of size $20$.

\textbf{Runtime analysis: } Table~\ref{table:compareoriginal} shows that \Tool is slower than I/O formulation. However, even for ConvBig architectures, \Tool takes less than $8$ seconds (for $20$ executions) for MNIST and takes less than $12$ seconds (for $10$ executions) for CIFAR10. The timings are much smaller compared to the timeouts allotted for similar architectures in the SOTA competition for verification of DNNs (VNN-Comp \cite{vnncomp}) ($200$ seconds per execution).

\textbf{\Tool componentwise analysis: } In Appendix Table~\ref{table:compareComponents}, we show results for different components of \Tool including individual bound refinement using $\naivedual$ (Eq.~\ref{eq:dualNaive}), individual bound refinement with MILP formulation, and cross-executional bound refinement without MILP formulation. Note that only cross-executional bound refinement without MILP can prove the absence of common adversarial perturbation for a set of executions even if non-relational verification fails on all of them. Hence, even without MILP, cross-executional bound refinement serves as a promising approach for relational verification. Appendix Table~\ref{table:compareComponents} shows for some cases (i.e. MNIST and CIFAR10 standard DNNs) I/O formulation (static linear approximation with MILP) outperforms individual refinements while both individual refinement with MILP and cross-execution refinement are always more precise. As expected, \Tool (cross-execution refinement with MILP) yields the most precise results while cross-execution refinement without MILP achieves the second-best results with notably faster runtime. Componentwise runtime analysis is in Appendix~\ref{appenexp:additionalComponentRuntime}.

\textbf{Different $\epsilon$ and $k$ values: } Fig.~\ref{fig:mnist_eps} and Appendix Fig.~\ref{fig:mnist_eps_appen},~\ref{fig:cifar_eps_appen} show the results of \Tool and both the baselines on relational properties defined with different $\epsilon$ values on DNNs from Table~\ref{table:compareoriginal}. We also analyze the performance of \Tool for $k$-UAP verification defined with different $k$ and $\epsilon$ values in Appendix~\ref{appenexp:diffKEps} on DNNs from Table~\ref{table:compareoriginal}. For the MNIST DNNs, we consider up to $50$ executions, and for CIFAR10 DNNs we consider up to $25$ executions per property. For all $k$ and $\epsilon$ values \Tool is more precise than both baselines. In all cases, even for ConvBig MNIST and CIFAR10, \Tool takes less than $16$ and $25$ seconds respectively.

\textbf{Ablation on hyperparameters $k_0$ and $k_1$: } We analyze the impact of $k_0$ and $k_1$ on performance of \Tool in Appendix~\ref{appenexp:ablation_with_k0_k1}. As expected, with larger $k_0$ and $k_1$ \Tool's precision improves but it also increases \Tool's runtime.

\vspace{-4mm}
\section{Conclusion}
\vspace{-3mm}
In this work, we present \Tool, a general framework for improving the precision of relational verification of DNNs through cross-executional bound refinement.  
Our experiments, spanning various relational properties, DNN architectures, and training methods demonstrate the effectiveness of utilizing dependencies across multiple executions. Furthermore, \Tool with cross-executional bound refinement proves to exceed the capabilities of the current state-of-the-art relational verifier \cite{iclrUap}.
While our focus has been on relational properties within the same DNN across multiple executions, \Tool can be extended to properties involving different DNNs, such as local equivalence of DNN pairs \cite{reluDiff} or properties defined over an ensemble of DNNs. Additionally, \Tool can be leveraged for training DNNs on relational properties. We leave these extensions as future work.
\clearpage
\newpage
\section{Impact and Ethics}
\label{sec:impactStatement}
This paper introduces research aimed at advancing the field of Machine Learning. We do not identify any specific societal consequences of our work that need to be explicitly emphasized here.
\bibliography{main}
\bibliographystyle{icml2024}
\appendix
\onecolumn
\section{Formal encoding of relational properties}
\label{sec:FormalDefinition}
\subsection{k-UAP verification}
\label{appendix:uapFormulation}
Given a set of $k$ points $\vect{X} =\{ \vect{x_1}, ... ,\vect{x_k}\} $  where for all $i \in [k]$, $\vect{x_i} \in \mathbb{R}^{n_{0}}$ and $\epsilon \in \mathbb{R}$ we can first define individual input constraints used to define $L_{\infty}$ input region for each execution $\forall i \in [k].\phi_{in}^i(\vect{x^{*}_i}) = \|\vect{x^{*}_i} - \vect{x_i}\|_\infty \leq \epsilon$. We define $\inspecset^\delta(\vect{x^{*}_1}, \dots, \vect{x^{*}_k})$ as follows:

\begin{equation}\label{eq:UAPinsetdef}
    \inspecset^\delta(\vect{x^{*}_1}, \dots, \vect{x^{*}_k}) = \bigwedge_{(i, j \in [k]) \wedge (i < j)} (\vect{x^{*}_i} - \vect{x^{*}_j} = \vect{x_i} - \vect{x_j})
\end{equation}

Then, we have the input specification as $\inspecset(\vect{x^{*}_1}, \dots, \vect{x^{*}_k}) = \bigwedge_{i=1}^k \phi_{in}^i(\vect{x^{*}_i}) \wedge \crossinspecset(\vect{x^{*}_1}, \dots, \vect{x^{*}_k})$. 

Next, we define $\outspecset(\vect{x^{*}_1}, \dots, \vect{x^{*}_k})$ as conjunction of $k$ clauses each defined by $\outpropparam{i}(\vect{y_i})$ where $\vect{y_i} = N(\vect{x^{*}_i})$.
Now we define $\outpropparam{i}(\vect{y_i}) = \bigwedge_{j=1}^{n_{l}} (\vect{c_{i,j}}^T\vect{y_i} \geq 0)$ where $\vect{c_{i,j}} \in \mathbb{R}^{n_{l}}$ is defined as follows

\begin{equation}
    \forall a \in [n_l].c_{i,j,a} = \begin{cases}
        1 &\text{if } a \neq j \text{ and } a \text{ is the correct label for } \vect{y_i}  \\
        -1 &\text{if } a = j \text{ and } a \text{ is not the correct label for } \vect{y_i}\\
        0 &\text{otherwise}
    \end{cases}
\end{equation}
In this case, the tuple of inputs $(\vect{x^{*}_1}, \dots, \vect{x^{*}_k})$ satisfies the input specification $\inspecset(\vect{x^{*}_1}, \dots, \vect{x^{*}_k})$ iff for all $i \in [k]$, $\vect{x^{*}_i} = \vect{x_{i}} + \pmb{\delta}$ where $\pmb{\delta} \in \mathbb{R}^{n_0}$ and $\|\pmb{\delta}\|_{\infty} \leq \epsilon$.  Hence, the relational property $(\inspecset, \outspecset)$ defined above verifies whether there is an adversarial perturbation $\pmb{\delta} \in \mathbb{R}^{n_0}$ with $\|\pmb{\delta}\|_{\infty} \leq \epsilon$ that can misclassify \textbf{all} $k$ inputs. Next, we show the formulation for the worst-case UAP accuracy of the k-UAP verification problem as described in section~\ref{sec:prelims}. 
Let, for any $\pmb{\delta} \in \mathbb{R}^{n_0}$ and $\|\pmb{\delta}\|_{\infty} \leq \epsilon$, $\mu(\delta)$ denotes the number of clauses ($\outpropparam{i}$) in $\outspecset$ that are satisfied. Then $\mu(\delta)$ is defined as follows
\begin{align}
   & z_{i}(\pmb{\delta}) = \begin{cases}
        1 &\outpropparam{i}(N(\vect{x_{i}} + \pmb{\delta}))\text{ is $True$} \label{eq:zdef} \\
        0 &\text{otherwise}
    \end{cases} \\
    &\mu(\pmb{\delta}) = \sum_{i=1}^{k}z_{i}(\pmb{\delta}) \label{eq:muDef}
\end{align}
Since $\outpropparam{i}(N(\vect{x_{i}} + \pmb{\delta}))$  is $True$
iff the perturbed input $\vect{x_{i}} + \pmb{\delta}$ is correctly classified by $N$, for any $\pmb{\delta} \in \mathbb{R}^{n_0}$ and $\|\pmb{\delta}\|_{\infty} \leq \epsilon$, $\mu(\pmb{\delta})$ captures the number of correct classifications over the set of perturbed inputs $\{\vect{x_{1}} + \pmb{\delta}, \dots, \vect{x_{k}} + \pmb{\delta}\}$. The worst-case k-UAP accuracy $\Output(\inspecset, \outspecset)$ for $(\inspecset, \outspecset)$ is as follows
\begin{align}
\label{eq:kUAPprobFormulation}
\Output(\inspecset, \outspecset) = \min\limits_{\pmb{\delta} \in \mathbb{R}^{n_0},\;\|\pmb{\delta}\| \leq \epsilon} \mu(\pmb{\delta}) 
\end{align}
\subsection{Worst case Hamming distance verification}
\label{appendix:hammingFormulation}
We consider a set of $k$ unperturbed inputs $\vect{X} = \{\vect{x_1}, ... ,\vect{x_k}\}$ where for all $i \in [k]$, $\vect{x_{i}} \in \mathbb{R}^{n_0}$, a peturbation budget $\epsilon \in \mathbb{R}$, and a binary digit classifier neural network $N_2: \mathbb{R}^{n_0} \rightarrow \mathbb{R}^2$.
We can define a binary digit string $\vect{S^*} \in \{0,1\}^k$ as a sequence of binary digits where each input $\vect{x_i}$ to $N_2$ is an image of a binary digit. We are interested in bounding the worst-case hamming distance between $\vect{S}$, the binary digit string classified by $N_2$, and $\vect{S^*}$ the actual binary digit string corresponding to the list of perturbed images $\forall i \in [k]. \vect{x^{*}_i} = \vect{x_i} + \pmb{\delta}$ s.t. $\pmb{\delta} \in \mathbb{R}^{n_0}$ and $\|\pmb{\delta}\|_\infty \leq \epsilon$. Given these definitions, we can use the $\inspecset$, $\outspecset$ and $\mu(\pmb{\delta})$ defined in section~\ref{appendix:uapFormulation} defined for k-UAP verification. In this case, the worst case hamming distance $\Output(\inspecset, \outspecset)$ is defined as $\Output(\inspecset, \outspecset) = k - \min\limits_{\pmb{\delta} \in \mathbb{R}^{n_0},\;\|\pmb{\delta}\| \leq \epsilon} \mu(\pmb{\delta})$.

\section{Theorectical guarantees for cross-execution bound refinement}
We obtain the theoretical guarantees of cross-execution bound refinement over $n$ executions. Note that we do not show the theoretical guarantees for a pair of executions separately as it is just a special case with $n=2$.
\subsection{Theorectical guarantees for $n$ of executions}
\label{sec:nTheorems}
\subsubsection{Theorems for LP formulation}
\label{appendix:theoremlpPair}
First, we show the correctness of the LP formulation in Eq.~\ref{eq:lpFormnExecution} or for 
 pair of execution in Eq.~\ref{eq:lpForm} (Theorem~\ref{thm:pairCorrect}). We also show that for fixed linear approximations $\{(\vect{L_1}, b_1), \dots, (\vect{L_n}, b_n)\}$ of $N$,  the LP formulation is exact i.e. it always proves the absence of common adversarial perturbation if it does not exist (Theorem~\ref{thm:pairExact}). In this case,  $\outspecset(\vect{y_1}, \dots, \vect{y_n}) = \bigvee_{i=1}^{n}(\vect{c_i}^T\vect{y_i} \geq 0)$ where the outputs of $N$ are $\vect{y_i} = N(\vect{x_i} + \pmb{\delta})$.
Let, $\opt{t}$ be the optimal solution of the LP in Eq.~\ref{eq:lpFormnExecution}.
\begin{lemma}
\label{lem:pairCorrect}
$\opt{t} = \min\limits_{\pmb{\delta} \in \mathbb{R}^{n_0},\;\|\pmb{\delta}\|_{\infty} \leq \epsilon} \max\limits_{1 \leq i \leq n}\vect{L_i}^{T}(\vect{x_i} + \pmb{\delta}) + b_i$.
\end{lemma}
\begin{proof}
$\opt{t} = \min\limits_{\pmb{\delta} \in \mathbb{R}^{n_0},\;\|\pmb{\delta}\| \leq \epsilon} t(\pmb{\delta})$ where if $\|\pmb{\delta}\|_{\infty} \leq \epsilon$ then $t(\delta)$ satisfies the following constraints $t(\delta) \geq \vect{L_i}^{T}(\vect{x_i} + \pmb{\delta}) + b_i$ for all $i \in [n]$  then $t(\delta) \geq \max\limits_{1 \leq i \leq n}\vect{L_i}^{T}(\vect{x_i} + \pmb{\delta}) + b_i$. 
Let, $\opt{l} = \min\limits_{\pmb{\delta} \in \mathbb{R}^{n_0},\;\|\pmb{\delta}\| \leq \epsilon} \max\limits_{1 \leq i \leq n}\vect{L_i}^{T}(\vect{x_i} + \pmb{\delta}) + b_i$.
\begin{align}
   \opt{t} \geq  \min\limits_{\pmb{\delta} \in \mathbb{R}^{n_0},\;\|\pmb{\delta}\|_{\infty} \leq \epsilon} \max_{1 \leq i \leq n}\vect{L_i}^{T}(\vect{x_i} + \pmb{\delta}) + b_i = \opt{l}\label{eq:optTgreater}
\end{align}
Next, we show that $\opt{l} \geq \opt{t}$. $\opt{l} = \max_{1 \leq i \leq n}\vect{L_i}^{T}(\vect{x_i} + \opt{\pmb{\delta}}) + b_i$ for some $\opt{\pmb{\delta}}$ where $\opt{\pmb{\delta}} \in \mathbb{R}^{n_0}$ and $\|\opt{\pmb{\delta}}\|_{\infty} \leq \epsilon$, then $\opt{l}$ satisfies the constraints $\opt{l} \geq \vect{L_i}^{T}(\vect{x_i} + \opt{\pmb{\delta}}) + b_i$ for all $i \in [n]$. Since $\opt{l}$ is a valid feasible solution of the LP in Eq.~\ref{eq:lpFormnExecution} then $\opt{l} \geq \opt{t}$ as $\opt{t}$ is the optimal solution of the LP.

$\opt{l} \geq \opt{t}$ and from Eq.~\ref{eq:optTgreater} $\opt{l} \leq \opt{t}$ implies $\opt{l} = \opt{t}$.
\end{proof}

\begin{theorem}
\label{thm:pairCorrect}
For all $\pmb{\delta} \in \mathbb{R}^{n_0}$ and $\|\pmb{\delta}\|_{\infty} \leq \epsilon$, if for all $i \in [n]$, $\vect{L_i}^{T} (\vect{x_i} + \pmb{\delta}) + b_i \leq \vect{c_i}^T\vect{y_i}$ then $(\opt{t} \geq 0) \implies \left(\forall \pmb{\delta} \in \mathbb{R}^{n_0}. (\|\pmb{\delta}\|_{\infty} \leq \epsilon) \implies \outspecset(\vect{y_1},\dots, \vect{y_n})\right)$ holds.  
\end{theorem}
\begin{proof}
 Since, for all $i \in [n]$, $\vect{L_i}^{T} (\vect{x_i} + \pmb{\delta}) + b_i \leq \vect{c_i}^T\vect{y_i}$, for all $\pmb{\delta} \in \mathbb{R}^{n_0}$ and $\|\pmb{\delta}\|_{\infty} \leq \epsilon$, then $\min\limits_{\pmb{\delta} \in \mathbb{R}^{n_0},\;\|\pmb{\delta}\|_{\infty} \leq \epsilon} \max\limits_{1 \leq i \leq n}\vect{L_i}^{T}(\vect{x_i} + \pmb{\delta}) + b_i  \leq \min\limits_{\pmb{\delta} \in \mathbb{R}^{n_0},\;\|\pmb{\delta}\|_{\infty} \leq \epsilon} \max\limits_{1 \leq i \leq n} \vect{c_i}^{T}\vect{y_i}$
\begin{align*}
    &\opt{t} = \min\limits_{\pmb{\delta} \in \mathbb{R}^{n_0},\;\|\pmb{\delta}\|_{\infty} \leq \epsilon} \max_{1 \leq i \leq n}\vect{L_i}^{T}(\vect{x_i} + \pmb{\delta}) + b_i  \leq \min\limits_{\pmb{\delta} \in \mathbb{R}^{n_0},\;\|\pmb{\delta}\|_{\infty} \leq \epsilon} \max_{1 \leq i \leq n} \vect{c_i}^{T}\vect{y_i} \;\;\;\text{Using lemma~\ref{lem:pairCorrect}}\\
    &(\opt{t} \geq 0) \implies \left(\min\limits_{\pmb{\delta} \in \mathbb{R}^{n_0},\;\|\pmb{\delta}\|_{\infty} \leq \epsilon} \max_{1 \leq i \leq n} \vect{c_i}^{T}\vect{y_i}\right) \geq 0 \\
    & (\opt{t} \geq 0) \implies \left(\forall \pmb{\delta} \in \mathbb{R}^{n_0}. (\|\pmb{\delta}\|_{\infty} \leq \epsilon) \implies \outspecset(\vect{y_1}, \dots, \vect{y_n})\right)
\end{align*}
\end{proof}

\begin{theorem}
\label{thm:pairExact}
$\left(\forall \pmb{\delta} \in \mathbb{R}^{n_0}. (\|\pmb{\delta}\|_{\infty} \leq \epsilon) \implies \bigvee_{i=1}^{n} (\vect{L_i}^{T} (\vect{x_i} + \pmb{\delta}) + b_i \geq 0)\right)$ holds if and only if $\opt{t} \geq 0$.
\end{theorem}
\begin{proof}
From lemma~\ref{lem:pairCorrect}, $\opt{t} = \min\limits_{\pmb{\delta} \in \mathbb{R}^{n_0},\;\|\pmb{\delta}\|_{\infty} \leq \epsilon} \max\limits_{1 \leq i \leq n}\vect{L_i}^{T}(\vect{x_i} + \pmb{\delta}) + b_i$.
\begin{align}
    (\opt{t} \geq 0) &\implies \left(\min\limits_{\pmb{\delta} \in \mathbb{R}^{n_0},\;\|\pmb{\delta}\|_{\infty} \leq \epsilon} \max_{1 \leq i \leq n}\vect{L_i}^{T}(\vect{x_i} + \pmb{\delta}) + b_i\right) \geq 0 \nonumber \\
    &\implies \left(\forall \pmb{\delta} \in \mathbb{R}^{n_0}. (\|\pmb{\delta}\|_{\infty} \leq \epsilon) \implies \bigvee_{i=1}^{n} (\vect{L_i}^{T} (\vect{x_i} + \pmb{\delta}) + b_i \geq 0)\right) \label{eq:directPair}
\end{align}
\begin{align}
    (\opt{t} < 0) &\implies \left(\min\limits_{\pmb{\delta} \in \mathbb{R}^{n_0},\;\|\pmb{\delta}\|_{\infty} \leq \epsilon} \max_{1 \leq i \leq 2}\vect{L_i}^{T}(\vect{x_i} + \pmb{\delta}) + b_i\right) < 0 \nonumber \\
    &\implies \left(\exists \pmb{\delta} \in \mathbb{R}^{n_0}. \bigwedge_{i=1}^{n} (\vect{L_i}^{T} (\vect{x_i} + \pmb{\delta}) + b_i < 0) \bigwedge (\|\pmb{\delta}\|_{\infty} \leq \epsilon)\right) \nonumber \\
    \neg(\opt{t} \geq 0) &\implies \neg\left(\forall \pmb{\delta} \in \mathbb{R}^{n_0}. (\|\pmb{\delta}\|_{\infty} \leq \epsilon) \implies \bigvee_{i=1}^{n} (\vect{L_i}^{T} (\vect{x_i} + \pmb{\delta}) + b_i \geq 0)\right) \label{eq:contraposPair}
\end{align}
Using Eq.~\ref{eq:directPair} and Eq.~\ref{eq:contraposPair}, $(\opt{t} \geq 0) \iff \left(\forall \pmb{\delta} \in \mathbb{R}^{n_0}. (\|\pmb{\delta}\|_{\infty} \leq \epsilon) \implies \bigvee_{i=1}^{n} (\vect{L_i}^{T} (\vect{x_i} + \pmb{\delta}) + b_i \geq 0)\right)$.
\end{proof}

\subsubsection{Details for computing the Lagrangian Dual}
\label{appendix:duallpPair}
Next, we provide the details for computing the Lagrangian Dual of the LP formulation in Eq.~\ref{eq:lpFormnExecution}. The Lagrangian Dual is as follows where for all $i \in [n]$, $\lambda_i \geq 0$ are Lagrange multipliers.
\begin{align*}
    \max\limits_{0 \leq \lambda_i} \min\limits_{t \in \mathbb{R}, \|\pmb{\delta}\|_{\infty} \leq \epsilon} (1 - \sum_{i=1}^{n} \lambda_{i}) \times t + \sum_{i=1}^{n} \lambda_i \times \left(\linapprx_{i}^T(\vect{x_i} + \pmb{\delta}) + b_{i}\right)
\end{align*}
We set the coefficient of the unbounded variable $t$ to $0$ to avoid cases where $ \min\limits_{t \in \mathbb{R}, \|\pmb{\delta}\|_{\infty} \leq \epsilon} (1 - \sum_{i=1}^{n} \lambda_{i}) \times t + \sum_{i=1}^{n} \lambda_i \times \left(\linapprx_{i}^T(\vect{x_i} + \pmb{\delta}) + b_{i}\right) = -\infty$. This leads to the following Lagrangian Dual form
\begin{align*}
    \max\limits_{0 \leq \lambda_i} \min\limits_{\|\pmb{\delta}\|_{\infty} \leq \epsilon} \sum_{i=1}^{n} \lambda_i \times \left(\linapprx_{i}^T(\vect{x_i} + \pmb{\delta}) + b_{i}\right) \;\;\;\text{where $\sum_{i=1}^{n} \lambda_{i} = 1$}
\end{align*}
For all $i \in [n]$, let parametric linear approximations of $N$ are specified by $(\linapprxalpha{i}, \biasalpha{i})$ then the Lagrangian Dual is as follows
\begin{align*}
    \max\limits_{0 \leq \lambda_i} \min\limits_{\|\pmb{\delta}\|_{\infty} \leq \epsilon} \sum_{i=1}^{n} \lambda_i \times \left(\linapprxalpha{i}^T(\vect{x_i} + \pmb{\delta}) + \biasalpha{i}\right) \;\;\;\text{where $\sum_{i=1}^{n} \lambda_{i} = 1$}
\end{align*}
\subsubsection{Theorems for cross-execution bound refinement over $n$ of executions}
\label{appendix:correctnessPair}
Let, the $\optappx{t}(\dual)$ denote the solution obtained by the optimization technique and $\optappx{\optparams}$ denote the value of $\optparams$ corresponding to $\optappx{t}(\dual)$. Note that $\optappx{t}(\dual)$ can be different from global maximum $\opt{t}(\dual)$ with $\opt{t}(\dual) > \optappx{t}(\dual)$. 
We show that if $\optappx{t}(\dual) \geq 0$ then $\forall \pmb{\delta} \in \mathbb{R}^{n_0}. (\|\pmb{\delta}\|_{\infty} \leq \epsilon) \implies \outspecset(\vect{y_1}, \dots \vect{y_{n}})$ holds where $\vect{y_i} = N(\vect{x_i} + \pmb{\delta})$ for all $i \in [n]$.  
First, we prove the correctness of the characterization of $\dual(\optparams)$. 
\begin{lemma}
\label{lem:dualLemma}
For all $i \in [n]$, $0 \leq \lambda_i \leq 1$, $\sum_{i=1}^{n} \lambda_{i} = 1$, $\vect{l}_{i} \preceq \param{i} \preceq \vect{u}_{i}$, if $\optparams = (\param{1}, \dots, \param{n}, \lambda_1, \dots, \lambda_{n})$ then $\forall \pmb{\delta} \in \mathbb{R}^{n_0}. (\|\pmb{\delta}\|_{\infty} \leq \epsilon) \implies (\dual(\optparams) = \min\limits_{\pmb{\delta} \in \mathbb{R}^{n_0},\; \|\pmb{\delta}\|_{\infty} \leq \epsilon} \sum\limits_{i=1}^{n} \lambda_i \times \left(\linapprxalpha{i}^T(\vect{x_i} + \pmb{\delta}) + \biasalpha{i}\right)$ where $\dual(\optparams) = \sum\limits_{i=1}^{n} \lambda_i \times \constalpha{i} - \epsilon \times  \sum\limits_{j=1}^{n_0}\left|\sum\limits_{i=1}^{n} \lambda_i \times \linapprxalpha{i}[j]\right|$ and $\constalpha{i} = \linapprxalpha{i}^T\vect{x_i} + \biasalpha{i}$.
\end{lemma}
\begin{proof}
First we rewrite $\dual(\optparams)$ in Eq.~\ref{eq:pairBoundProof}
and find the closed form on $\min\limits_{\pmb{\delta} \in \mathbb{R}^{n_0},\; \|\pmb{\delta}\|_{\infty} \leq \epsilon} \sum\limits_{i=1}^{n} \lambda_i \times \linapprxalpha{i}^T\pmb{\delta}$ in Eq.~\ref{eq:pairFinalproof}.
\begin{align}
    \sum_{i=1}^{n} \lambda_i \times \left(\linapprxalpha{i}^T(\vect{x_i} + \pmb{\delta}) + \biasalpha{i}\right) &= \sum_{i=1}^{n} \lambda_i \times \constalpha{i} + \sum_{i=1}^{n} \lambda_i \times \linapprxalpha{i}^T\pmb{\delta} \nonumber \\
\min\limits_{\pmb{\delta} \in \mathbb{R}^{n_0},\; \|\pmb{\delta}\|_{\infty} \leq \epsilon} \sum_{i=1}^{n} \lambda_i \times \left(\linapprxalpha{i}^T(\vect{x_i} + \pmb{\delta}) + \biasalpha{i}\right) &= \sum_{i=1}^{n} \lambda_i \times \constalpha{i} + \min\limits_{\pmb{\delta} \in \mathbb{R}^{n_0},\; \|\pmb{\delta}\|_{\infty} \leq \epsilon} \sum_{i=1}^{n} \lambda_i \times \linapprxalpha{i}^T\pmb{\delta} \label{eq:pairBoundProof}
\end{align}
Now for fixed $\param{i}$, both $\linapprxalpha{i}, \pmb{\delta} \in \mathbb{R}^{n_0}$ are constant real vectors. Suppose for $j \in [n_0]$, $\linapprxalpha{i}[j]$ and $\pmb{\delta}[j]$ denotes the $j$-th component of $\linapprxalpha{i}$ and $ \pmb{\delta}$ respectively. Then,
\begin{align}
    \linapprxalpha{i}^T\pmb{\delta} &= \sum_{j=1}^{n_0}\linapprxalpha{i}[j] \times \pmb{\delta}[j] \nonumber \\
    \sum_{i=1}^{n} \lambda_i \times \linapprxalpha{i}^T\pmb{\delta} &= \sum_{j=1}^{n_0}\left(\sum_{i=1}^{n} \lambda_{i} \times \linapprxalpha{i}[j]\right) \times \pmb{\delta}[j] \nonumber \\
     -\epsilon \times \left|\sum_{i=1}^{n} \lambda_{i} \times \linapprxalpha{i}[j]\right| &= \min\limits_{-\epsilon \leq \pmb{\delta}[j] \leq \epsilon}\left(\sum_{i=1}^{n} \lambda_{i} \times \linapprxalpha{i}[j]\right) \times \pmb{\delta}[j] \label{eq:intermediatePair} \\
     \min\limits_{\pmb{\delta} \in \mathbb{R}^{n_0},\; \|\pmb{\delta}\|_{\infty} \leq \epsilon} \sum_{i=1}^{n} \lambda_i \times \linapprxalpha{i}^T\pmb{\delta} &= \sum_{j=1}^{n_0} \min\limits_{-\epsilon \leq \pmb{\delta}[j] \leq \epsilon}\left(\sum_{i=1}^{n} \lambda_{i} \times \linapprxalpha{i}[j]\right) \times \pmb{\delta}[j] \label{eq:penultimate} \\
     \min\limits_{\pmb{\delta} \in \mathbb{R}^{n_0},\; \|\pmb{\delta}\|_{\infty} \leq \epsilon} \sum_{i=1}^{n} \lambda_i \times \linapprxalpha{i}^T\pmb{\delta} &=-\epsilon \times \sum_{j=1}^{n_0} \left|\sum_{i=1}^{n} \lambda_{i} \times \linapprxalpha{i}[j]\right| \;\;\text{using Eq~\ref{eq:intermediatePair} and Eq.~\ref{eq:penultimate}} \label{eq:pairFinalproof}
\end{align}
Combing Eq.~\ref{eq:pairBoundProof} and Eq.~\ref{eq:pairFinalproof} 
\begin{align*}
\min\limits_{\pmb{\delta} \in \mathbb{R}^{n_0},\; \|\pmb{\delta}\|_{\infty} \leq \epsilon} \sum_{i=1}^{n} \lambda_i \times \left(\linapprxalpha{i}^T(\vect{x_i} + \pmb{\delta}) + \biasalpha{i}\right) =  \sum_{i=1}^{n} \lambda_i \times \constalpha{i} -\epsilon \times \sum_{j=1}^{n_0} \left|\sum_{i=1}^{n} \lambda_{i} \times \linapprxalpha{i}[j]\right| = \dual(\optparams)
\end{align*}
\end{proof}

\begin{theorem}[Correctness of bound refinement over $n$ executions] If $\optappx{t}(\dual) \geq 0$ then $\left(\forall \pmb{\delta} \in \mathbb{R}^{n_0}. (\|\pmb{\delta}\| \leq \epsilon) \implies \outspecset(\vect{y_1}, \dots \vect{y_{n}})\right)$ holds where $\vect{y_i} = N(\vect{x_i} + \pmb{\delta})$ for all $i \in [n]$.
\end{theorem}
\begin{proof}
$\optappx{t}(\dual) = \dual(\optappx{\optparams})$ where $\optappx{\optparams} = (\opt{\param{1}}, \dots, \opt{\param{n}}, \opt{\lambda_{1}}, \dots, \opt{\lambda_{n}})$ and for all $i \in [n]$, $\vect{l_i} \preceq \opt{\param{i}} \preceq \vect{u_i}$, $0 \leq \opt{\lambda_i}\leq 1$, $\sum\limits_{i=1}^{n} \opt{\lambda_i} = 1$.  Then using lemma~\ref{lem:dualLemma} we get
\begin{align}
    &\dual(\optappx{\optparams}) = \min\limits_{\pmb{\delta} \in \mathbb{R}^{n_0},\; \|\pmb{\delta}\|_{\infty} \leq \epsilon} \sum_{i=1}^{n} \opt{\lambda_i} \times \left(\linapprxoptalpha{i}^T(\vect{x_i} + \pmb{\delta}) + \biasoptalpha{i}\right) \label{eq:apprxOptimalDual}
\end{align}
Next we show that $\dual(\optappx{\optparams}) \leq \min\limits_{\pmb{\delta} \in \mathbb{R}^{n_0}, \|\pmb{\delta}\|_{\infty} \leq \epsilon}\max\limits_{1 \leq i \leq n} \vect{c_i}^{T}\vect{y_i}$ where $\vect{y_i} = N(\vect{x_i} + \pmb{\delta})$.
\begin{align}
    &\left(\linapprxoptalpha{i}^T(\vect{x_i} + \pmb{\delta}) + \biasoptalpha{i}\right) \leq \vect{c_i}^{T}\vect{y_i} \;\;\;\;\text{$\forall i \in [n]$, $\vect{y_i} = N(\vect{x_i} + \pmb{\delta})$ and $\|\pmb{\delta}\|_{\infty} \leq \epsilon$} \nonumber \\
    &\left(\linapprxoptalpha{i}^T(\vect{x_i} + \pmb{\delta}) + \biasoptalpha{i}\right) \leq \max\limits_{1 \leq i \leq n}\vect{c_i}^{T}\vect{y_i} \;\;\;\;\;\;\;\;\text{$\forall i \in [n]$ and $\|\pmb{\delta}\|_{\infty} \leq \epsilon$} \nonumber\\
    &\sum_{i=1}^{n} \opt{\lambda_i} \times \left(\linapprxoptalpha{i}^T(\vect{x_i} + \pmb{\delta}) + \biasoptalpha{i}\right) \leq \max\limits_{1 \leq i \leq n}\vect{c_i}^{T}\vect{y_i} \times \sum_{i=1}^{n} \opt{\lambda_i}\;\;\;\;\;\;\;\;\text{since $\forall i \in [n]$, $\opt{\lambda_i} \geq 0$ and $\|\pmb{\delta}\|_{\infty} \leq \epsilon$} \nonumber\\
    &\sum_{i=1}^{n} \opt{\lambda_i} \times \left(\linapprxoptalpha{i}^T(\vect{x_i} + \pmb{\delta}) + \biasoptalpha{i}\right) \leq \max\limits_{1 \leq i \leq n}\vect{c_i}^{T}\vect{y_i} \;\;\;\;\;\;\;\;\text{since $\sum\limits_{i=1}^{n}\opt{\lambda_i} = 1$ and $\|\pmb{\delta}\|_{\infty} \leq \epsilon$} \nonumber\\
    &\dual(\optappx{\optparams}) = \min\limits_{\pmb{\delta} \in \mathbb{R}^{n_0},\; \|\pmb{\delta}\|_{\infty} \leq \epsilon} \sum_{i=1}^{n} \opt{\lambda_i} \times \left(\linapprxoptalpha{i}^T(\vect{x_i} + \pmb{\delta}) + \biasoptalpha{i}\right) \leq \min\limits_{\pmb{\delta} \in \mathbb{R}^{n_0}, \|\pmb{\delta}\|_{\infty} \leq \epsilon}\max\limits_{1 \leq i \leq n} \vect{c_i}^{T}\vect{y_i} 
    \label{eq:finalProofDual}
\end{align}
Using Eq.~\ref{eq:finalProofDual} we show that
\begin{align*}
(\optappx{t}(\dual) \geq 0) \implies (\dual(\optappx{\optparams}) \geq 0) \implies  \left(\min\limits_{\pmb{\delta} \in \mathbb{R}^{n_0}, \|\pmb{\delta}\|_{\infty} \leq \epsilon}\max\limits_{1 \leq i \leq n} \vect{c_i}^{T}\vect{y_i}\right) \geq 0  \\
\implies \left(\forall \pmb{\delta} \in \mathbb{R}^{n_0}. (\|\pmb{\delta}\|_{\infty} \leq \epsilon) \implies \outspecset(\vect{y_1}, \dots \vect{y_{n}})\right)
\end{align*}
\end{proof}
Similar to Theorem~\ref{thm:optimalBounds}, we show the optimal solution $\opt{t}(\dual)$ obtained with $\dual(\optparams)$ is always as good as $\opt{t}(\naivedual)$ i.e. $ \opt{t}(\dual) \geq \opt{t}(\naivedual)$ for $n$ executions.
\begin{theorem} 
If $\opt{t}(\dual) = \max_{\optparams} \dual(\optparams)$ and $\opt{t}(\naivedual) = \max\limits_{\param{1}, \dots, \param{n}} \naivedual(\param{1}, \dots, \param{n})$ then $\opt{t}(\naivedual) \leq \opt{t}(\dual)$.
\end{theorem}
\begin{proof}
For any $(\param{1}, \dots, \param{n})$ satisfying $\vect{l}_{i} \preceq \param{i} \preceq \vect{u}_{i}$ for all $i \in [n]$, we consider $\optparams_i = (\param{1}, \dots, \param{n}, \lambda_1=0, \dots, \lambda_{i} = 1, \dots, \lambda_n=0)$. 
Then $\dual(\optparams_i) = \min\limits_{\|\pmb{\delta}\|_{\infty} \leq \epsilon} \linapprxalpha{i}^T(\vect{x_i} + \pmb{\delta}) + \biasalpha{i}$. Since, $\opt{t}(\dual) \geq \dual(\optparams_i)$ for all $i \in [n]$ then $\opt{t}(\dual) \geq \max\limits_{1 \leq i \leq n}\dual(\optparams_i) = \naivedual(\param{1}, \dots, \param{n})$. Hence, $\opt{t}(\dual) \geq \max\limits_{\param{1}, \dots \param{n}} \naivedual(\param{1}, \param{2}) = \opt{t}(\naivedual)$.
\end{proof}

Next, we characterize one sufficient condition where $\opt{t}(\dual)$ is strictly better i.e. $\opt{t}(\dual) > \opt{t}(\naivedual)$. Note that Theorem~\ref{thm:strictBetter} shows one possible case where $\opt{t}(\dual)$ is strictly better and not the only possible condition where $\opt{t}(\dual) > \opt{t}(\naivedual)$ i.e. it is not necessary hold if $\opt{t}(\dual) > \opt{t}(\naivedual)$. Let, $(\opt{\param{1}}, \dots, \opt{\param{n}})$ be the optimal parameters corresponding to $\opt{t}(\naivedual)$.
\begin{theorem}
\label{thm:strictBetter}
If for all $i \in [n]$ there exists $j \in [n]$ such that $(\constoptalpha{j} - \constoptalpha{i})> \epsilon \times (\|\linapprxoptalpha{j}\|_{1} - \|\linapprxoptalpha{i}\|_{1})$ or $2\times\|\linapprxoptalpha{i}\|_{1} - \|\linapprxoptalpha{i} + \linapprxoptalpha{j}\|_{1} >  \frac{\constoptalpha{i}}{\epsilon} - \frac{\constoptalpha{j}}{\epsilon} $  holds then $\opt{t}(\dual) > \opt{t}(\naivedual)$.
\end{theorem}
\begin{proof}
Since $\opt{t}(\naivedual) = \max\limits_{1\leq i \leq k} \min\limits_{\pmb{\delta} \in \mathbb{R}^{n_0}, \|\pmb{\delta}\|_{\infty} \leq \epsilon} \linapprxoptalpha{i}^{T}\pmb{{\delta}} + \constoptalpha{i} = \max\limits_{1\leq i \leq k} -\epsilon \times \left(\sum_{j =1}^{n_0} \left|\linapprxoptalpha{i}[j]\right|\right) + \constoptalpha{i}$. This implies
$\opt{t}(\naivedual) = \max\limits_{1\leq i \leq k} - \epsilon \times \|\linapprxoptalpha{i}\|_{1} + \constoptalpha{i}$. Now for any $i_0 \in [n]$ if $\opt{t}(\naivedual) = - \epsilon \times \|\linapprxoptalpha{i_0}\|_{1} + \constoptalpha{i_0}$ (there exists at least one such $i_0$) then
\begin{align}
    &-\epsilon \times \|\linapprxoptalpha{i_0}\|_{1} + \constoptalpha{i_0} \geq - \epsilon \times \|\linapprxoptalpha{j}\|_{1} + \constoptalpha{j} \;\;\;\text{$\forall j \in [n]$} \nonumber \\
  &2\times\|\linapprxoptalpha{i_0}\|_{1} - \|\linapprxoptalpha{i_0} + \linapprxoptalpha{j_0}\|_{1} >  \frac{\constoptalpha{i_0}}{\epsilon} - \frac{\constoptalpha{j_0}}{\epsilon}  \;\;\text{for some $j_0 \in [n]$} \nonumber\\
  &\frac{1}{2}\times\left(- \epsilon \times (\|\linapprxoptalpha{i_0} + \linapprxoptalpha{j_0}\|_{1}) + \constoptalpha{i_0} + \constoptalpha{j_0} \right) > - \epsilon \times \|\linapprxoptalpha{i_0}\|_{1} + \constoptalpha{i_0} = \opt{t}(\naivedual) \label{eq:intermediateStrict}
\end{align}
$\opt{t}(\dual) = \max\limits_{\optparams} \dual(\optparams)$ now consider $\overline{\optparams} = (\opt{\param{1}}, \dots, \opt{\param{m}}, \lambda_1=0, \dots, \lambda_{i_0} = \frac{1}{2}, \dots, \lambda_{j_0} = \frac{1}{2}, \dots \lambda_n=0)$
\begin{align*}
\opt{t}(\dual) \geq \dual(\overline{\optparams}) = \frac{1}{2}\times\left(- \epsilon \times (\|\linapprxoptalpha{i_0} + \linapprxoptalpha{j_0}\|_{1}) + \constoptalpha{i_0} + \constoptalpha{j_0} \right) \\
\opt{t}(\dual) > - \epsilon \times \|\linapprxoptalpha{i_0}\|_{1} + \constoptalpha{i_0} = \opt{t}(\naivedual) \;\;\;\text{Using Eq.~\ref{eq:intermediateStrict}}
\end{align*}
\end{proof}
One simple example where this sufficient condition holds is $\constoptalpha{i} = \constoptalpha{j} = 0$ and $\linapprxoptalpha{i_0} = -\linapprxoptalpha{j_0}$ and $-\linapprxoptalpha{i_0}$ and $-\linapprxoptalpha{j_0}$ are non-zero vectors. 

\subsection{Cross-execution bound refinement for conjunction of linear inequalities}
We consider $n$ executions of $N$ on perturbed inputs given by $\{\vect{x_1} + \pmb{\delta}, \dots, \vect{x_n} + \pmb{\delta}\}$. In this case, to prove the absence of \textbf{common} adversarial perturbation we need to show for all $i \in [n]$ the outputs $\vect{y_i} = N(\vect{x_i} + \pmb{\delta})$ satisfy $\outspecset(\vect{y_1}, \dots, \vect{y_n}) = \bigvee_{i=1}^{n} \outpropparam{i}(\vect{y_i})$. Here, $\outpropparam{i}(\vect{y_i}) = \bigwedge_{j=1}^{m} (\vect{c_{i, j}}^{T}\vect{y_{i}} \geq 0)$ and $\vect{c_{i, j}} \in \mathbb{R}^{n_l}$. First, we prove lemmas necessary for characterizing the optimizable closed form that can be used for bound refinement. 
\begin{lemma}
\label{lem:multiple1}
$\forall \pmb{\delta} \in \mathbb{R}^{n_0}. \left( (\|\pmb{\delta}\|_{\infty} \leq \epsilon)\implies \outspecset(\vect{y_1, \dots, \vect{y_n}})\right)$ if and only if $\left(\min\limits_{\pmb{\delta} \in \mathbb{R}^{n_0},  (\|\pmb{\delta}\|_{\infty} \leq \epsilon)}\max\limits_{1 \leq i \leq n} \min\limits_{1 \leq j \leq m} \vect{c_{i,j}}^{T}\vect{y_i}\right) \geq 0$ where for all $i \in [n]$, $\vect{y_i} = N(\vect{x_i} + \pmb{\delta})$, $\outspecset(\vect{y_1}, \dots, \vect{y_n}) = \bigvee_{i=1}^{n} \outpropparam{i}(\vect{y_i})$ and $\outpropparam{i}(\vect{y_i}) = \bigwedge_{j=1}^{m} (\vect{c_{i, j}}^{T}\vect{y_{i}} \geq 0)$.
\end{lemma}
\begin{proof}
We first show if $\left(\min\limits_{\pmb{\delta} \in \mathbb{R}^{n_0},  (\|\pmb{\delta}\|_{\infty} \leq \epsilon)}\max\limits_{1 \leq i \leq n} \min\limits_{1 \leq j \leq m} \vect{c_{i,j}}^{T}\vect{y_i}\right) \geq 0$ then $\forall \pmb{\delta} \in \mathbb{R}^{n_0}. \left( (\|\pmb{\delta}\|_{\infty} \leq \epsilon)\implies \outspecset(\vect{y_1, \dots, \vect{y_n}})\right)$.
\begin{align}
\left(\min\limits_{\pmb{\delta} \in \mathbb{R}^{n_0},  (\|\pmb{\delta}\|_{\infty} \leq \epsilon)}\max\limits_{1 \leq i \leq n} \min\limits_{1 \leq j \leq m} \vect{c_{i,j}}^{T}\vect{y_i}\right) \geq 0 &\implies (\forall \pmb{\delta} \in \mathbb{R}^{n_0}. (\|\pmb{\delta}\|_{\infty} \leq \epsilon)\implies (\max\limits_{1 \leq i \leq n} \min\limits_{1 \leq j \leq m} \vect{c_{i,j}}^{T}\vect{y_i}) \geq 0)  \nonumber\\
&\implies (\forall \pmb{\delta} \in \mathbb{R}^{n_0}. (\|\pmb{\delta}\|_{\infty} \leq \epsilon)\implies \vee_{i=1}^{n} ((\min\limits_{1 \leq j \leq m} \vect{c_{i,j}}^{T}\vect{y_i}) \geq 0))
 \nonumber\\
&\implies \left(\forall \pmb{\delta} \in \mathbb{R}^{n_0}. (\|\pmb{\delta}\|_{\infty} \leq \epsilon)\implies \vee_{i=1}^{n} \wedge_{j=1}^{m}(\vect{c_{i,j}}^{T}\vect{y_i} \geq 0)
\right) \nonumber\\
&\implies\left(\forall \pmb{\delta} \in \mathbb{R}^{n_0}. \left( (\|\pmb{\delta}\|_{\infty} \leq \epsilon)\implies \outspecset(\vect{y_1, \dots, \vect{y_n}})\right)\right) 
\end{align}
Next, we show if $\forall \pmb{\delta} \in \mathbb{R}^{n_0}. \left( (\|\pmb{\delta}\|_{\infty} \leq \epsilon)\implies \outspecset(\vect{y_1, \dots, \vect{y_n}})\right)$ then $\left(\min\limits_{\pmb{\delta} \in \mathbb{R}^{n_0},  (\|\pmb{\delta}\|_{\infty} \leq \epsilon)}\max\limits_{1 \leq i \leq n} \min\limits_{1 \leq j \leq m} \vect{c_{i,j}}^{T}\vect{y_i}\right) \geq 0$.
\begin{align}
\left(\min\limits_{\pmb{\delta} \in \mathbb{R}^{n_0},  (\|\pmb{\delta}\|_{\infty} \leq \epsilon)}\max\limits_{1 \leq i \leq n} \min\limits_{1 \leq j \leq m} \vect{c_{i,j}}^{T}\vect{y_i}\right) < 0 &\implies  (\exists \pmb{\delta} \in \mathbb{R}^{n_0}. (\|\pmb{\delta}\|_{\infty} \leq \epsilon) \wedge ((\max\limits_{1 \leq i \leq n} \min\limits_{1 \leq j \leq m} \vect{c_{i,j}}^{T}\vect{y_i}) < 0))\nonumber \\
&\implies (\exists \pmb{\delta} \in \mathbb{R}^{n_0}. (\|\pmb{\delta}\|_{\infty} \leq \epsilon) \wedge \neg (\vee_{i=1}^{n} \outpropparam{i}(\vect{y_i})))\nonumber \\
&\implies \neg (\forall \pmb{\delta} \in \mathbb{R}^{n_0}.  (\|\pmb{\delta}\|_{\infty} \leq \epsilon)\implies \outspecset(\vect{y_1, \dots, \vect{y_n}})) \label{eq:finalLemmam}
\end{align}
Eq.~\ref{eq:finalLemmam} is equivalent to showing the following
\begin{align*}
\left(\forall \pmb{\delta} \in \mathbb{R}^{n_0}. \left( (\|\pmb{\delta}\|_{\infty} \leq \epsilon)\implies \outspecset(\vect{y_1, \dots, \vect{y_n}})\right)\right) \implies \left(\min\limits_{\pmb{\delta} \in \mathbb{R}^{n_0},  (\|\pmb{\delta}\|_{\infty} \leq \epsilon)}\max\limits_{1 \leq i \leq n} \min\limits_{1 \leq j \leq m} \vect{c_{i,j}}^{T}\vect{y_i}\right) \geq 0    
\end{align*}
\end{proof}

\begin{lemma}
\label{lem:multiple2}
$\min\limits_{\pmb{\delta} \in \mathbb{R}^{n_0},  (\|\pmb{\delta}\|_{\infty} \leq \epsilon)}\max\limits_{1 \leq i \leq n} \min\limits_{1 \leq j \leq m} \vect{c_{i,j}}^{T}\vect{y_i} = \min\limits_{j_1 \in [m], \dots, j_{n} \in [m]}S(j_1, \dots, j_{n})$
where for all $i \in [n]$ and $j_{i} \in [m]$ $S(j_1, \dots, j_{n})$ is defined as $S(j_1, \dots, j_{n}) = \min\limits_{\pmb{\delta} \in \mathbb{R}^{n_0},  (\|\pmb{\delta}\|_{\infty} \leq \epsilon)}\max\limits_{1 \leq i \leq n} \vect{c_{i,j_{i}}}^{T}\vect{y_i}$.
\end{lemma}
\begin{proof}
First, we show $\min\limits_{\pmb{\delta} \in \mathbb{R}^{n_0},  (\|\pmb{\delta}\|_{\infty} \leq \epsilon)}\max\limits_{1 \leq i \leq n} \min\limits_{1 \leq j \leq m} \vect{c_{i,j}}^{T}\vect{y_i} \leq \min\limits_{j_1 \in [m], \dots, j_{n} \in [m]}S(j_1, \dots, j_{n})$.
\begin{align}
&\vect{c_{i,j_{i}}}^{T}\vect{y_i} \geq \min\limits_{1 \leq j \leq m} \vect{c_{i,j}}^{T}\vect{y_i} \;\;\;\;\text{$\forall i \in [n]$ and $\forall j_{i} \in [m]$} \nonumber \\
&S(j_1, \dots, j_{n}) = \min\limits_{\pmb{\delta} \in \mathbb{R}^{n_0},  (\|\pmb{\delta}\|_{\infty} \leq \epsilon)}\max\limits_{1 \leq i \leq n} \vect{c_{i,j_{i}}}^{T}\vect{y_i} \geq \min\limits_{\pmb{\delta} \in \mathbb{R}^{n_0},  (\|\pmb{\delta}\|_{\infty} \leq \epsilon)}\max\limits_{1 \leq i \leq n} \min\limits_{1 \leq j \leq m} \vect{c_{i,j}}^{T}\vect{y_i} \;\;\;\;\text{$\forall j_{1} \in [m], \dots, j_{n} \in [m]$} \nonumber\\
&\min\limits_{j_1 \in [m], \dots, j_{n} \in [m]}S(j_1, \dots, j_{n}) \geq \min\limits_{\pmb{\delta} \in \mathbb{R}^{n_0},  (\|\pmb{\delta}\|_{\infty} \leq \epsilon)}\max\limits_{1 \leq i \leq n} \min\limits_{1 \leq j \leq m} \vect{c_{i,j}}^{T}\vect{y_i}
\label{eq:greaterthan}
\end{align}

Next, we show $\min\limits_{\pmb{\delta} \in \mathbb{R}^{n_0},  (\|\pmb{\delta}\|_{\infty} \leq \epsilon)}\max\limits_{1 \leq i \leq n} \min\limits_{1 \leq j \leq m} \vect{c_{i,j}}^{T}\vect{y_i} \geq \min\limits_{j_1 \in [m], \dots, j_{n} \in [m]}S(j_1, \dots, j_{n})$. 
There exists $\opt{\pmb{\delta}} \in \mathbb{R}^{n_0}$ such that $\|\opt{\pmb{\delta}}\|_{\infty} \leq \epsilon$, $\opt{\vect{y_i}} = N(\vect{x_i} + \opt{\pmb{\delta}})$ and $\max\limits_{1 \leq i \leq n} \min\limits_{1 \leq j \leq m} \vect{c_{i,j}}^{T}\opt{\vect{y_i}} = \min\limits_{\pmb{\delta} \in \mathbb{R}^{n_0},  (\|\pmb{\delta}\|_{\infty} \leq \epsilon)}\max\limits_{1 \leq i \leq n} \min\limits_{1 \leq j \leq m} \vect{c_{i,j}}^{T}\vect{y_i}$.
Let, $\opt{j_i} = \argmin\limits_{1 \leq j \leq m} \vect{c_{i,j}}^{T}\opt{\vect{y_i}}$ then

\begin{align}
&S(\opt{j_1}, \dots, \opt{j_{n}}) = \min\limits_{\pmb{\delta} \in \mathbb{R}^{n_0},  (\|\pmb{\delta}\|_{\infty} \leq \epsilon)}\max\limits_{1 \leq i \leq n} \vect{c_{i,\opt{j_{i}}}}^{T}\vect{y_i} \nonumber\\
& S(\opt{j_1}, \dots, \opt{j_{n}})\leq \max\limits_{1 \leq i \leq n} \vect{c_{i,\opt{j_{i}}}}^{T}\opt{\vect{y_i}} = \max\limits_{1 \leq i \leq n} \min\limits_{1 \leq j \leq m} \vect{c_{i,j}}^{T}\opt{\vect{y_i}} \;\;\;\;\text{since $\opt{j_i} = \argmin\limits_{1 \leq j \leq m} \vect{c_{i,j}}^{T}\opt{\vect{y_i}}$} \nonumber \\
&\min\limits_{j_1 \in [m], \dots, j_{n} \in [m]}S(j_1, \dots, j_{n}) \leq S(\opt{j_1}, \dots, \opt{j_{n}}) \leq \min\limits_{\pmb{\delta} \in \mathbb{R}^{n_0},  (\|\pmb{\delta}\|_{\infty} \leq \epsilon)}\max\limits_{1 \leq i \leq n} \min\limits_{1 \leq j \leq m} \vect{c_{i,j}}^{T}\vect{y_i} \label{eq:lessthan}
\end{align}
Combining Eq.~\ref{eq:greaterthan} and Eq.~\ref{eq:lessthan} we show $\min\limits_{\pmb{\delta} \in \mathbb{R}^{n_0},  (\|\pmb{\delta}\|_{\infty} \leq \epsilon)}\max\limits_{1 \leq i \leq n} \min\limits_{1 \leq j \leq m} \vect{c_{i,j}}^{T}\vect{y_i} = \min\limits_{j_1 \in [m], \dots, j_{n} \in [m]}S(j_1, \dots, j_{n})$.
\end{proof}

\begin{theorem}
\label{lem:multipleToSingle}
$\forall \pmb{\delta} \in \mathbb{R}^{n_0}. \left( (\|\pmb{\delta}\|_{\infty} \leq \epsilon)\implies \outspecset(\vect{y_1, \dots, \vect{y_n}})\right)$ if and only if   $\left(\min\limits_{j_1 \in [m], \dots, j_{n} \in [m]}S(j_1, \dots, j_{n})\right) \geq 0$ where for all $i \in [n]$, $\vect{y_i} = N(\vect{x_i} + \pmb{\delta})$, $\outspecset(\vect{y_1}, \dots, \vect{y_n}) = \bigvee_{i=1}^{n} \outpropparam{i}(\vect{y_i})$, $\outpropparam{i}(\vect{y_i}) = \bigwedge_{j=1}^{m} (\vect{c_{i, j}}^{T}\vect{y_{i}} \geq 0)$ and $S(j_1, \dots, j_{n}) = \min\limits_{\pmb{\delta} \in \mathbb{R}^{n_0},  (\|\pmb{\delta}\|_{\infty} \leq \epsilon)}\max\limits_{1 \leq i \leq n} \vect{c_{i,j_{i}}}^{T}\vect{y_i}$.
\end{theorem}
\begin{proof}
Follows from lemma~\ref{lem:multiple1} and lemma~\ref{lem:multiple2}.
\end{proof}

\subsubsection{Reduction to bound refinement with single linear inequality}
\label{sec:reductionToSingleLinearIneq}
Theorem~\ref{lem:multipleToSingle} allows us to learn parameters for each $S(j_1, \dots, j_{n})$ separately so that $S(j_1, \dots, j_{n}) \geq 0$ for each $(j_1, \dots, j_{n})$ where each $j_i \in [m]$. For $S(j_1, \dots, j_{n})$, let $\{(\linapprxalpha{{j_1}}, \biasalpha{j_1}), \dots,  (\linapprxalpha{{j_n}}, \biasalpha{j_n})\}$ denote the linear approximations satisfying $\linapprxalpha{{j_i}}^{T}(\vect{x_i} + \pmb{\delta}) + \biasalpha{j_i} \leq \vect{c_{i, j_i}}^{T}\vect{y_i}$ for any $\pmb{\delta} \in \mathbb{R}^{n_0}$  such that $\|\pmb{\delta}\|_{\infty} \leq \epsilon$ and $\vect{l_{j_i}} \preceq \param{{j_i}} \preceq \vect{u_{j_i}}$. Then we can use cross-execution bound refinement for $n$ executions to learn the parameters $(\param{j_1}, \dots, \param{j_{n}})$. We repeat this process for all $(j_1, \dots, j_{n})$. However, the number of possible choices for $(j_1, \dots, j_{n})$ is $m^{n}$ and learning parameters $(\param{j_1}, \dots, \param{j_{n}})$ for all possible $(j_1, \dots, j_{n})$ is only practically feasible when both $(m, n)$ are small constants. For larger values of $(m, n)$ we greedily pick $(j_1, \dots, j_{n})$ for learning parameters to avoid the exponential blowup as detailed below.

\textbf{Avoiding exponential blowup: } Instead of learning parameters for all possible $(j_1, \dots, j_{n})$ we greedily select only single tuple $(\opt{j_1}, \dots, \opt{j_{n}})$. For the $i$-th execution with $\outpropparam{i}(\vect{y_{i}}) = \wedge_{i=1}^{m} (\vect{c_{i, j}}^{T}\vect{y_i} \geq 0)$, let $\{(\linapprxalphaini{{i,1}}, \biasalphaini{{i,1}}), \dots, (\linapprxalphaini{{i,m}}, \biasalphaini{{i,m}})\}$ dentoes linear approximations satisfying $\linapprxalphaini{{i,j}}^{T}(\vect{x_i} + \pmb{\delta}) + \biasalphaini{{i, j}} \leq \vect{c_{i, j}}^{T}\vect{y_i}$ for all $j \in [m]$ and for all $\pmb{\delta} \in \mathbb{R}^{n_0}$ and $\|\pmb{\delta}\| \leq \epsilon$. Note that for all $j \in [m]$, $\vect{l_i} \preceq \paramini{i, j} \preceq \vect{u_i}$ are the initial values of the parameters $\param{i, j}$. 
Now, for we select $\opt{j_i}$ for each execution as $\opt{j_i} = \argmin\limits_{j \in [m]} \min\limits_{\pmb{\delta} \in \mathbb{R}^{n_0}, \|\pmb{\delta}\|_{\infty} \leq \epsilon} \linapprxalphaini{{i,j}}^{T}(\vect{x_i} + \pmb{\delta}) + \biasalphaini{{i, j}}$.

Intuitively, we use $\opt{j_i}$ to determine the linear inequality $\vect{c_{i, \opt{j_i}}}^Ty_i \geq 0$ that is likely to be violated. For the tuple $(\opt{j_1}, \dots, \opt{j_{n}})$, let $\optappx{\optparams} = (\opt{\param{\opt{j_1}}}, \dots, \opt{\param{\opt{j_n}}}, \opt{\lambda_{\opt{j_1}}}, \dots, \opt{\lambda_{\opt{j_n}}})$ denote the learned parameters (which may not correspond to global optimum). Then we use the same parameters across all $m$ linear approximations for the $i$-th execution i.e. $\{(\linapprxoptalphamulti{{i, 1}}{\opt{j_i}}, \biasoptalphamulti{{i, 1}}{\opt{j_i}}), \dots, (\linapprxoptalphamulti{{i, m}}{\opt{j_i}}, \biasoptalphamulti{{i, m}}{\opt{j_i}})\}$. 
In this case, $\optappx{t}(\dual)$ is defined as $\optappx{t}(G) = \min\limits_{\pmb{\delta} \in \mathbb{R}^{n_0},  (\|\pmb{\delta}\|_{\infty} \leq \epsilon)}\max\limits_{1 \leq i \leq n} \min\limits_{1 \leq j \leq m} \linapprxoptalphamulti{{i, j}}{\opt{j_i}}^{T}(\vect{x_i} + \pmb{\delta}) + \biasoptalphamulti{{i, j}}{\opt{j_i}}$. Next, we prove the correctness of the bound refinement. 

\begin{theorem}[Correctness of bound refinement for a conjunction of linear inequalities]
\label{thm:correctnessOvermultipleineq}
If $\optappx{t}(\dual) \geq 0$ then $\forall \pmb{\delta} \in \mathbb{R}^{n_0}. \left( (\|\pmb{\delta}\|_{\infty} \leq \epsilon)\implies \outspecset(\vect{y_1, \dots, \vect{y_n}})\right)$ where $\optappx{t}(G) = \min\limits_{\pmb{\delta} \in \mathbb{R}^{n_0},  (\|\pmb{\delta}\|_{\infty} \leq \epsilon)}\max\limits_{1 \leq i \leq n} \min\limits_{1 \leq j \leq m} \linapprxoptalphamulti{{i, j}}{\opt{j_i}}^{T}(\vect{x_i} + \pmb{\delta}) + \biasoptalphamulti{{i, j}}{\opt{j_i}}$ and for all $i \in [n]$, $\vect{y_i} = N(\vect{x_i} + \pmb{\delta})$.
\end{theorem}
\begin{proof}
First we show that $\optappx{t}(\dual) \leq  \min\limits_{\pmb{\delta} \in \mathbb{R}^{n_0},  (\|\pmb{\delta}\|_{\infty} \leq \epsilon)}\max\limits_{1\leq i \leq n}\min\limits_{1 \leq j \leq m}\vect{c_{i, j}}^{T}\vect{y_i}$
\begin{align}
&\linapprxoptalphamulti{{i, j}}{\opt{j_i}}^{T}(\vect{x_i} + \pmb{\delta}) + \biasoptalphamulti{{i, j}}{\opt{j_i}} \leq \vect{c_{i, j}}^{T}\vect{y_i} \;\;\;\;\text{$\forall i \in [n]$, $\forall j \in [m]$ and for all $\pmb{\delta} \in \mathbb{R}^{n_0}$ s.t $\|\pmb{\delta}\|_{\infty} \leq \epsilon$}\nonumber\\
&\min\limits_{1 \leq j \leq m}\linapprxoptalphamulti{{i, j}}{\opt{j_i}}^{T}(\vect{x_i} + \pmb{\delta}) + \biasoptalphamulti{{i, j}}{\opt{j_i}} \leq \min\limits_{1 \leq j \leq m}\vect{c_{i, j}}^{T}\vect{y_i} \;\;\;\;\text{$\forall i \in [n]$ and for all $\pmb{\delta} \in \mathbb{R}^{n_0}$ s.t $\|\pmb{\delta}\|_{\infty} \leq \epsilon$}\nonumber\\
&\max\limits_{1\leq i \leq n}\min\limits_{1 \leq j \leq m}\linapprxoptalphamulti{{i, j}}{\opt{j_i}}^{T}(\vect{x_i} + \pmb{\delta}) + \biasoptalphamulti{{i, j}}{\opt{j_i}} \leq \max\limits_{1\leq i \leq n}\min\limits_{1 \leq j \leq m}\vect{c_{i, j}}^{T}\vect{y_i} \;\;\;\;\text{for all $\pmb{\delta} \in \mathbb{R}^{n_0}$ s.t $\|\pmb{\delta}\|_{\infty} \leq \epsilon$}\nonumber\\
&\min\limits_{\pmb{\delta} \in \mathbb{R}^{n_0},  (\|\pmb{\delta}\|_{\infty} \leq \epsilon)}\max\limits_{1\leq i \leq n}\min\limits_{1 \leq j \leq m}\linapprxoptalphamulti{{i, j}}{\opt{j_i}}^{T}(\vect{x_i} + \pmb{\delta}) + \biasoptalphamulti{{i, j}}{\opt{j_i}} \leq \min\limits_{\pmb{\delta} \in \mathbb{R}^{n_0},  (\|\pmb{\delta}\|_{\infty} \leq \epsilon)}\max\limits_{1\leq i \leq n}\min\limits_{1 \leq j \leq m}\vect{c_{i, j}}^{T}\vect{y_i} \nonumber \\
&\optappx{t}(\dual) \leq \min\limits_{\pmb{\delta} \in \mathbb{R}^{n_0},  (\|\pmb{\delta}\|_{\infty} \leq \epsilon)}\max\limits_{1\leq i \leq n}\min\limits_{1 \leq j \leq m}\vect{c_{i, j}}^{T}\vect{y_i} \label{eq:multiThmFinalEq}
\end{align}
Using lemma~\ref{lem:multiple1} and Eq~\ref{eq:multiThmFinalEq}
\begin{align}
(\optappx{t}(\dual) \geq 0) &\implies \left(\min\limits_{\pmb{\delta} \in \mathbb{R}^{n_0},  (\|\pmb{\delta}\|_{\infty} \leq \epsilon)}\max\limits_{1\leq i \leq n}\min\limits_{1 \leq j \leq m}\vect{c_{i, j}}^{T}\vect{y_i}\right) \geq 0 \nonumber\\
&\implies \forall \pmb{\delta} \in \mathbb{R}^{n_0}. \left( (\|\pmb{\delta}\|_{\infty} \leq \epsilon)\implies \outspecset(\vect{y_1, \dots, \vect{y_n}})\right)\nonumber
\end{align}
\end{proof}

\subsection{Handling general $\|\cdot\|_{p}$ norm}
\label{sec:lPNorm}
For general $\|\cdot\|_{p}$ norm we can generalize the dual formulation $\dual(\optparams)$ in the following way. Since, $\|\pmb{\delta}\|_{p} \leq \epsilon$ and $\constalpha{i} = \linapprxalpha{i}^T\vect{x_i} + \biasalpha{i}$ then
\begin{align}
\dual(\optparams) &= \min\limits_{\|\pmb{\delta}\|_{p} \leq \epsilon} \sum_{i=1}^{n} \lambda_i \times \left(\linapprxalpha{i}^T(\vect{x_i} + \pmb{\delta}) + \biasalpha{i}\right) \nonumber\\
\dual(\optparams) &= \sum_{i=1}^{n} \lambda_i \times \constalpha{i} + \min\limits_{\pmb{\delta} \in \mathbb{R}^{n_0},\; \|\pmb{\delta}\|_{p} \leq \epsilon} \sum_{i=1}^{n} \lambda_i \times \linapprxalpha{i}^T\pmb{\delta} \nonumber \\
\dual(\optparams) &= \sum_{i=1}^{n} \lambda_i \times \constalpha{i} - \epsilon \times \left\|\sum_{i=1}^{n}\lambda_i \times \linapprxalpha{i}\right\|_{q} \;\;\; \text{Using H\"{o}lder's Inequality with $\frac{1}{q} = 1 - \frac{1}{p}$}\nonumber 
\end{align}

\subsection{MILP formulations and correctness}
In this section, we show the MILP formulations for the k-UAP and worst-case hamming distance verification and present the theoretical results corresponding to the correctness and efficacy of the MILP formulations.

Let $I = \{i \;|\; \text{non-relational verifier does not verify ($\inprop^{i}, \outpropparam{i})$}\}$ denotes the executions that remain unverified by the non-relational verifier. For all $i \in I$, $j \in [m]$ let $(\linapprx_{i, j}^{k'}, b_{i, j}^{k'})$ denote the linear approximations satisfying $\linapprx_{i, j}^{k'}(\vect{x_i} + \pmb{\delta}) + b_{i, j}^{k'} \leq \vect{c_{i, j}}^{T}\vect{y_i}$ for all $\pmb{\delta} \in \mathbb{R}^{n_0}$ and $\|\pmb{\delta}\|_{\infty} \leq \epsilon$ where $k' \leq \sum_{i=1}^{k_1} \binom{k_0}{i} + 1$ and $\vect{y_i} = N(\vect{x_i} + \pmb{\delta})$. Note that each linear approximations $(\linapprx_{i, j}^{k'}, b_{i, j}^{k'})$ are obtained by the non-relational verifier or by the cross-execution bound refinement. 

\subsubsection{MILP formulations}
\label{appendix:milPFormulationKUAPHamm}
\textbf{MILP formulation for k-UAP: }
\begin{align}
    &\min\;\; M \nonumber\\
    \nonumber\\
    & \|\pmb{\delta}\|_{\infty} \leq \epsilon \nonumber\\
    & \linapprx_{i, j}^{k'}(\vect{x_i} + \pmb{\delta}) + b_{i, j}^{k'} \leq o_{i, j} \;\;\;\;\text{$\forall i \in I$, $\forall j \in [m]$ $\forall k'$} \nonumber\\
    & z_{i} = \left(\left(\min\limits_{j \in [m]} o_{i,j}\right) \geq 0\right) \;\;\;\;\text{for all $i \in I$ $z_{i} \in \{0, 1\}$} \nonumber\\
    & \overline{k} = k - |I| \;\;\;\;\text{[number of executions verified by non-relational verifier]} \nonumber\\
    &M = \sum_{i \in I} z_{i} + \overline{k} \label{eq:kUAPMILP}
\end{align}

\textbf{MILP formulation for worst-case hamming distance: }
\begin{align*}
    &\max\;\; M \\
    \\
    & \|\pmb{\delta}\|_{\infty} \leq \epsilon \\
    & \linapprx_{i, j}^{k'}(\vect{x_i} + \pmb{\delta}) + b_{i, j}^{k'} \leq o_{i, j} \;\;\;\;\text{$\forall i \in I$, $\forall j \in [m]$ $\forall k'$} \\
    & z_{i} = \left(\left(\min\limits_{j \in [m]} o_{i,j}\right) \geq 0\right) \;\;\;\;\text{for all $i \in I$ $z_{i} \in \{0, 1\}$} \\
    &M = |I| - \sum_{i \in I} z_{i}
\end{align*}
\textbf{Correctness for eliminating individually verified executions: } We formally prove that eliminating individually verified executions is correct and does not lead to precision loss. 
\begin{theorem}
\label{thm:eliminationCorrect}
$\Output(\inspecset, \outspecset) = (k - |I|) + \min\limits_{\pmb{\delta} \in \mathbb{R}^{n_0}, \|\pmb{\delta}\|_{\infty} \leq \epsilon}\sum\limits_{i \in I} z_{i}(\pmb{\delta})$ where $z_{i}(\pmb{\delta})$ is defined in Eq.~\ref{eq:zdef}, $\Output(\inspecset, \outspecset)$ is defined in Eq.~\ref{eq:kUAPprobFormulation} and for all $j \in [k] \setminus I$,  $\forall \pmb{\delta} \in \mathbb{R}^{n_0}. (\|\pmb{\delta}\|_{\infty} \leq \epsilon) \implies (z_j(\delta) = 1)$ holds.
\end{theorem}
\begin{proof}
\begin{align*}
\Output(\inspecset, \outspecset) &= \min\limits_{\pmb{\delta} \in \mathbb{R}^{n_0}, \|\pmb{\delta}\|_{\infty} \leq \epsilon}\sum\limits_{i=1}^{k} z_{i}(\pmb{\delta}) \\
& = \min\limits_{\pmb{\delta} \in \mathbb{R}^{n_0}, \|\pmb{\delta}\|_{\infty} \leq \epsilon}\sum\limits_{i \in ([k] \setminus I)} z_{i}(\pmb{\delta}) + \min\limits_{\pmb{\delta} \in \mathbb{R}^{n_0}, \|\pmb{\delta}\|_{\infty} \leq \epsilon}\sum\limits_{i \in I} z_{i}(\pmb{\delta}) \\
& = (k - |I|) + \min\limits_{\pmb{\delta} \in \mathbb{R}^{n_0}, \|\pmb{\delta}\|_{\infty} \leq \epsilon}\sum\limits_{i \in I} z_{i}(\pmb{\delta})\;\;\;\;\text{since $\forall \pmb{\delta} \in \mathbb{R}^{n_0}. (\|\pmb{\delta}\|_{\infty} \leq \epsilon) \implies (z_j(\delta) = 1)$}
\end{align*}
\end{proof}
\textbf{Soundness of MILP formulation: } For soundness, we show that the optimal value of the MILP formulation (in Eq.~\ref{eq:kUAPMILP}) $\apprxOutput(\inspecset, \outspecset)$ is always a valid lower bound of $\Output(\inspecset, \outspecset)$. The soundness of the worst-case hamming distance formulation can be proved similarly.
\begin{theorem}[Sondness of the MILP formulation in Eq.~\ref{eq:kUAPMILP}]
\label{thm:racoonSoundness}
$\apprxOutput(\inspecset, \outspecset) \leq \Output(\inspecset, \outspecset)$ where $\apprxOutput(\inspecset, \outspecset)$ is the optimal solution of the MILP in Eq.~\ref{eq:kUAPMILP} and $\Output(\inspecset, \outspecset)$ is defined in Eq.~\ref{eq:kUAPprobFormulation}.  
\end{theorem}
\begin{proof}
We prove this by contradiction. Suppose, $\apprxOutput(\inspecset, \outspecset) > \Output(\inspecset, \outspecset)$ then there exists $\opt{\pmb{\delta}} \in \mathbb{R}^{n_0}$ such that $\|\opt{\pmb{\delta}}\|_{\infty} \leq \epsilon$ and $\apprxOutput(\inspecset, \outspecset) > \mu(\opt{\pmb{\delta}})$ where $\mu(\pmb{\delta})$ defined in Eq.~\ref{eq:muDef}.

For all $i \in I$, $j \in [m]$ the linear approximation $(\linapprx_{i, j}^{k'}, b_{i, j}^{k'})$ satisfies $\linapprx_{i, j}^{k'}(\vect{x_i} + \pmb{\delta}) + b_{i, j}^{k'} \leq \vect{c_{i, j}}^{T}\vect{y_i}$ for all $\pmb{\delta} \in \mathbb{R}^{n_0}$ and $\|\pmb{\delta}\|_{\infty} \leq \epsilon$ where $k' \leq \sum_{i=1}^{k_1} \binom{k_0}{i} + 1$ and $\vect{y_i} = N(\vect{x_i} + \pmb{\delta})$. Let, $\ol{z_i}(\opt{\pmb{\delta}}) = \left(\min\limits_{j \in [m]} \ol{o_{i,j}}(\opt{\pmb{\delta}}) \geq 0\right)$ where $\ol{o_{i,j}}(\opt{\pmb{\delta}}) = \max\limits_{k'} \linapprx_{i, j}^{k'}(\vect{x_i} + \opt{\pmb{\delta}}) + b_{i, j}^{k'}$. Then $\apprxOutput(\inspecset, \outspecset) \leq \overline{k} + \sum\limits_{i \in I} \ol{z_i}(\opt{\pmb{\delta}})$ and $\mu(\opt{\pmb{\delta}}) < \overline{k} + \sum\limits_{i \in I} \ol{z_i}(\opt{\pmb{\delta}})$.
\begin{align}
    &\mu(\opt{\pmb{\delta}}) < \overline{k} + \sum\limits_{i \in I} \ol{z_i}(\opt{\pmb{\delta}}) \nonumber\\
    &\implies \sum\limits_{i \in I} z_i(\opt{\pmb{\delta}}) < \sum\limits_{i \in I} \ol{z_i}(\opt{\pmb{\delta}}) \;\;\;\;\text{where $z_i(\opt{\pmb{\delta}})$ defined in Eq.~\ref{eq:zdef}} \label{eq:intermediate}
\end{align}
Eq.~\ref{eq:intermediate} implies that there exist $i_0 \in I$ such that $z_{i_0}(\opt{\pmb{\delta}}) = 0$ and $\ol{z_{i_0}}(\opt{\pmb{\delta}}) = 1$. Since $z_{i_0}(\opt{\pmb{\delta}}) = 0$ then there exists $j_0 \in [m]$ such that $\vect{c_{i_0, j_0}}^{T}\opt{\vect{y_{i_0}}} < 0$ where $\opt{\vect{y_{i_0}}} = N(\vect{x_{i_0}} + \opt{\pmb{\delta}})$
\begin{align*}
 &\min\limits_{j \in [m]} \ol{o_{i_0,j}}(\opt{\pmb{\delta}}) \leq  \ol{o_{i_0,j_0}}(\opt{\pmb{\delta}}) \leq  \vect{c_{i_0, j_0}}^{T}\opt{\vect{y_{i_0}}}< 0 \\
 &(\min\limits_{j \in [m]} \ol{o_{i_0,j}}(\opt{\pmb{\delta}}) < 0) \implies (\ol{z_{i_0}}(\opt{\pmb{\delta}}) = 0) \;\;\;\;\text{Contradiction since $\ol{z_{i_0}}(\opt{\pmb{\delta}}) = 1$}
\end{align*}
\end{proof}

Next, we show that \Tool is always at least as precise as the current SOTA relational verifier \cite{iclrUap}. Note that \cite{iclrUap} uses the same MILP formulation (Eq.~\ref{eq:kUAPMILP}) except instead of using $k'$ linear approximations $\{(\linapprx_{i, j}^{1}, b_{i, j}^{1}), \dots, (\linapprx_{i, j}^{k'}, b_{i, j}^{k'})\}$ it uses a single statically obtained linear approximation say $\{(\linapprx_{i, j}^{1}, b_{i, j}^{1})\}$.

\begin{theorem}[\Tool is at least as precise as \cite{iclrUap}]
\label{thm:racoonBetter}
$\apprxOutput_{b}(\inspecset, \outspecset) \leq \apprxOutput(\inspecset, \outspecset)$ where $\apprxOutput(\inspecset, \outspecset)$ is the optimal solution of the MILP in Eq.~\ref{eq:kUAPMILP} and $\apprxOutput_{b}(\inspecset, \outspecset)$ is the optimal solution from the baseline \cite{iclrUap}.  
\end{theorem}
\begin{proof}
Now we show that for $i \in I$, $\forall j \in [m]$ for every feasible value of the variable $o_{i,j}$ in Eq.~\ref{eq:kUAPMILP} is also a feasible value of the same variable $o_{i,j}$ in MILP of \cite{iclrUap}.  Given $\forall k'$, $\linapprx_{i, j}^{k'}(\vect{x_i} + \pmb{\delta}) + b_{i, j}^{k'} \leq o_{i, j}$ then trivally $o_{i, j}$  satisfies condition $\linapprx_{i, j}^{1}(\vect{x_i} + \pmb{\delta}) + b_{i, j}^{1} \leq o_{i, j}$ used by the baseline \cite{iclrUap}. Subsequently for all $i \in I$ every feasible value of $z_{i}$ in Eq.~\ref{eq:kUAPMILP} is also a feasible value of the same variable $z_{i}$ in the MILP of \cite{iclrUap}. Let. for all $i \in I$, $\mathcal{Z}$ and $\mathcal{Z}_{b}$ denote the sets of all feasible values of variables $(z_{1}, \dots, z_{I})$ from the MILP in Eq.~\ref{eq:kUAPMILP} and the baseline \cite{iclrUap} respectively. Then $\mathcal{Z} \subseteq \mathcal{Z}_{b}$ which implies
\begin{align*}
   \apprxOutput_{b}(\inspecset, \outspecset) &\leq k - |I| + \min\limits_{(z_{1}, \dots, z_{I}) \in \mathcal{Z}_{b}} \sum_{i \in I} z_{i} \\
   &\leq k - |I| + \min\limits_{(z_{1}, \dots, z_{I}) \in \mathcal{Z}} \sum_{i \in I} z_{i} = \apprxOutput(\inspecset, \outspecset)\;\;\;\text{Since $\mathcal{Z} \subseteq \mathcal{Z}_{b}$}
\end{align*}
\end{proof}

\section{Worst-case time complexity analysis of \Tool}
\label{appendix:WorstcaseComplexity}
Let, the total number of neurons in $N$ be $n_{t}$ and the number of layers in $N$ is $l$. Then for each execution, the worst-case cost of running the non-relational verifier \cite{auto_lirpa} is $O(l^{2} \times n_{t}^{3})$. We assume that we run $I_{t}$ number of iterations with the optimizer and the cost of each optimization step over a set of $n$ executions is $O(n \times C_{o})$. In general, $C_{o}$ is similar to the cost of the non-relational verifier i.e. $O(l^{2} \times n_{t}^{3})$. Then the total cost of cross-execution refinement is $O(T \times I_{t} \times C_{o})$ where $T = \sum_{i=1}^{k_1} \left(\binom{k_0}{i} \times i\right)$. Assuming MILP with $O(k \times n_{l})$ integer variables in the worst-case takes $C_{M}(k \times n_{l})$ time. Then the worst-case complexity of \Tool is $O(k \times l^{2} \times n_{t}^{3}) + O(T \times I_{t} \times C_{o}) + C_{M}(k \times n_{l})$.

\section{Details of DNN archietectures}
\label{appenexp:archDetails}

\begin{table*}[htb]
\centering
\captionsetup{justification=centering}
\caption{DNN architecture details}
\label{table:netArchDetails}
\resizebox{0.48\textwidth}{!}{
\begin{tabular}{@{}c c c c c c@{}}
\toprule
Dataset & Model & Type & Train & \# Layers & \# Params \\
\midrule
\text{ }   & IBPSmall & Conv & IBP & 4 & 60k \\
\text{ }   & ConvSmall & Conv & Standard & 4 & 80k \\
\text{ }   & ConvSmall & Conv & PGD & 4 & 80k \\
MNIST   & ConvSmall & Conv & DiffAI & 4 & 80k \\
\text{ }   & ConvSmall & Conv & COLT & 4 & 80k \\
\text{ } & IBPMedium & Conv & IBP & 5 &  400k \\
\text{ } & ConvBig & Conv & DiffAI & 7 &  1.8M \\
\midrule
\text{ } & IBP-Small &  Conv & IBP & 4 & 60k  \\
\text{ }  & ConvSmall & Conv & Standard & 4 & 80k \\
\text{ }  & ConvSmall & Conv & PGD & 4 & 80k \\
CIFAR10 & ConvSmall &  Conv & DiffAI & 4 & 80k \\
\text{ }  & ConvSmall & Conv & COLT & 4 & 80k \\
\text{ }& IBPMedium & Conv & IBP & 5 &  2.2M \\
\text{ }& ConvBig & Conv & DiffAI & 7 &  2.5M \\
\midrule
\end{tabular}
}
\end{table*}
\subsection{Implementation Details}
\label{appenexp:implemenDetails}
We implemented our method in Python with Pytorch V1.11 and used Gurobi V10.0.3 as an off-the-shelf MILP solver. The implementation of cross-execution bound refinement is built on top of the SOTA DNN verification tool \lirpa \cite{alphaCrown} and uses Adam \cite{kingma2014adam} for parameter learning. We run $20$ iterations of Adam on each set of executions. For each relational property, we use $k_0 = 6$ and $k_1 = 4$ for deciding which set of executions to consider for cross-execution refinement as discussed in section~\ref{sec:algorithm}. We use a single NVIDIA A100-PCI GPU with 40 GB RAM for bound refinement and an Intel(R) Xeon(R) Silver 4214R CPU @ 2.40GHz with 64 GB RAM for MILP optimization.

\clearpage
\newpage
\clearpage
\newpage
\section{UAP accuracy over data distribution}
\label{appenexp:uapDistribution}

\begin{table*}[htb]
\centering
\captionsetup{justification=centering}
\caption{Statistical estimation worst case UAP accuracy over input distribution using $k$-UAP accuracy different with different $k$ values.}
\label{table:UAPOverDataDist}
\resizebox{0.98\textwidth}{!}{
\begin{tabular}{@{}c c c c c c c cc c c c c c@{}}
\toprule
Dataset & Property & 
Network & Training &  Perturbation & \multicolumn{3}{c}{Non-relational Verifier} & \multicolumn{3}{c}{I/O Formulation} & \multicolumn{3}{c}{\Tool}  \\
\text{ } & \text{ } & 
Structure & \text{Method} \text{ } & \text{Bound ($\epsilon$) } &  \multicolumn{3}{c}{UAP Acc. (\%)} & \multicolumn{3}{c}{UAP Acc. (\%)} & \multicolumn{3}{c}{UAP Acc. (\%)} \\
\text{ } & \text{ } & \text{ } & \text{} \text{ } & \text{ } & $k=20$ & $k=30$ & $k=50$ & $k=20$ & $k=30$ & $k=50$ & $k=20$ & $k=30$ & $k=50$ \\ 
\midrule
\text{ } & UAP &  ConvSmall & Standard & 0.08 & 11.00 & 15.33 & 19.20 & 20.50 & 28.0 &33.40 & 26.50\;(\textcolor{mgreen}{+6.00}) & 31.00\;(\textcolor{mgreen}{+3.00})& 34.40\;(\textcolor{mgreen}{+1.00})\\ 
\text{ } & UAP &  ConvSmall & PGD & 0.10 & 43.00 & 46.00 & 47.80 & 44.50 & 49.00 & 53.00 & 49.50\;(\textcolor{mgreen}{+5.00}) & 54.00\;(\textcolor{mgreen}{+5.00}) &57.00\;(\textcolor{mgreen}{+4.00})\\ 
\text{MNIST} &UAP &  IBPSmall & IBP & 0.13 & 47.00 & 51.00 & 55.20 & 47.50 & 51.67 & 57.80 & 61.50\;(\textcolor{mgreen}{+14.00}) & 67.67\;(\textcolor{mgreen}{+16.00}) & 69.80\;(\textcolor{mgreen}{+12.00})\\ 
\text{ } & UAP &  ConvSmall & DiffAI & 0.13 & 28.50 & 31.00 & 35.20 & 33.50 & 38.67 & 46.20 & 40.50\;(\textcolor{mgreen}{+7.00}) & 45.00\;(\textcolor{mgreen}{+6.33}) & 48.60\;(\textcolor{mgreen}{+2.40})\\ 
\text{ } & UAP &  ConvSmall & COLT & 0.15 & 41.50 & 46.33 & 48.80 & 41.50 & 46.67 & 49.80 & 58.00\;(\textcolor{mgreen}{+16.50}) & 63.00\;(\textcolor{mgreen}{+16.33}) & 65.60\;(\textcolor{mgreen}{+15.80})\\ 
\midrule
Dataset & Property & 
Network & Training &  Perturbation & \multicolumn{3}{c}{Non-relational Verifier} & \multicolumn{3}{c}{I/O Formulation} & \multicolumn{3}{c}{\Tool}  \\
\text{ } & \text{ } & 
Structure & \text{Method} \text{ } & \text{Bound ($\epsilon$) } &  \multicolumn{3}{c}{UAP Acc. (\%)} & \multicolumn{3}{c}{UAP Acc. (\%)} & \multicolumn{3}{c}{UAP Acc. (\%)} \\
\text{ } & \text{ } & \text{ } & \text{} \text{ } & \text{ } & $k=15$ & $k=20$ & $k=25$ & $k=15$ & $k=20$ & $k=25$ & $k=15$ & $k=20$ & $k=25$ \\ 
\midrule
\text{ } & UAP &  ConvSmall & Standard & 1.0/255 & 16.93 & 19.00 & 21.90 & 22.27& 27.00 & 30.70 & 24.27\;(\textcolor{mgreen}{+2.00}) & 28.00\;(\textcolor{mgreen}{+1.00}) & 32.30\;(\textcolor{mgreen}{+1.60}) \\  
\text{ } & UAP &  ConvSmall & PGD & 2.0/255 & 19.60 & 27.50 & 30.30 & 23.60 & 33.00 & 35.50 & 24.27\;(\textcolor{mgreen}{+0.67}) & 33.50\;(\textcolor{mgreen}{+0.50}) & 37.50\;(\textcolor{mgreen}{+2.00}) \\  
CIFAR10& UAP &  IBPSmall & IBP & 3.0/255 & 23.60 &31.50  & 34.30  & 23.60 & 31.50 & 35.10 &34.27\;(\textcolor{mgreen}{+10.67}) &42.00\;(\textcolor{mgreen}{+10.50}) & 46.30\;(\textcolor{mgreen}{+11.20})\\ 
\text{ }  & UAP &  ConvSmall & DiffAI &3.0/255 & 36.27 &39.00 & 43.90 & 38.93 & 45.50 & 50.70 & 40.93\;(\textcolor{mgreen}{+2.00}) & 46.50\;(\textcolor{mgreen}{+1.00})& 51.50\;(\textcolor{mgreen}{+0.80})\\ 
\text{ } & UAP &  ConvSmall & COLT & 6.0/255 & 13.60 & 19.50 & 21.50  & 18.93 & 26.50 & 27.50  & 22.93\;(\textcolor{mgreen}{+4.00}) & 29.00\;(\textcolor{mgreen}{+2.50}) & 29.90\;(\textcolor{mgreen}{+2.40})\\ 
\midrule
\end{tabular}
}
\end{table*}
All values in Table~\ref{table:UAPOverDataDist} are computed using Theorem 2 of \cite{iclrUap} with $\xi = 0.1$.

\section{\Tool componentwise efficacy analysis on all DNNs}
\label{appenexp:additionalComponentAllNet}

\begin{table*}[htb]
\centering
\captionsetup{justification=centering}
\caption{\Tool Componentwise Efficacy Analysis}
\label{table:compareComponents}
\resizebox{0.98\textwidth}{!}{
\begin{tabular}{@{}c c c c c c c c c c@{}}
\toprule
Dataset & Network & Training &  Perturbation & Non-relational  & I/O & Individual & Individual & Cross-Execution & \Tool  \\
\text{ } & Structure & Method &  Bound ($\epsilon$) & Verifier & Formulation & Refinement & Refinement with MILP & Refinement & verifier \\
\text{ } & \text{ } & \text{ } &  \text{ } & Avg. UAP Acc. (\%) & Avg. UAP Acc. (\%) & Avg. UAP Acc. (\%) & Avg. UAP Acc. (\%) & Avg. UAP Acc. (\%) & Avg. UAP Acc. (\%) \\
\midrule
\text{ } & ConvSmall & Standard & 0.08 & 38.5 & 48.0 & 42.5 & 50.5 & 51.0 & 54.0 \\
\text{ } & ConvSmall & PGD & 0.10 & 70.5 & 72.0 & 72.5 & 74.0 & 76.5 & 77.0 \\
\text{ } & IBPSmall & IBP & 0.13 & 74.5 & 75.0 & 84.0 & 84.5 & 89.0 & 89.0 \\
MNIST & ConvSmall & DiffAI & 0.13 & 56.0 & 61.0 & 61.0 & 64.5 & 67.0 & 68.0 \\
\text{ } & ConvSmall & COLT & 0.15 & 69.0 & 69.0 & 72.5 & 72.5 & 81.5 & 85.5 \\
\text{ } & IBPMedium & IBP & 0.20 & 80.5 & 82.0 & 91.0 & 91.0 & 93.5 & 93.5 \\
\text{ } & ConvBig & DiffAI & 0.20 & 80.0 & 80.0 & 86.0 & 86.0 & 90.0 & 93.0 \\
\midrule
\text{ } & ConvSmall & Standard & 1.0/255 & 52.0 & 55.0 & 52.0 & 55.0 & 57.0 & 58.0 \\
\text{ } & ConvSmall & PGD & 3.0/255 & 21.0 & 26.0 & 22.0 & 27.0 & 29.0 & 29.0 \\
\text{ } & IBPSmall & IBP & 6.0/255 & 17.0 & 17.0 & 29.0 & 29.0 & 36.0 & 39.0 \\
CIFAR10 & ConvSmall & DiffAI & 8.0/255 & 16.0 & 20.0 & 26.0 & 28.0 & 29.0 & 30.0 \\
\text{ } & ConvSmall & COLT & 8.0/255 & 18.0 & 21.0 & 22.0 & 22.0 & 24.0 & 26.0 \\
\text{ } & IBPMedium & IBP & 3.0/255 & 46.0 & 50.0 & 63.0 & 63.0 & 69.0 & 71.0 \\
\text{ } & ConvBig & DiffAI & 3.0/255 & 17.0 & 20.0 & 24.0 & 25.0 & 25.0 & 25.0 \\
\midrule
\end{tabular}
}
\end{table*}

\clearpage
\newpage
\section{\Tool componentwise runtime analysis on all DNNs}
\label{appenexp:additionalComponentRuntime}

\begin{table*}[htb]
\centering
\captionsetup{justification=centering}
\caption{\Tool Componentwise Runtime Analysis}
\label{table:compareComponentsRuntime}
\resizebox{0.98\textwidth}{!}{
\begin{tabular}{@{}c c c c c c c c c c@{}}
\toprule
Dataset & Network & Training &  Perturbation & Non-relational  & I/O & Individual & Individual & Cross-Execution & \Tool  \\
\text{ } & Structure & Method &  Bound ($\epsilon$) & Verifier & Formulation & Refinement & Refinement with MILP & Refinement & verifier \\
\text{ } & \text{ } & \text{ } &  \text{ } & Avg. Time (sec.) & Avg. Time (sec.) & Avg. Time (sec.) & Avg. Time (sec.) & Avg. Time (sec.) & Avg. Time (sec.) \\
\midrule
\text{ } & ConvSmall & Standard & 0.08 & 0.02 & 2.66 & 0.58 & 3.21 & 3.42 & 5.21 \\
\text{ } & ConvSmall & PGD & 0.10 & 0.02 & 0.93 & 0.61 & 1.82 & 3.47 & 4.33 \\
\text{ } & IBPSmall & IBP & 0.13 & 0.02 & 1.02 & 0.48 & 1.78 & 1.58 & 2.02 \\
MNIST & ConvSmall & DiffAI & 0.13 & 0.01 & 1.10 & 0.52 & 2.11 & 2.84 & 3.99 \\
\text{ } & ConvSmall & COLT & 0.15 & 0.02 & 0.99 & 0.50 & 1.82 & 2.17 & 2.68 \\
\text{ } & IBPMedium & IBP & 0.20 &  0.07 & 0.99 & 0.90 & 2.02 & 1.91 & 2.27 \\
\text{ } & ConvBig & DiffAI & 0.20 & 1.85 & 2.23 & 3.70 & 4.07 & 7.36 & 7.60 \\
\midrule
\text{ } & ConvSmall & Standard & 1.0/255 & 0.02 & 3.46 & 0.50 & 5.48 & 2.99 & 7.22 \\
\text{ } & ConvSmall & PGD & 3.0/255 & 0.01 & 1.57 & 0.40 & 3.64 & 2.44 & 5.56 \\
\text{ } & IBPSmall & IBP & 6.0/255 & 0.02 & 2.76 & 0.56 & 3.92 & 3.32 & 6.76 \\
CIFAR10 & ConvSmall & DiffAI & 8.0/255 & 0.01 & 2.49 & 0.49 & 4.75 & 2.96 & 7.09 \\
\text{ } & ConvSmall & COLT & 8.0/255 & 0.04 & 2.41 & 0.92 & 3.95 & 6.73 & 11.02 \\
\text{ } & IBPMedium & IBP & 3.0/255 & 0.15 & 2.13 & 1.77 & 4.07 & 5.28 & 6.12 \\
\text{ } & ConvBig & DiffAI & 3.0/255 & 1.33 & 3.42 & 3.27 & 5.89 & 10.45 & 11.92 \\
\midrule
\end{tabular}
}
\end{table*}

\section{Additional plots for k-UAP for different \texorpdfstring{$\pmb{\epsilon}$}{} values}
\label{appenexp:additionalEpsilon}
\begin{figure*}[htb]
\centering
\begin{minipage}[b]{.3\textwidth}
\includegraphics[width=\textwidth]{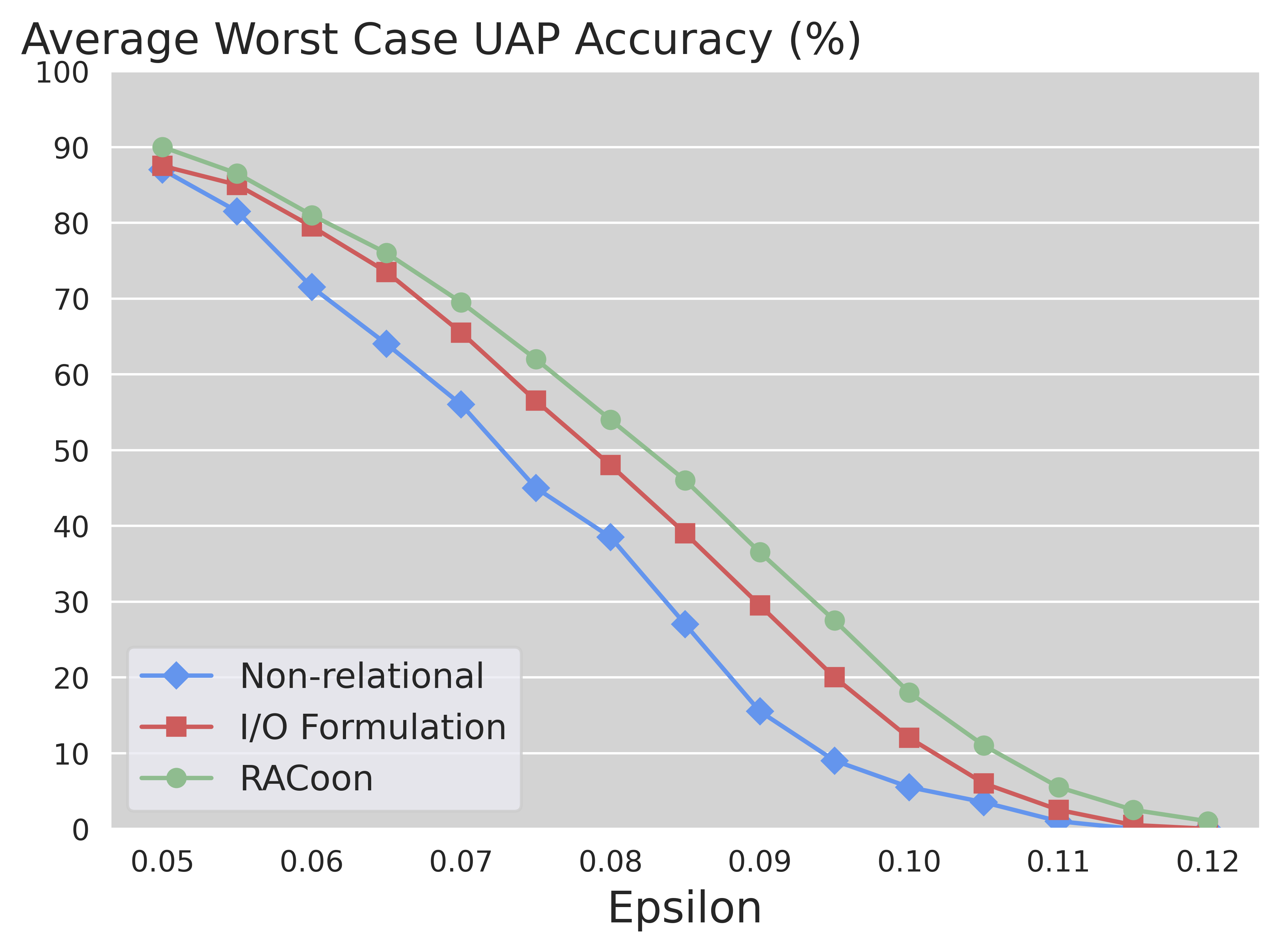}
\captionsetup{labelformat=empty}
\caption{(a) ConvSmall (Standard)}
\end{minipage}\qquad
\addtocounter{figure}{-1}
\begin{minipage}[b]{.3\textwidth}
\includegraphics[width=\textwidth]{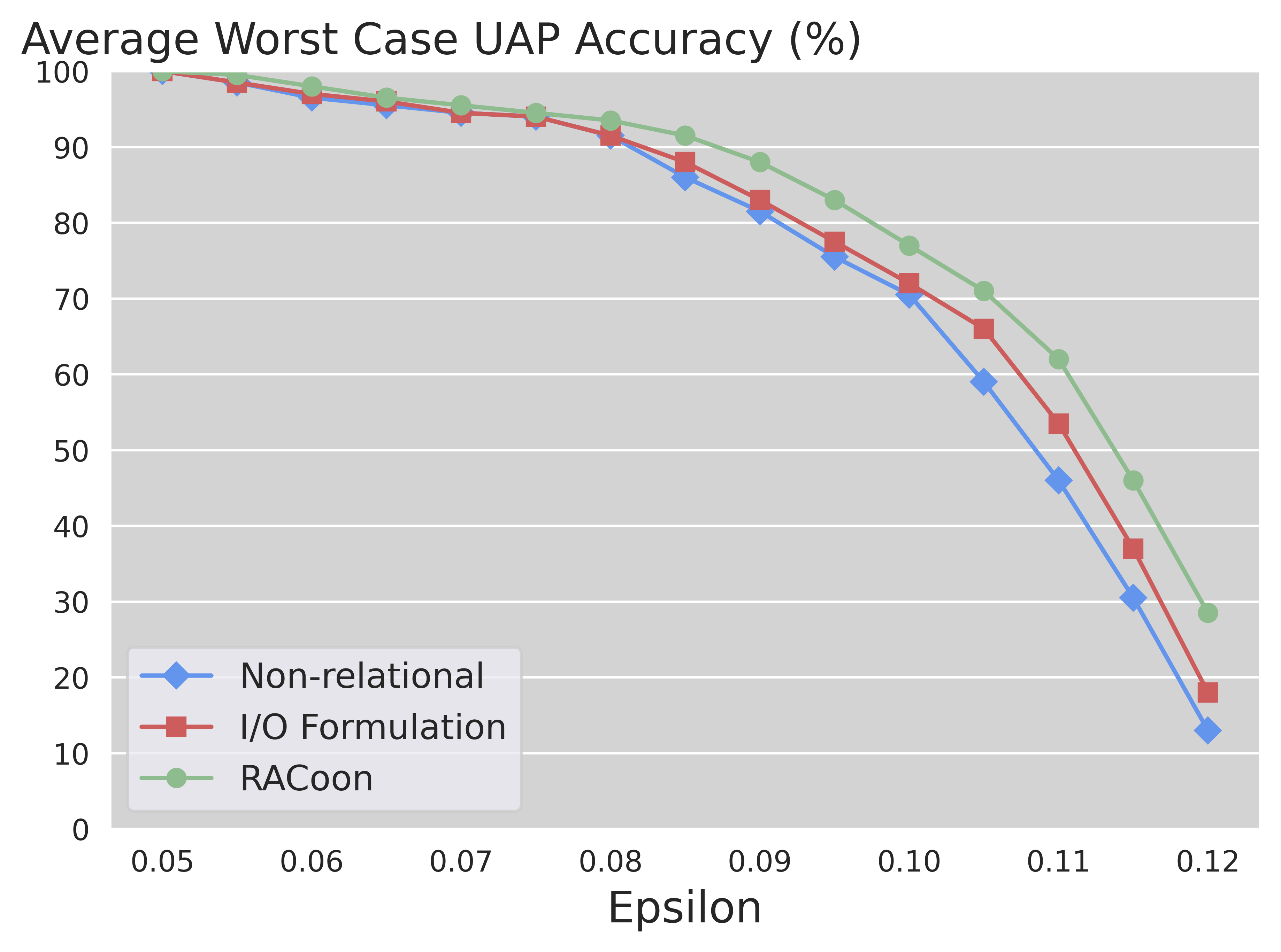}
\captionsetup{labelformat=empty}
\caption{(b) ConvSmall (PGD)}
\end{minipage}\qquad
\addtocounter{figure}{-1}
\begin{minipage}[b]{.3\textwidth}
\includegraphics[width=\textwidth]{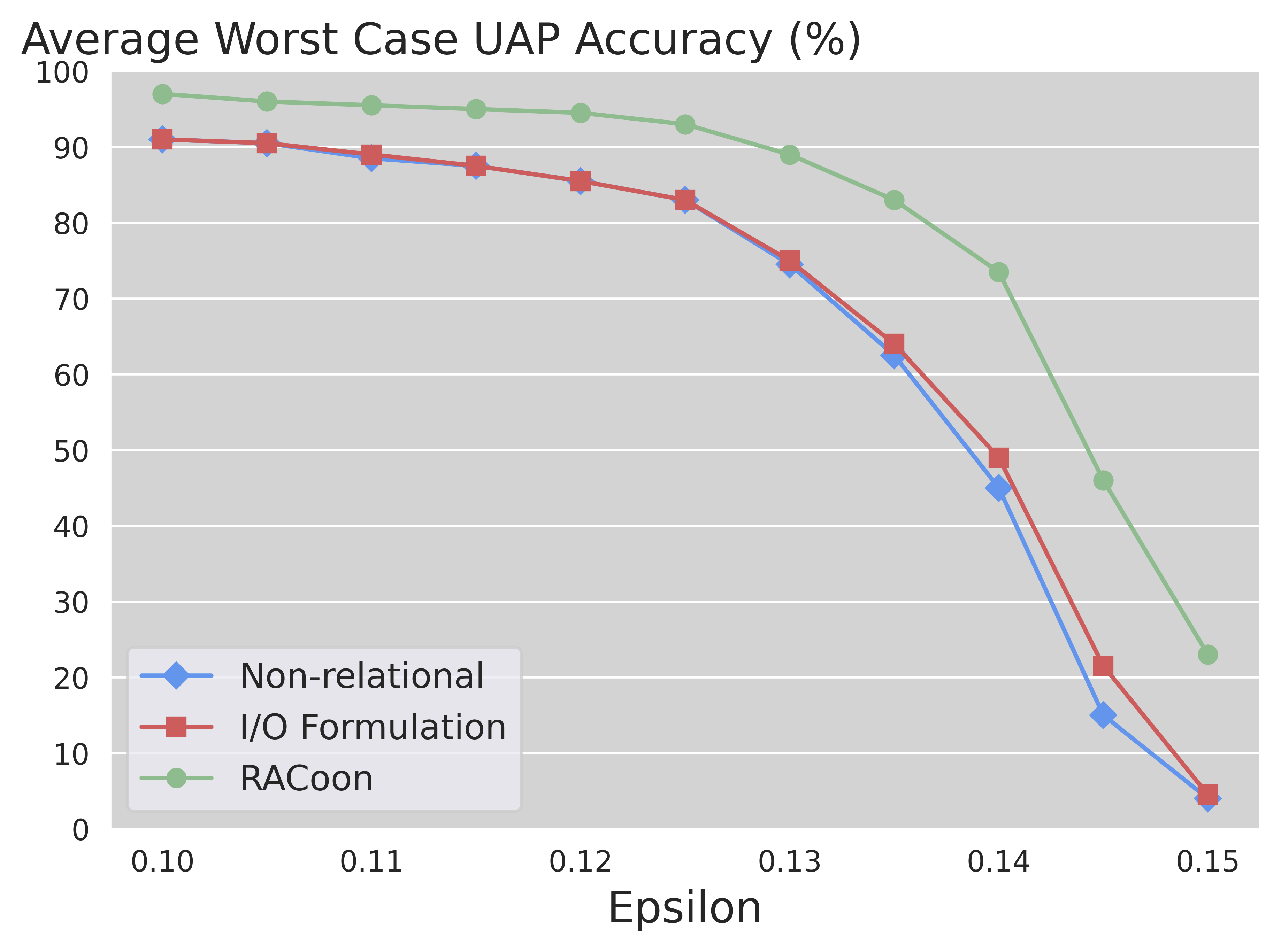}
\captionsetup{labelformat=empty}
\caption{(c) IBPSmall}
\end{minipage}\qquad
\addtocounter{figure}{-1}
\begin{minipage}[b]{.3\textwidth}
\includegraphics[width=\textwidth]{plots/mnist_0.1.onnx_20.png}
\captionsetup{labelformat=empty}
\caption{(d) ConvSmall (COLT)}
\end{minipage}\qquad
\addtocounter{figure}{-1}
\begin{minipage}[b]{.3\textwidth}
\includegraphics[width=\textwidth]{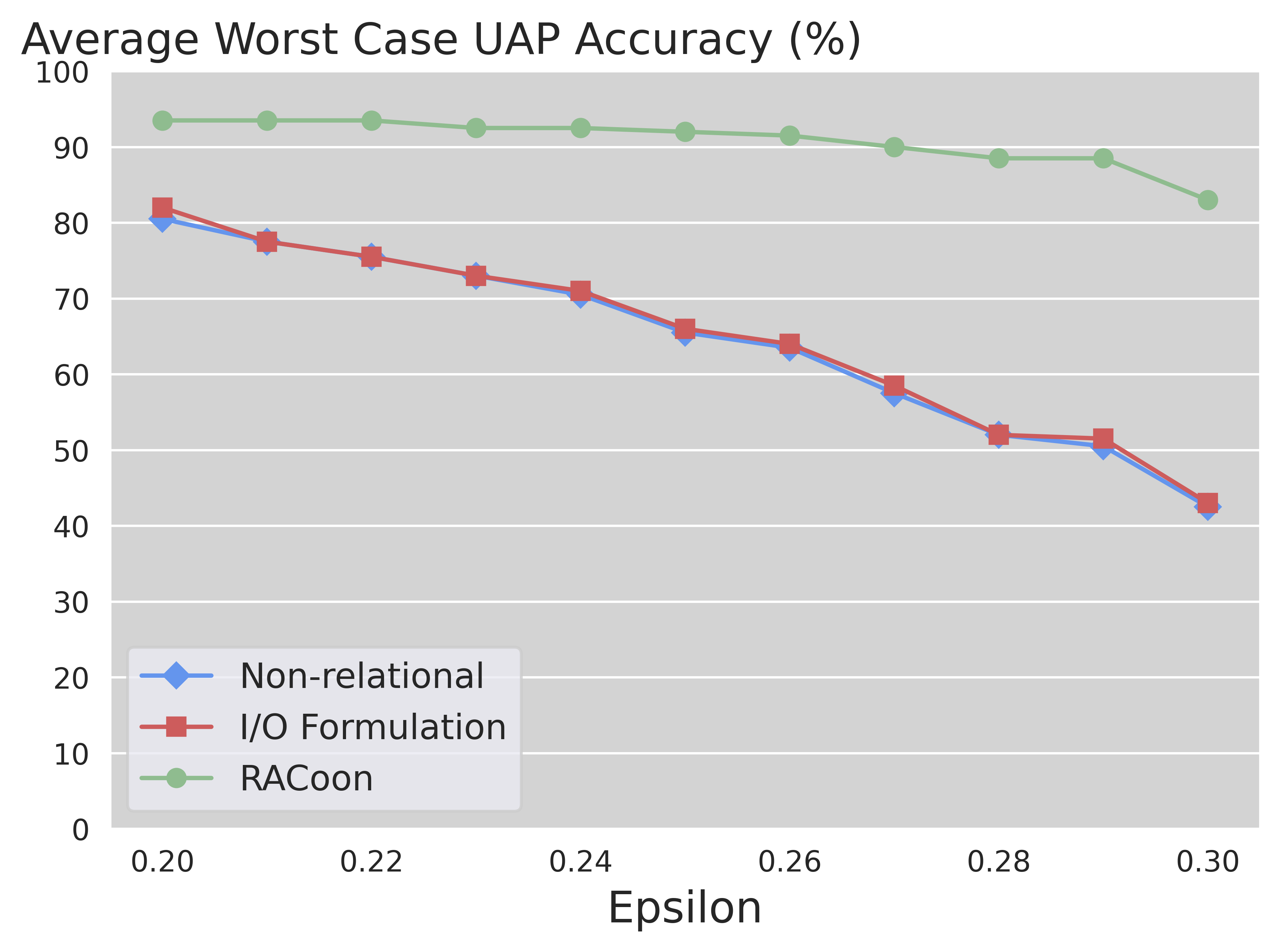}
\captionsetup{labelformat=empty}
\caption{(e) IBPMedium}
\end{minipage}\qquad
\addtocounter{figure}{-1}
\begin{minipage}[b]{.3\textwidth}
\includegraphics[width=\textwidth]{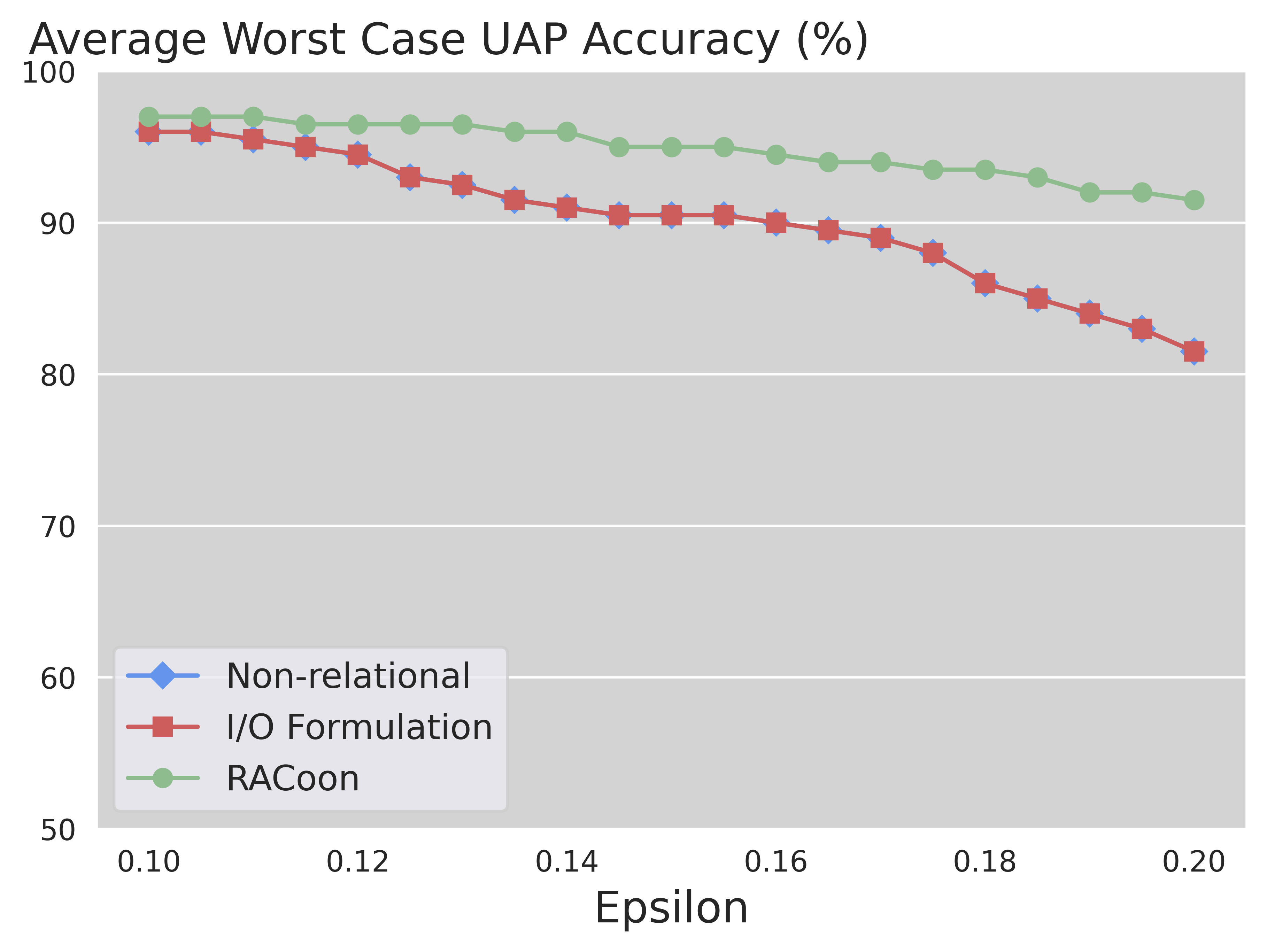}
\captionsetup{labelformat=empty}
\caption{(f) ConvBig (DiffAI)}
\end{minipage}\qquad
\addtocounter{figure}{-1}
\caption{Average worst-case UAP accuracy for different $\epsilon$ values for networks trained on MNIST.}
\label{fig:mnist_eps_appen}
\end{figure*}

\begin{figure}[htb]
\centering
\begin{minipage}[b]{.3\textwidth}
\includegraphics[width=\textwidth]{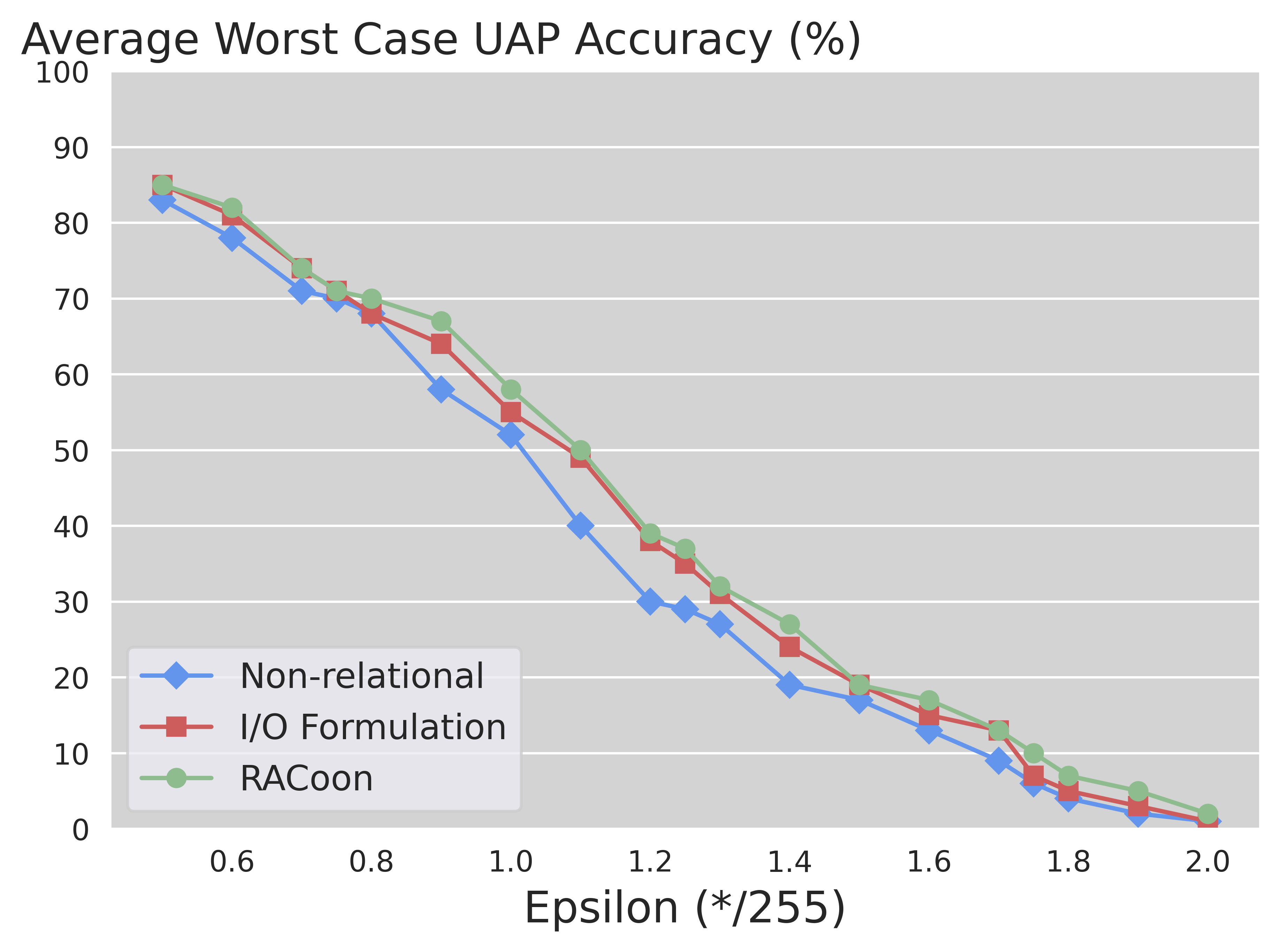}
\captionsetup{labelformat=empty}
\caption{(a) ConvSmall (Standard)}
\end{minipage}\qquad
\addtocounter{figure}{-1}
\begin{minipage}[b]{.3\textwidth}
\includegraphics[width=\textwidth]{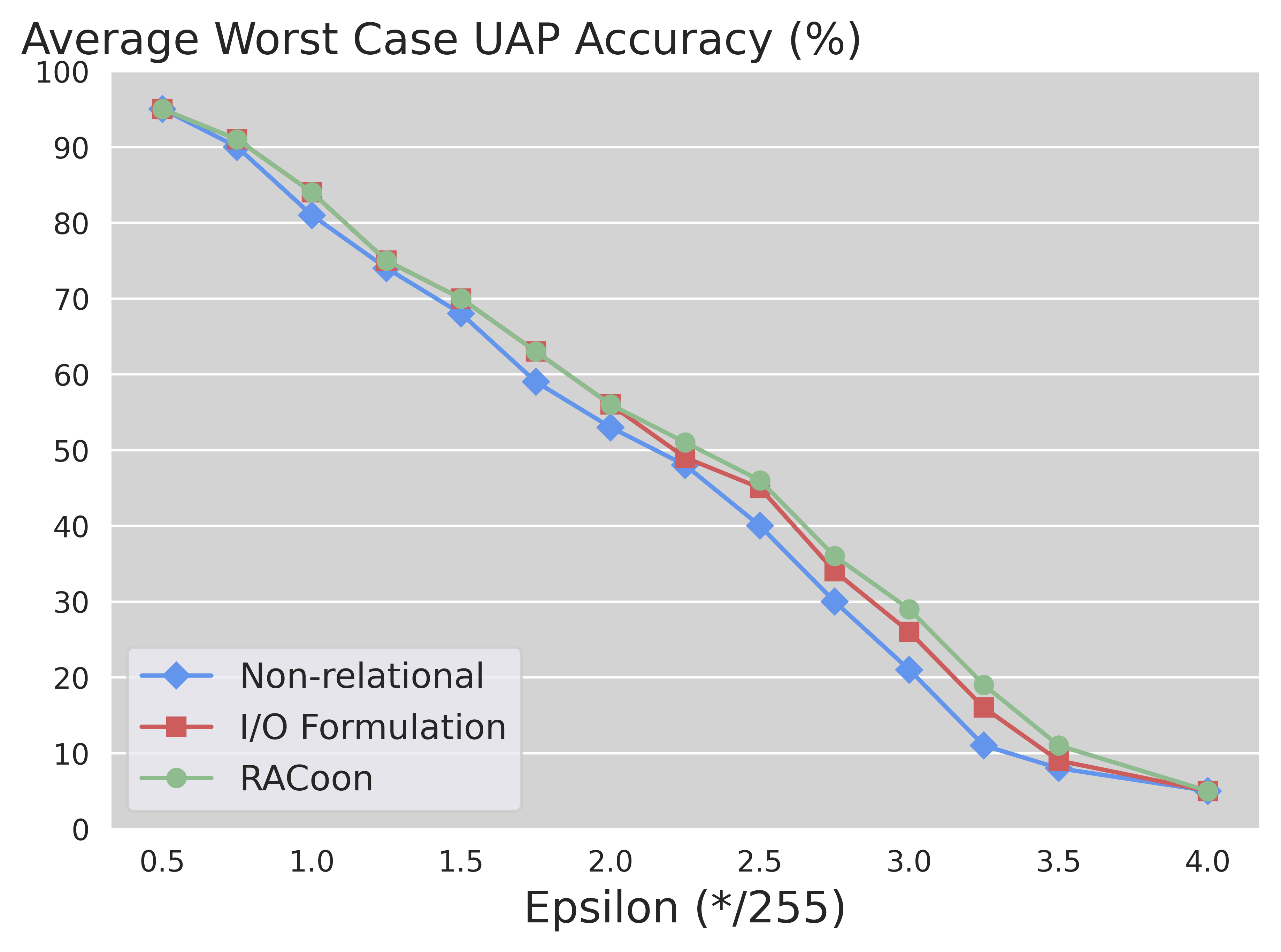}
\captionsetup{labelformat=empty}
\caption{(b) ConvSmall (PGD)}
\end{minipage}\qquad
\addtocounter{figure}{-1}
\begin{minipage}[b]{.3\textwidth}
\includegraphics[width=\textwidth]{plots/cifar_cnn_2layer_width_2_best.pth_10.png}
\captionsetup{labelformat=empty}
\caption{(c) IBPSmall}
\end{minipage}\qquad
\addtocounter{figure}{-1}
\begin{minipage}[b]{.3\textwidth}
\includegraphics[width=\textwidth]{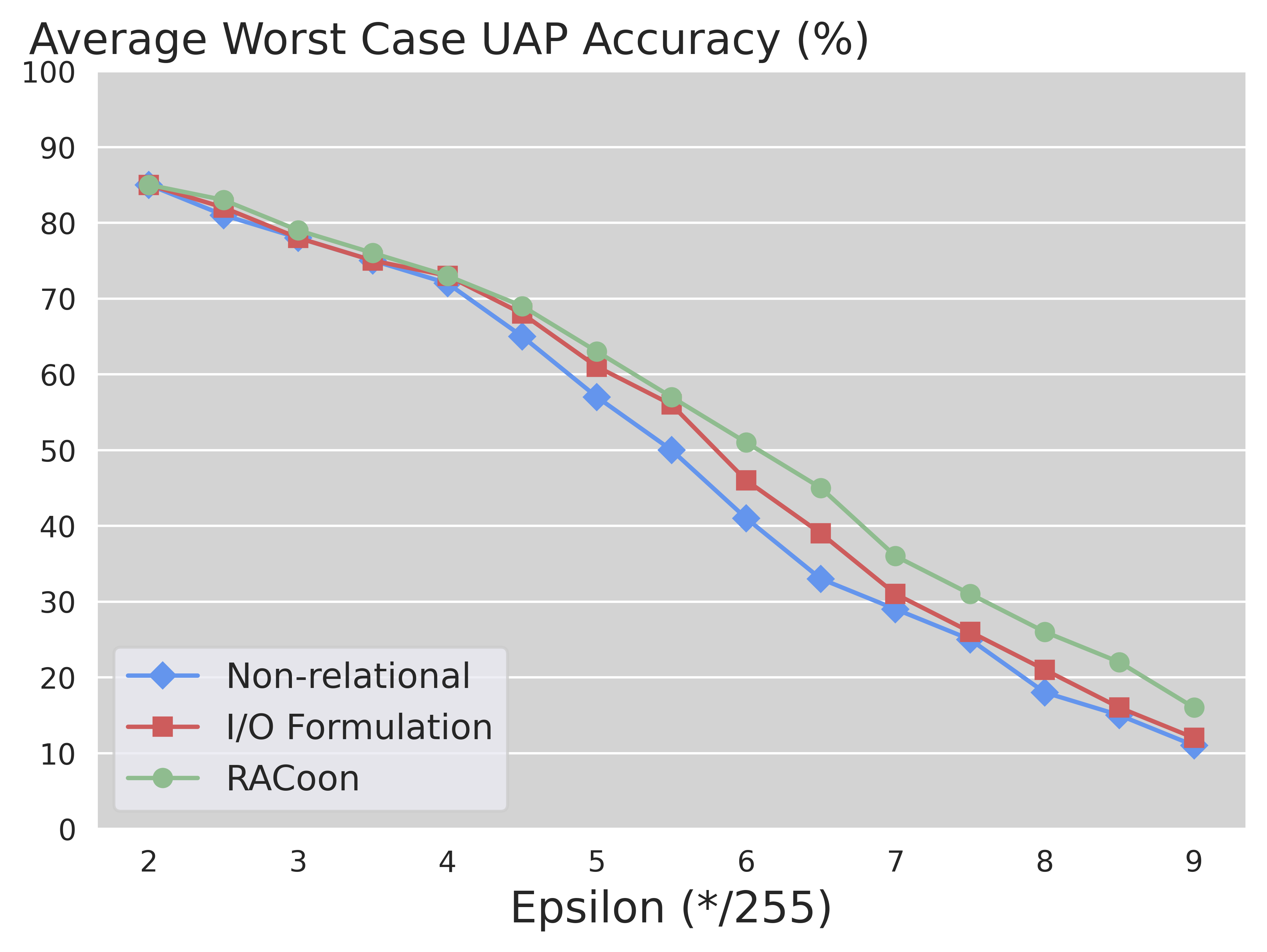}
\captionsetup{labelformat=empty}
\caption{(d) ConvSmall (COLT)}
\end{minipage}\qquad
\addtocounter{figure}{-1}
\begin{minipage}[b]{.3\textwidth}
\includegraphics[width=\textwidth]{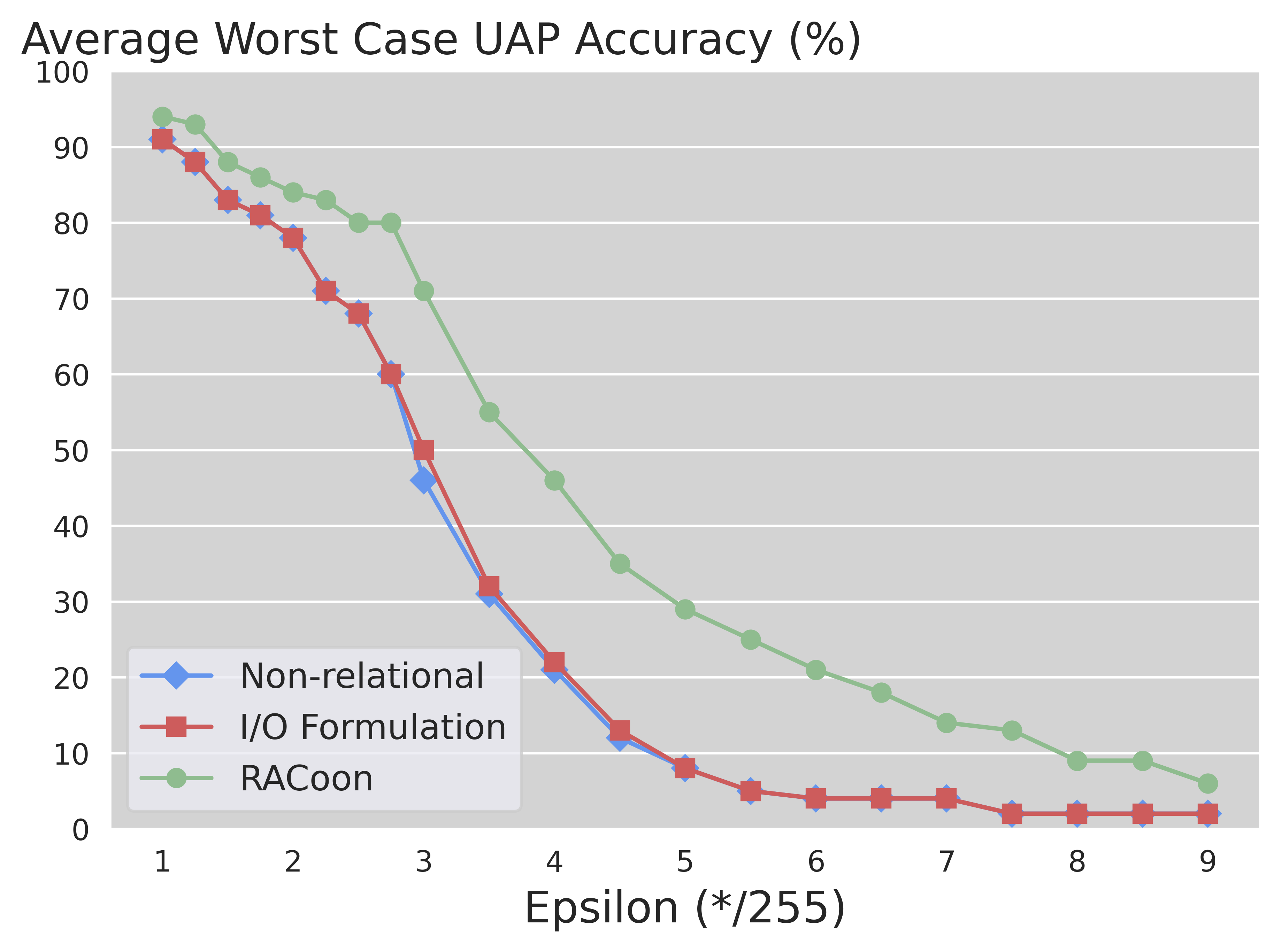}
\captionsetup{labelformat=empty}
\caption{(e) IBPMedium}
\end{minipage}\qquad
\addtocounter{figure}{-1}
\begin{minipage}[b]{.3\textwidth}
\includegraphics[width=\textwidth]{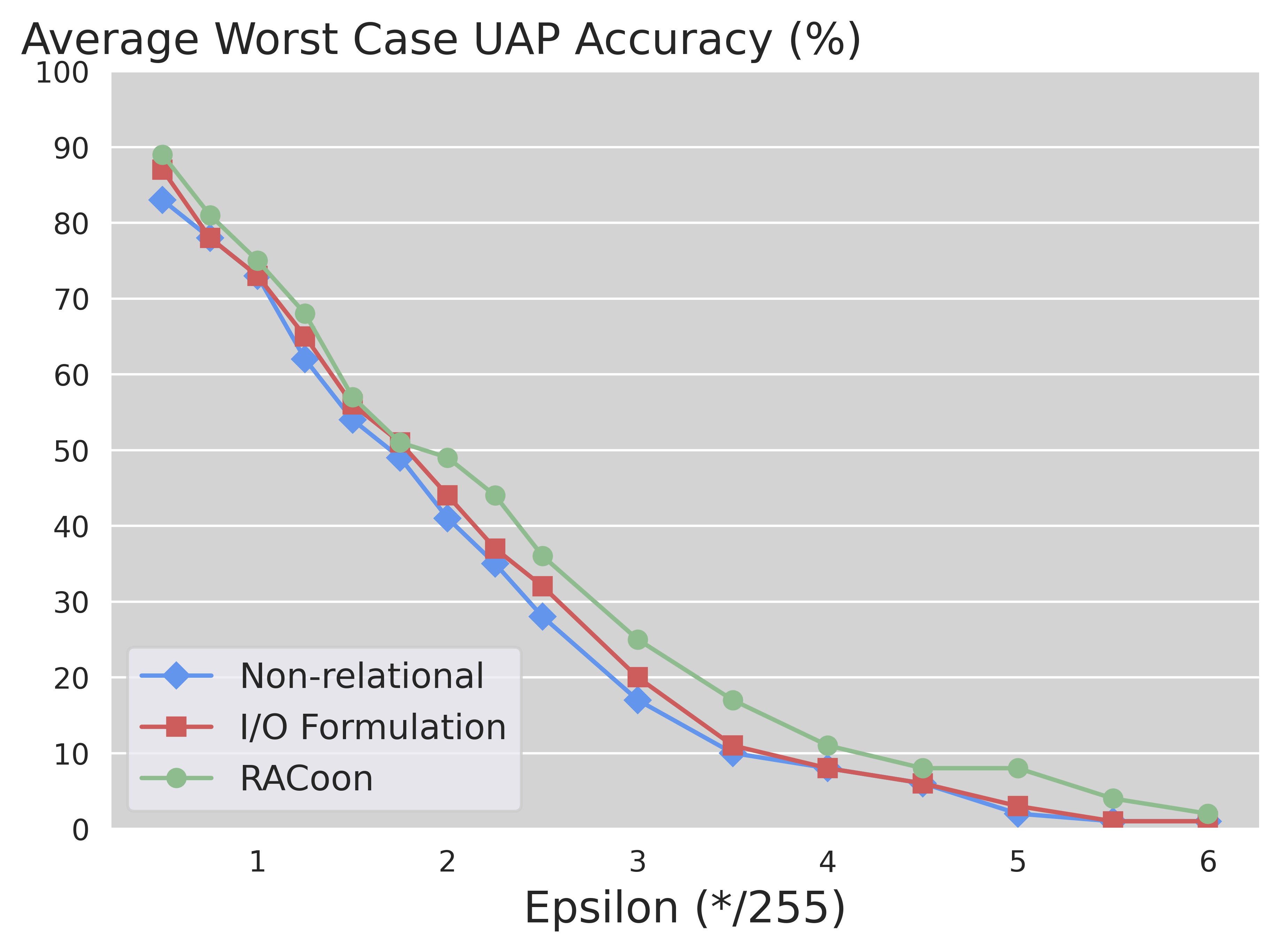}
\captionsetup{labelformat=empty}
\caption{(f) ConvBig (DiffAI)}
\end{minipage}\qquad
\addtocounter{figure}{-1}
\caption{Average worst-case UAP accuracy for different $\epsilon$ values for networks trained on CIFAR10.}
\label{fig:cifar_eps_appen}
\end{figure}

\clearpage
\newpage
\section{Plots for worst-case hamming distance for different \texorpdfstring{$\pmb{\epsilon}$}{} values}
\label{appenexp:additionalEpsilonHamming}
\begin{figure}[htb]
\centering
\begin{minipage}[b]{.3\textwidth}
\includegraphics[width=\textwidth]{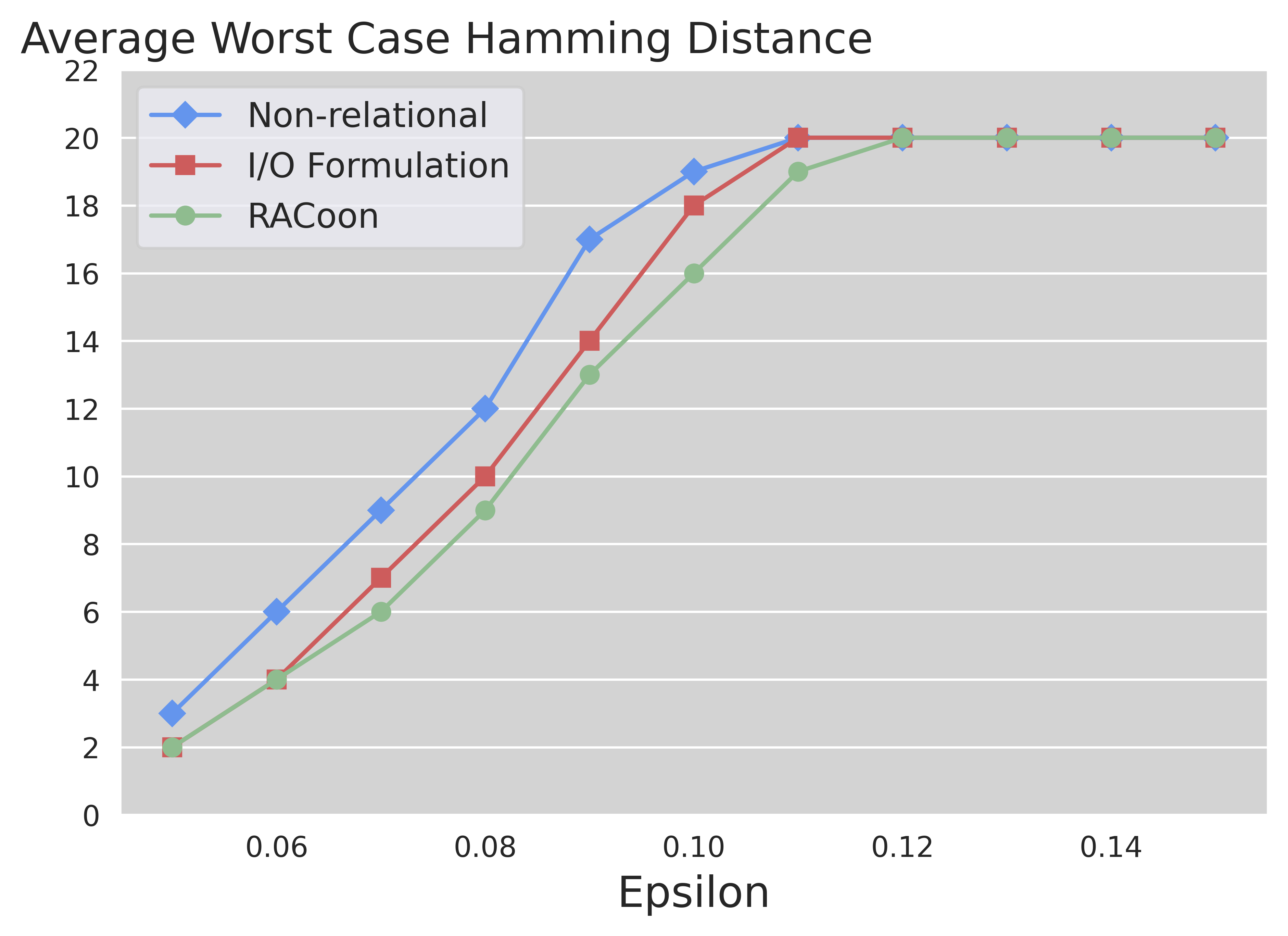}
\captionsetup{labelformat=empty}
\caption{(a) ConvSmall (Standard)}
\end{minipage}\qquad
\addtocounter{figure}{-1}
\begin{minipage}[b]{.3\textwidth}
\includegraphics[width=\textwidth]{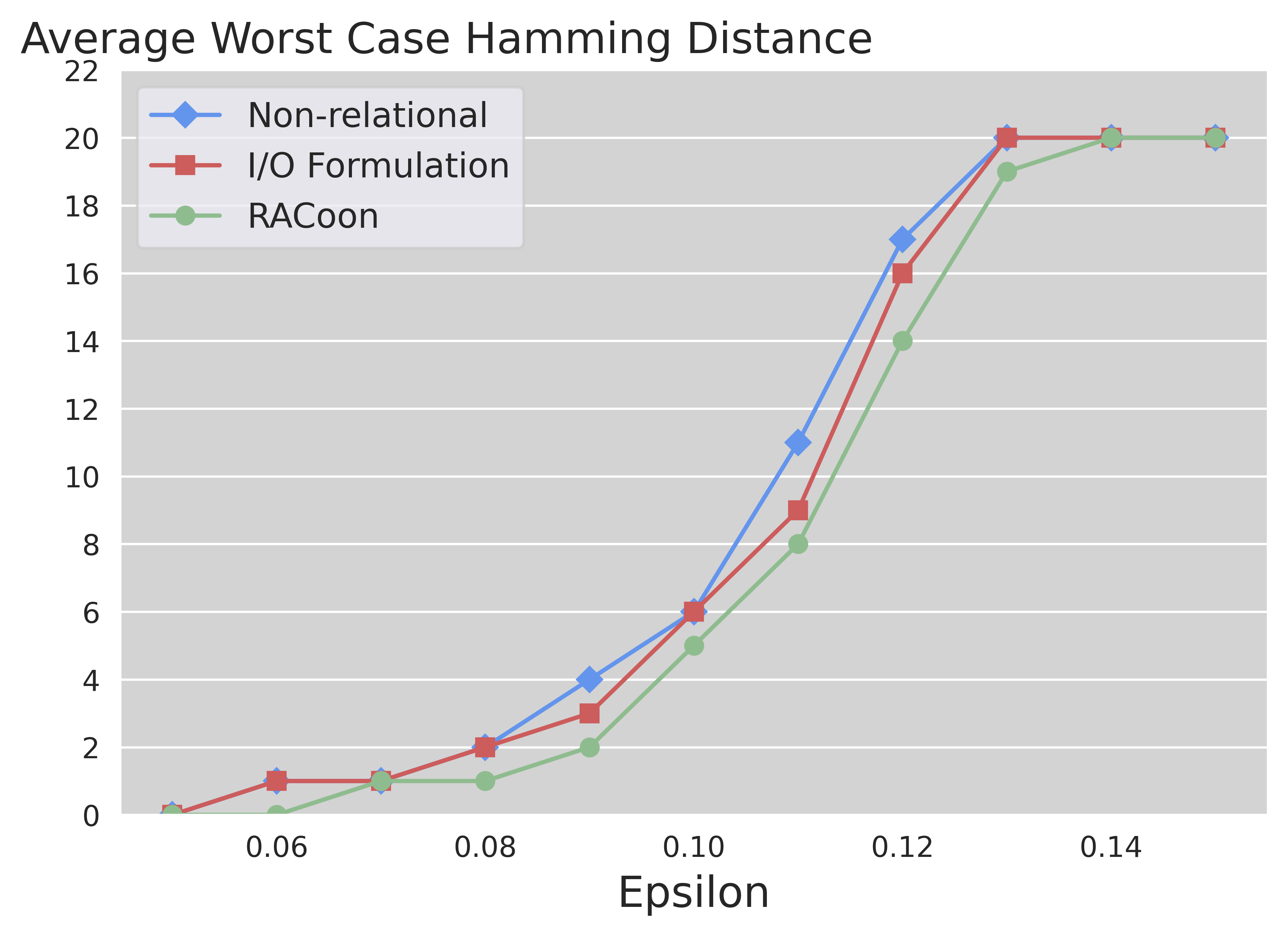}
\captionsetup{labelformat=empty}
\caption{(b) ConvSmall (PGD)}
\end{minipage}\qquad
\addtocounter{figure}{-1}
\begin{minipage}[b]{.3\textwidth}
\includegraphics[width=\textwidth]{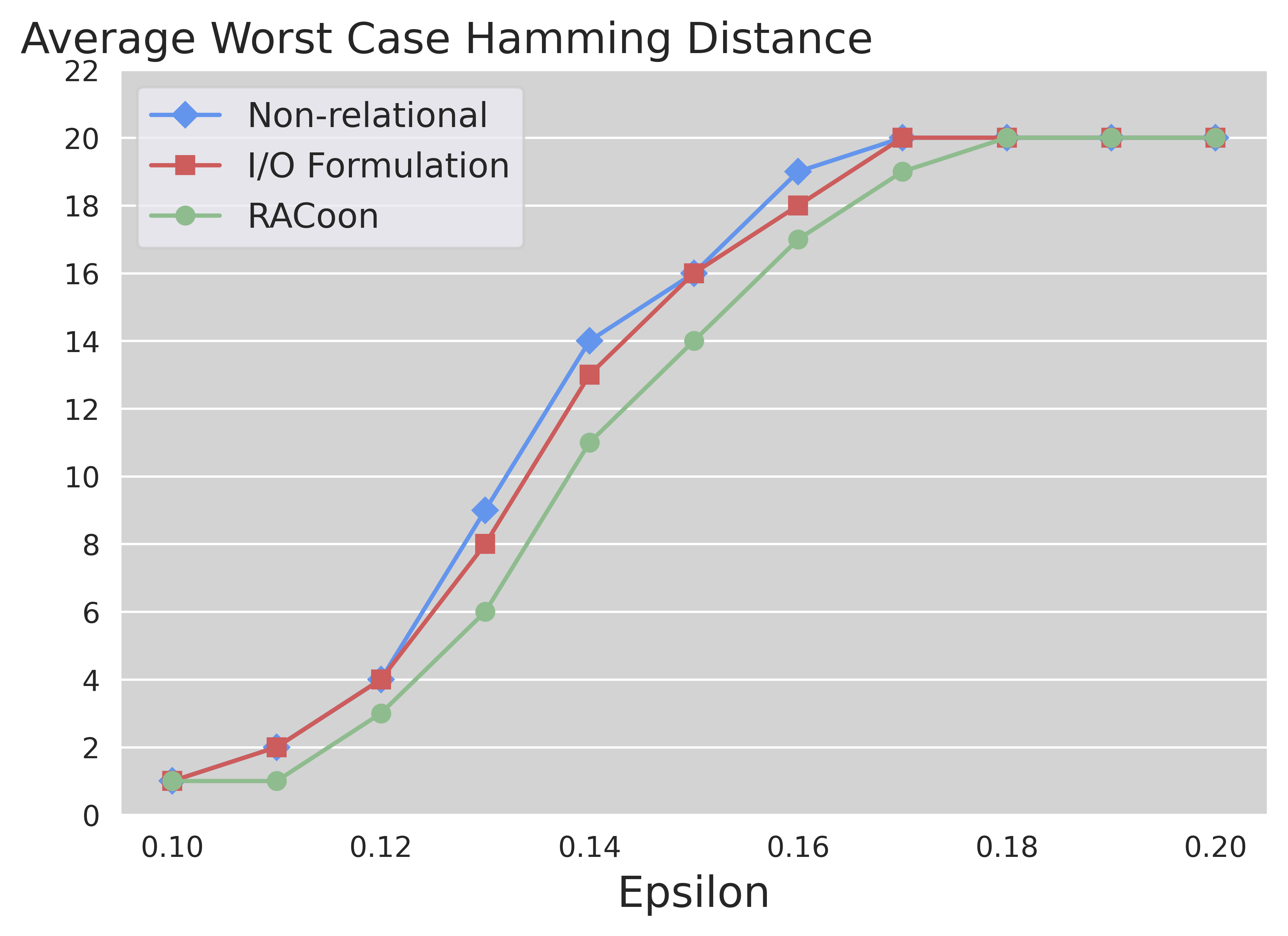}
\captionsetup{labelformat=empty}
\caption{(c) ConvBig (DiffAI)}
\end{minipage}\qquad
\addtocounter{figure}{-1}
\begin{minipage}[b]{.3\textwidth}
\includegraphics[width=\textwidth]{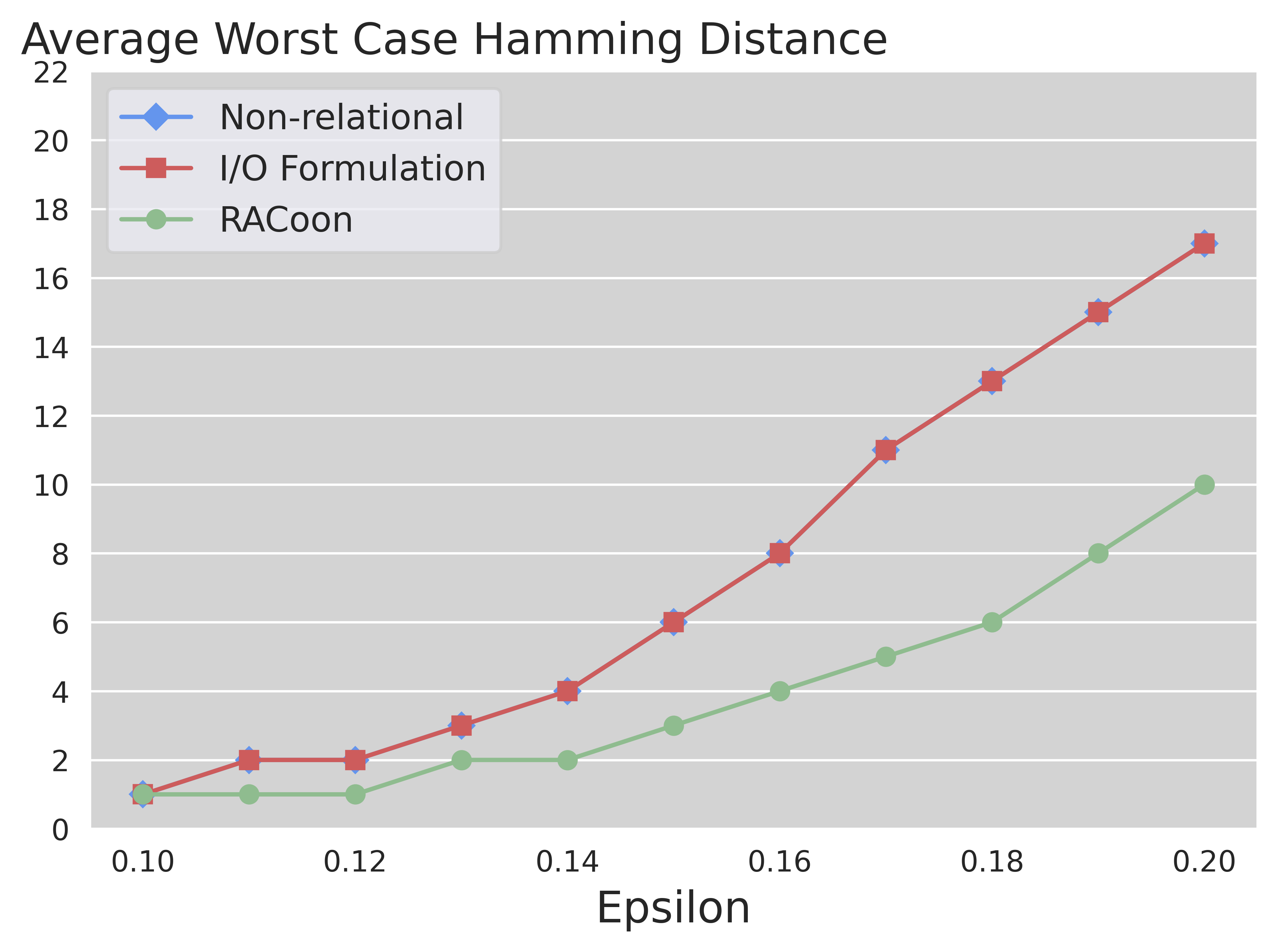}
\captionsetup{labelformat=empty}
\caption{(d) ConvSmall (COLT)}
\end{minipage}\qquad
\addtocounter{figure}{-1}
\begin{minipage}[b]{.3\textwidth}
\includegraphics[width=\textwidth]{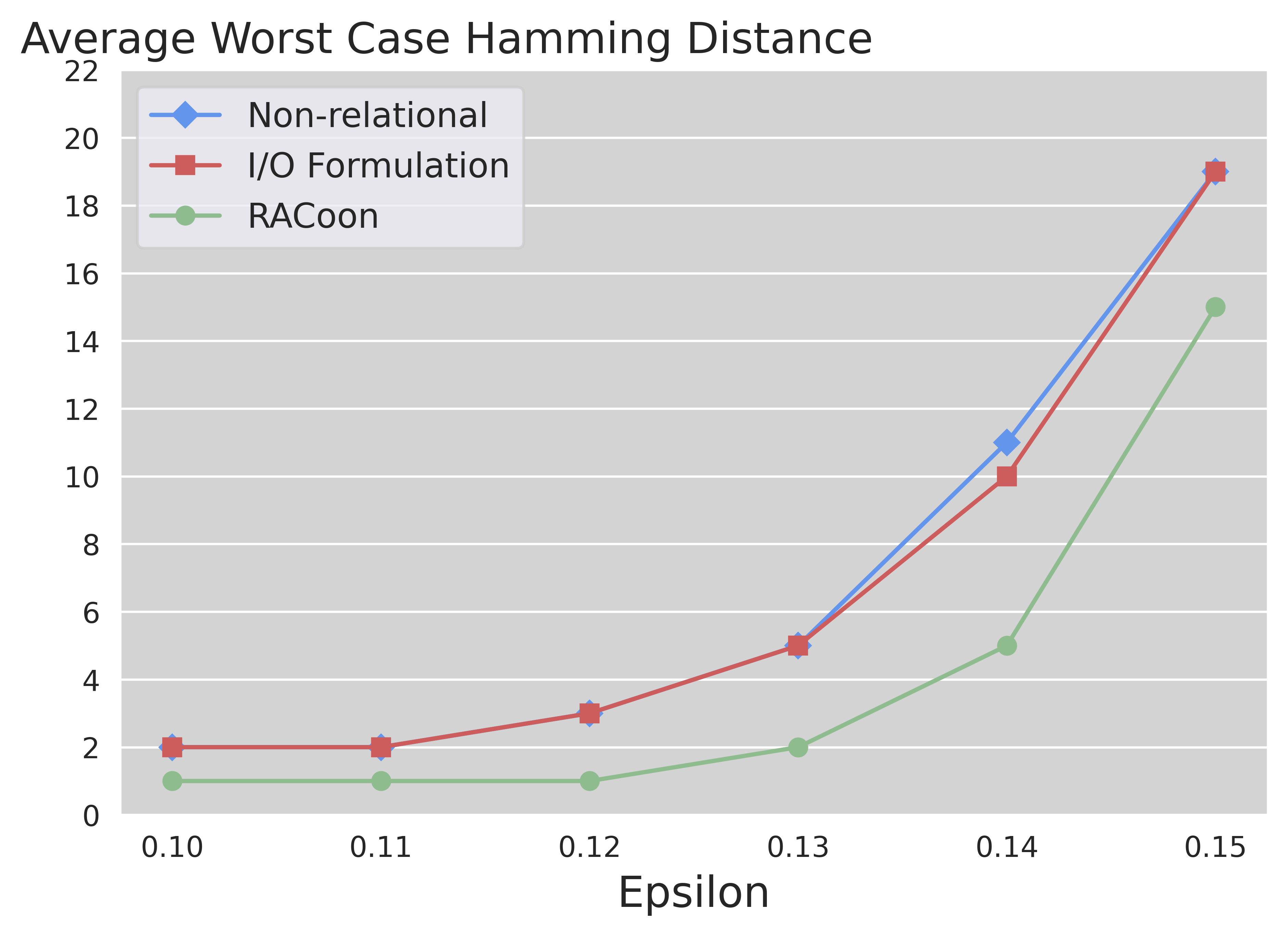}
\captionsetup{labelformat=empty}
\caption{(e) IBPSmall}
\end{minipage}\qquad
\addtocounter{figure}{-1}
\begin{minipage}[b]{.3\textwidth}
\includegraphics[width=\textwidth]{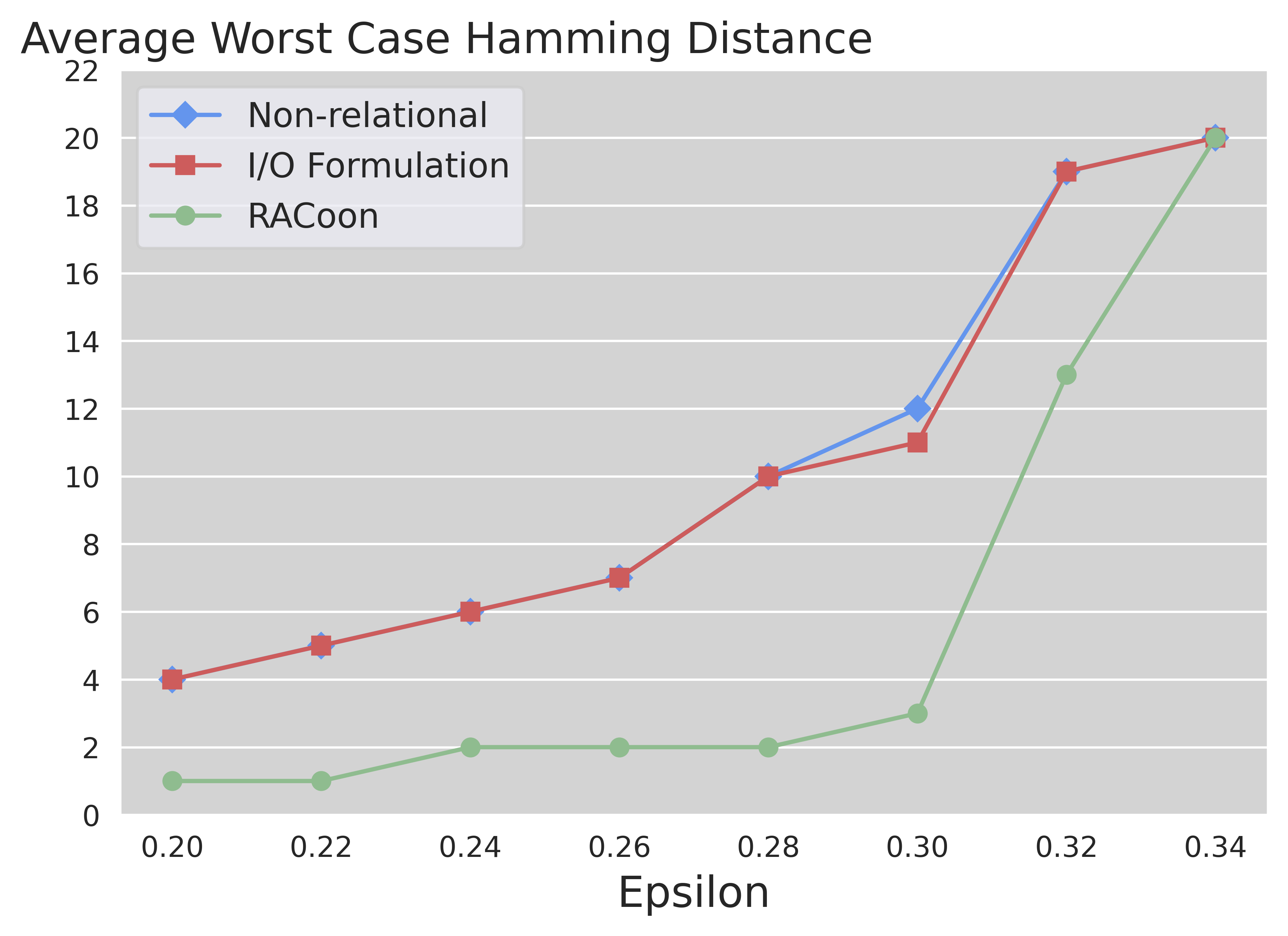}
\captionsetup{labelformat=empty}
\caption{(f) IBPMedium}
\end{minipage}\qquad
\addtocounter{figure}{-1}
\caption{Average worst-case hamming distance for different $\epsilon$ values for networks trained on MNIST on binary strings of length $k=20$.}
\label{fig:mnist_eps_appen_hamming}
\end{figure}

\section{k-UAP verification results for different \texorpdfstring{$\pmb{k}$}{} and \texorpdfstring{$\pmb{\epsilon}$}{} values}
\label{appenexp:diffKEps}
\begin{figure}[htb]
\centering
\begin{minipage}[b]{.18\textwidth}
\includegraphics[width=\textwidth]{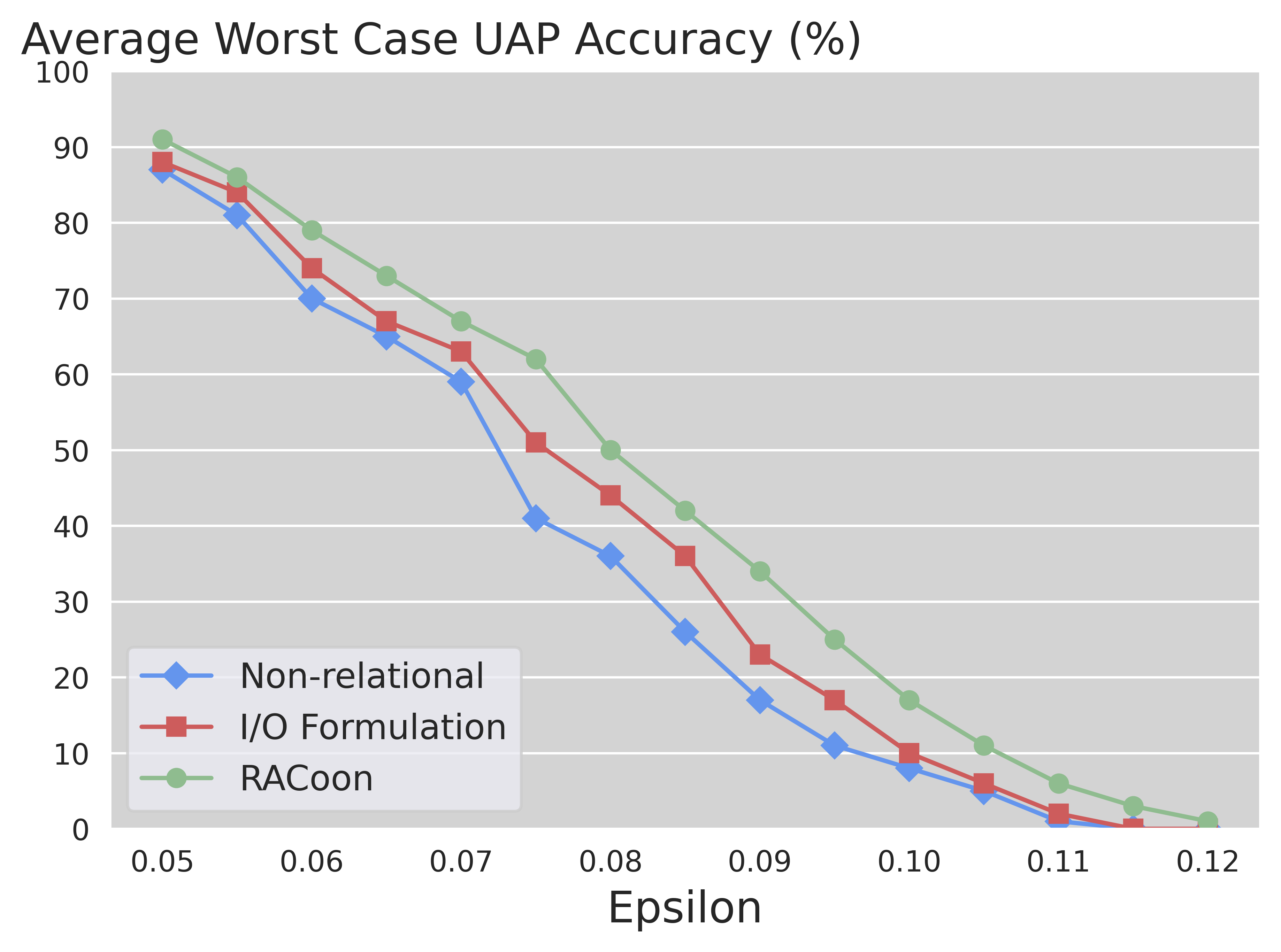}
\captionsetup{labelformat=empty}
\caption{(a) $k = 10$}
\end{minipage}
\addtocounter{figure}{-1}
\begin{minipage}[b]{.18\textwidth}
\includegraphics[width=\textwidth]{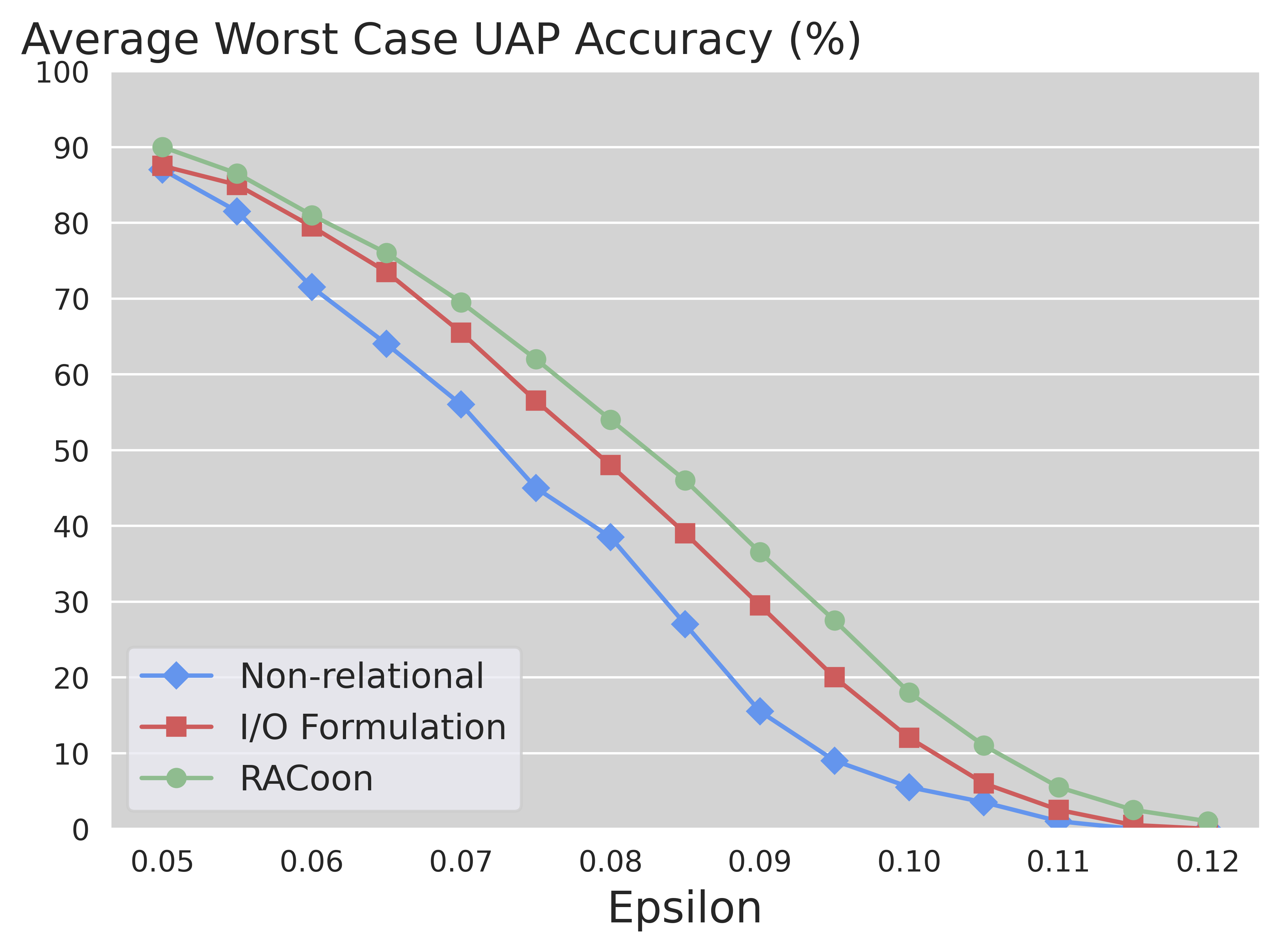}
\captionsetup{labelformat=empty}
\caption{(b) $k = 20$}
\end{minipage}
\addtocounter{figure}{-1}
\begin{minipage}[b]{.18\textwidth}
\includegraphics[width=\textwidth]{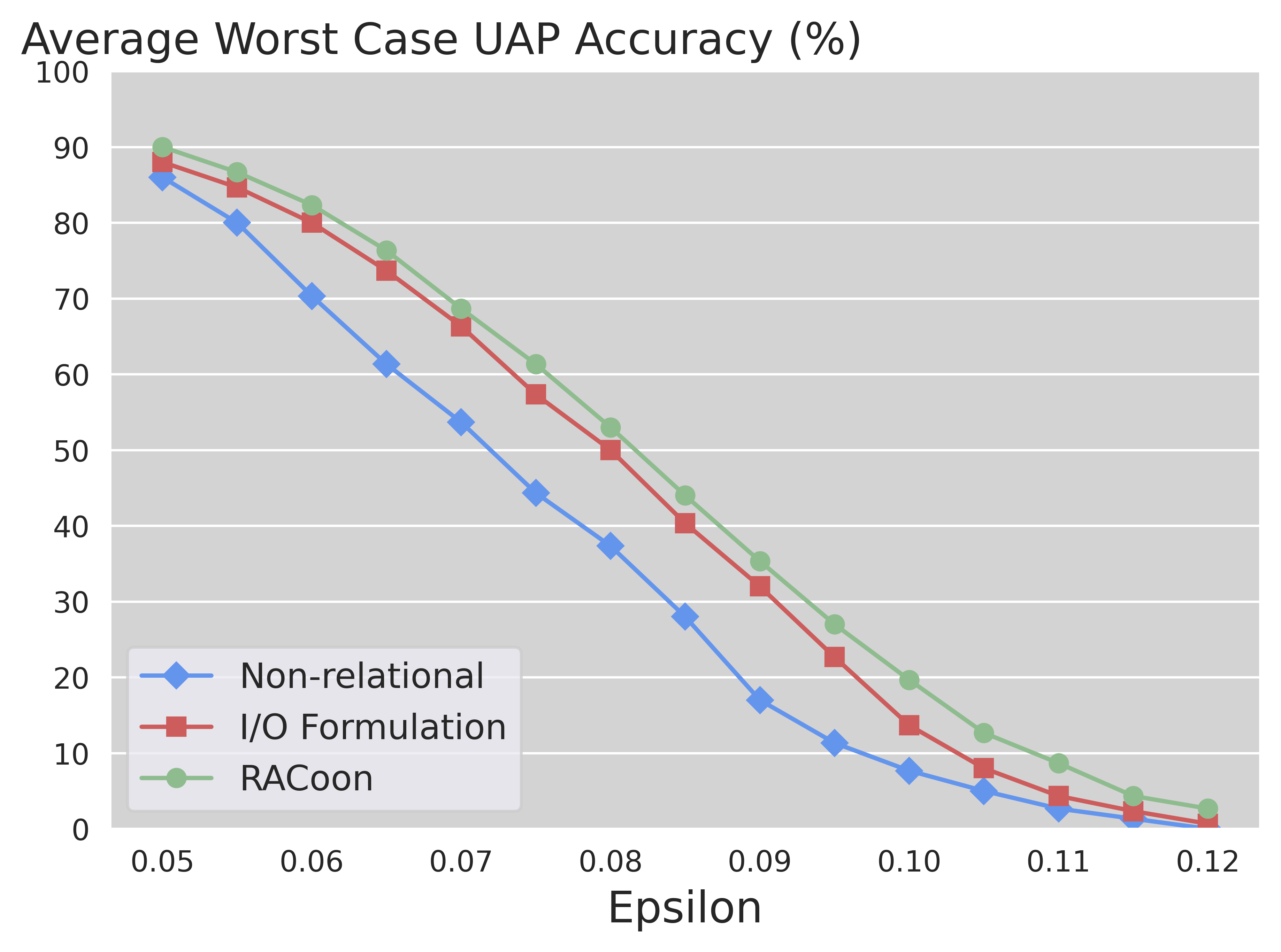}
\captionsetup{labelformat=empty}
\caption{(c) $k = 30$}
\end{minipage}
\addtocounter{figure}{-1}
\begin{minipage}[b]{.18\textwidth}
\includegraphics[width=\textwidth]{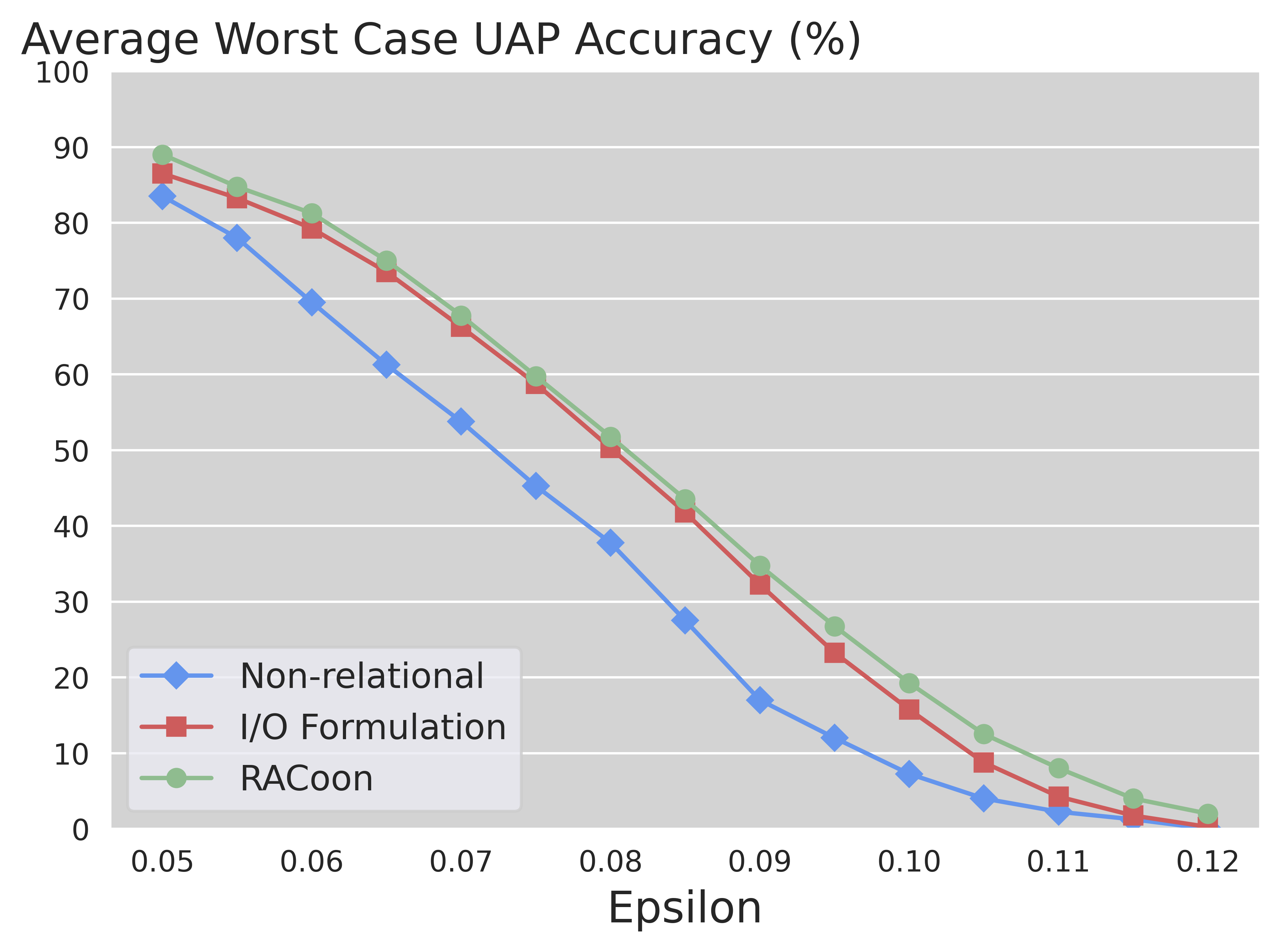}
\captionsetup{labelformat=empty}
\caption{(d) $k = 40$}
\end{minipage}
\addtocounter{figure}{-1}
\begin{minipage}[b]{.18\textwidth}
\includegraphics[width=\textwidth]{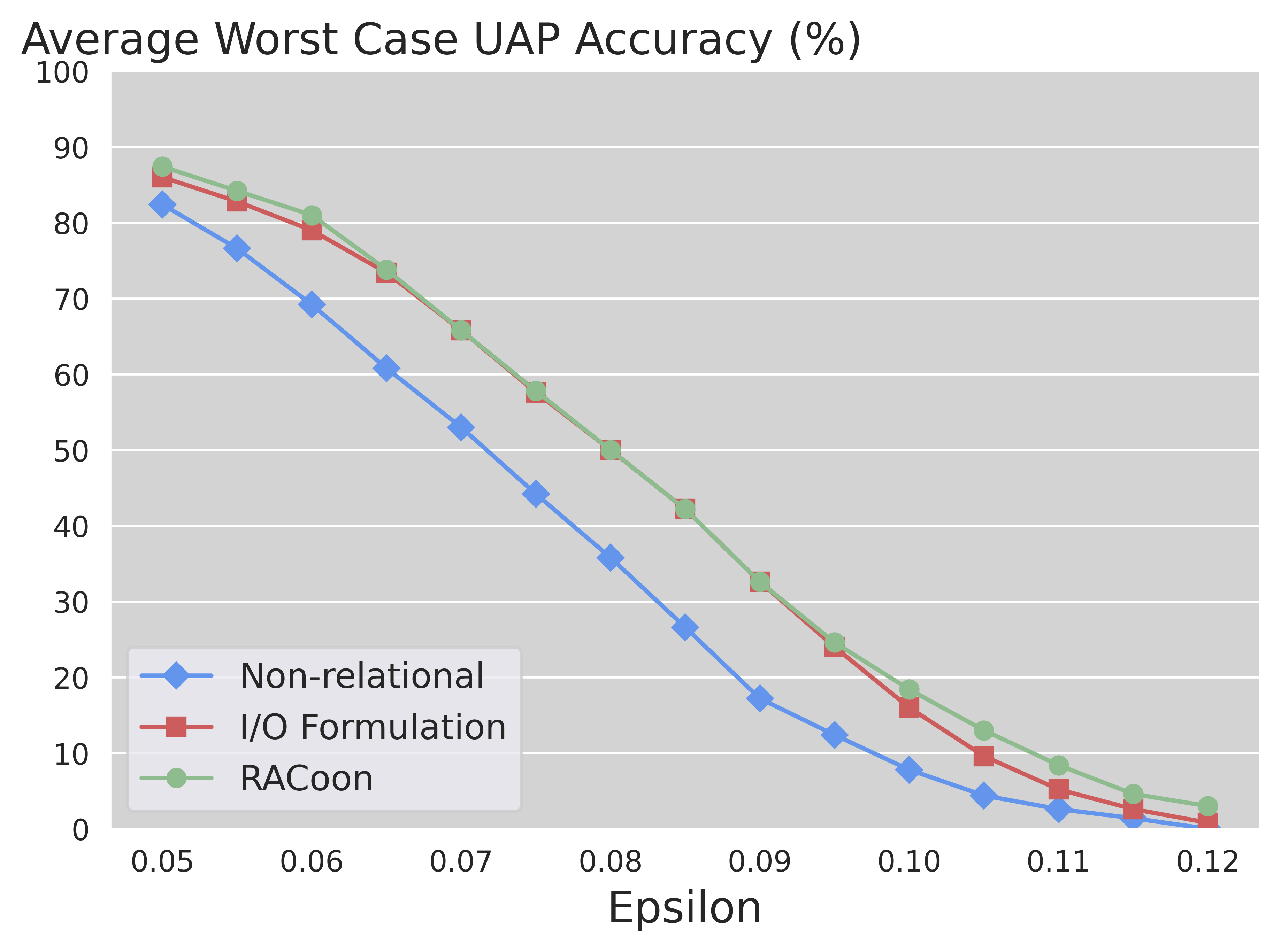}
\captionsetup{labelformat=empty}
\caption{(e) $k = 50$}
\end{minipage}
\addtocounter{figure}{-1}
\caption{Average worst-case UAP accuracy for different $k$ and $\epsilon$ values for ConvSmall Standard MNIST network.}
\label{fig:mnist_point_diff}
\end{figure}

\begin{figure}[htb]
\centering
\begin{minipage}[b]{.18\textwidth}
\includegraphics[width=\textwidth]{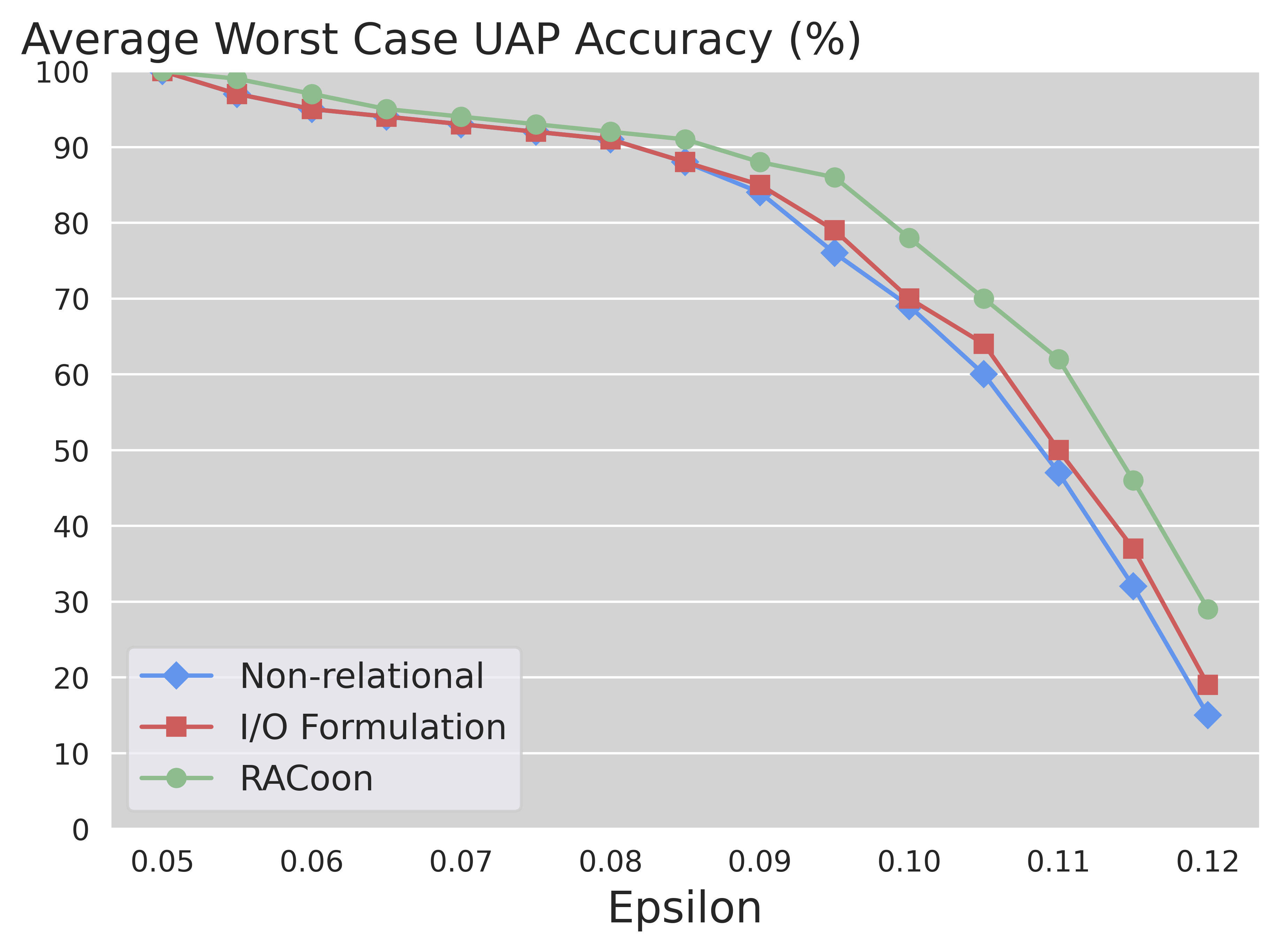}
\captionsetup{labelformat=empty}
\caption{(a) $k = 10$}
\end{minipage}
\addtocounter{figure}{-1}
\begin{minipage}[b]{.18\textwidth}
\includegraphics[width=\textwidth]{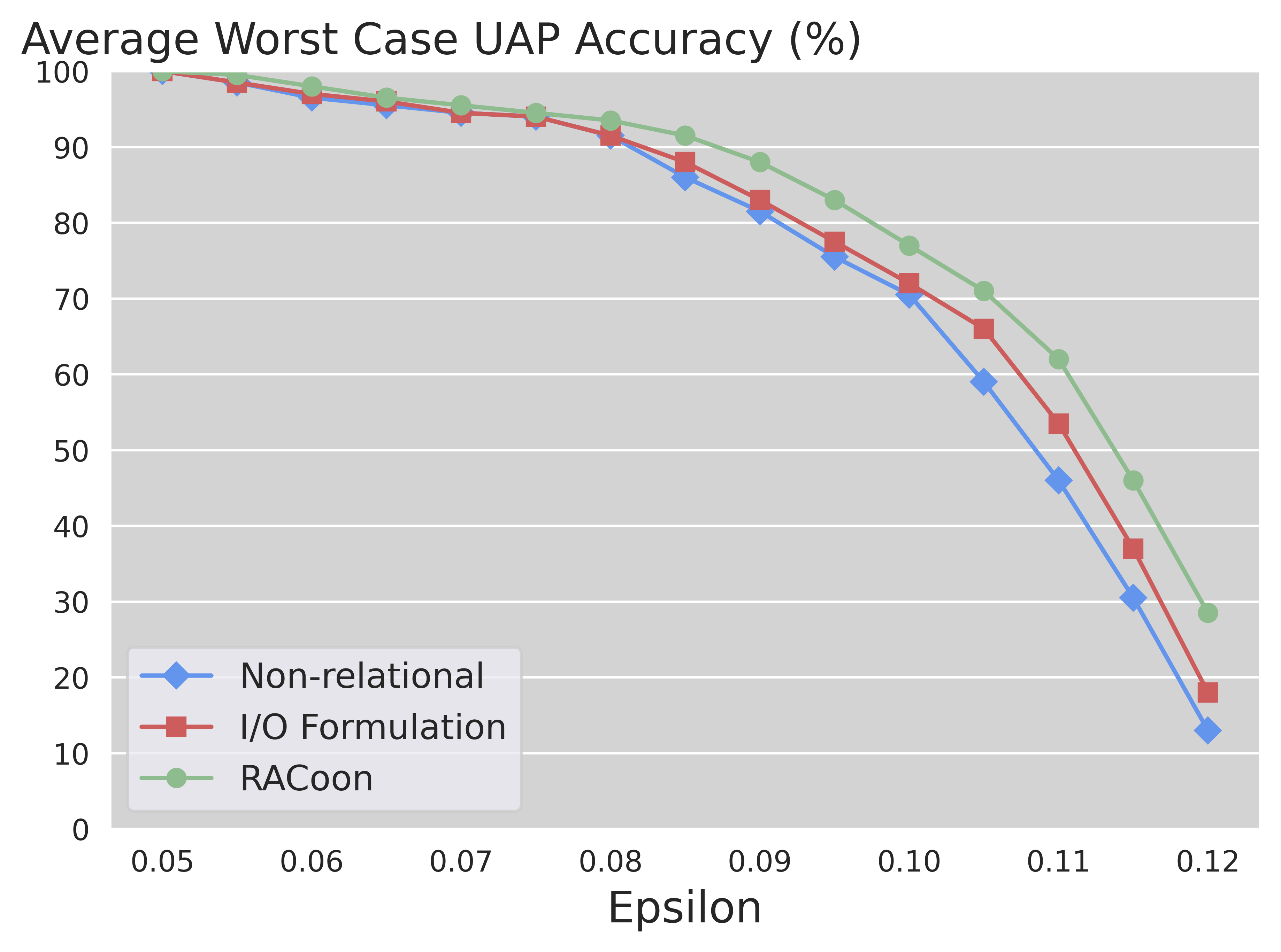}
\captionsetup{labelformat=empty}
\caption{(b) $k = 20$}
\end{minipage}
\addtocounter{figure}{-1}
\begin{minipage}[b]{.18\textwidth}
\includegraphics[width=\textwidth]{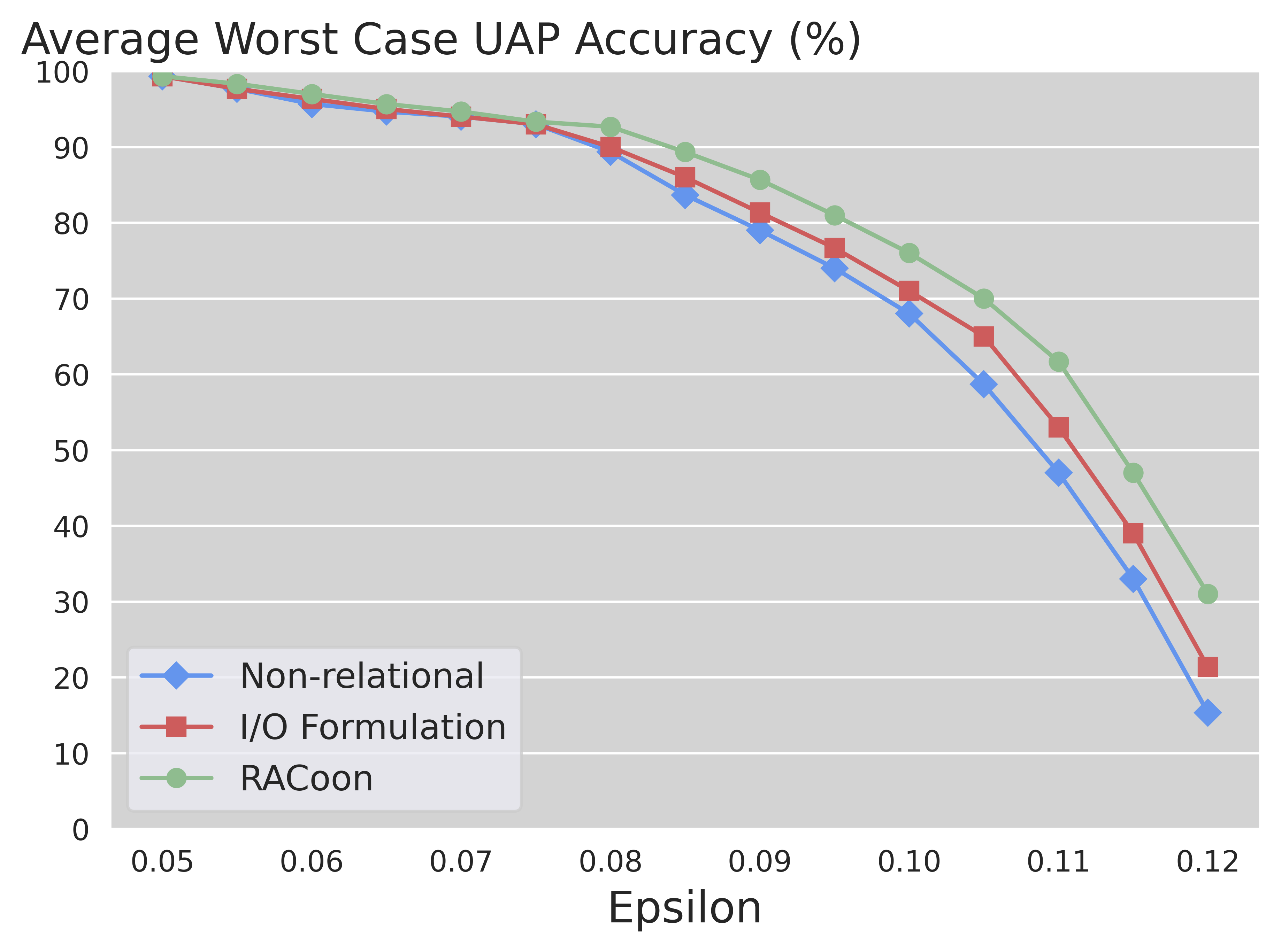}
\captionsetup{labelformat=empty}
\caption{(c) $k = 30$}
\end{minipage}
\addtocounter{figure}{-1}
\begin{minipage}[b]{.18\textwidth}
\includegraphics[width=\textwidth]{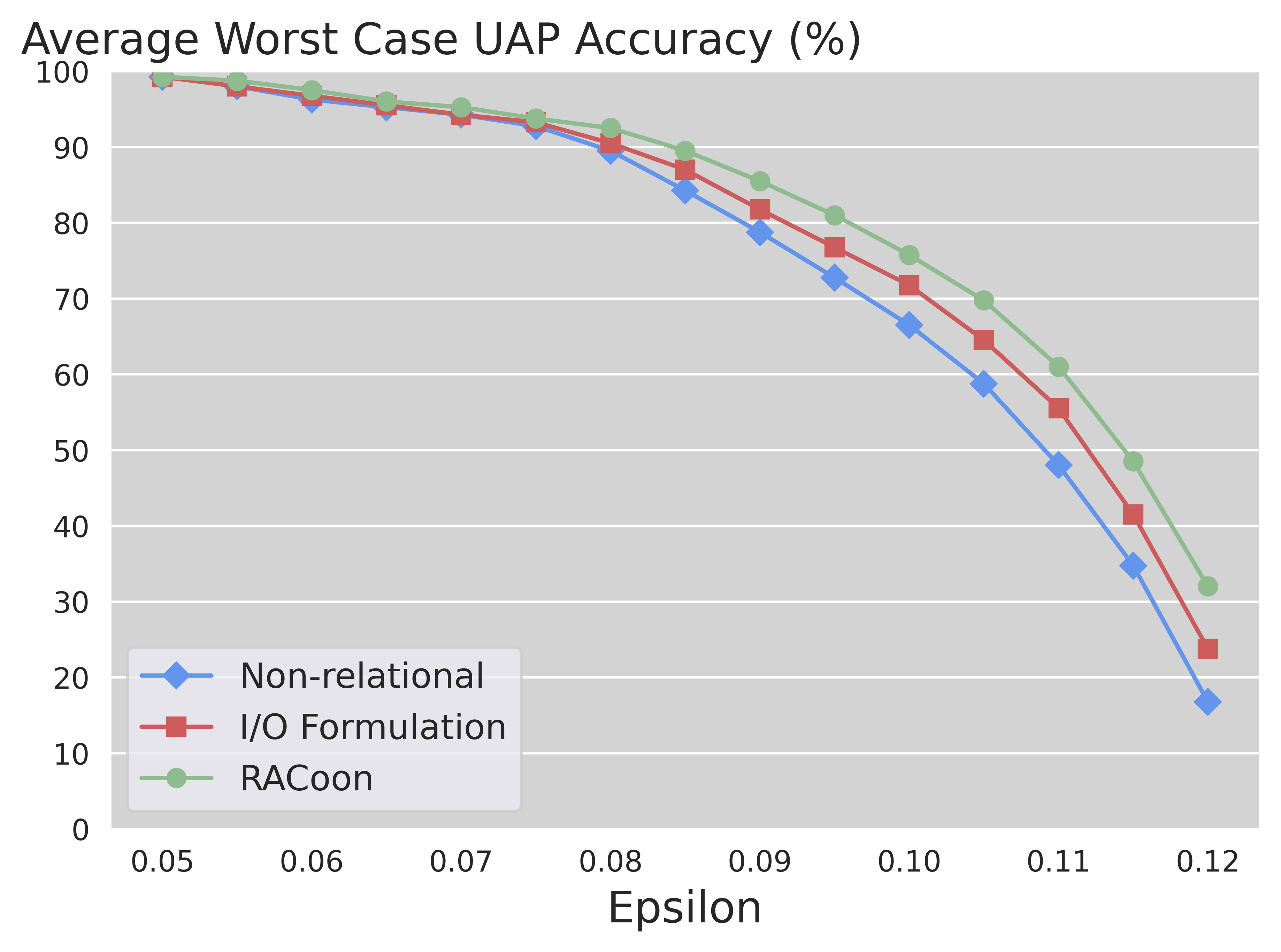}
\captionsetup{labelformat=empty}
\caption{(d) $k = 40$}
\end{minipage}
\addtocounter{figure}{-1}
\begin{minipage}[b]{.18\textwidth}
\includegraphics[width=\textwidth]{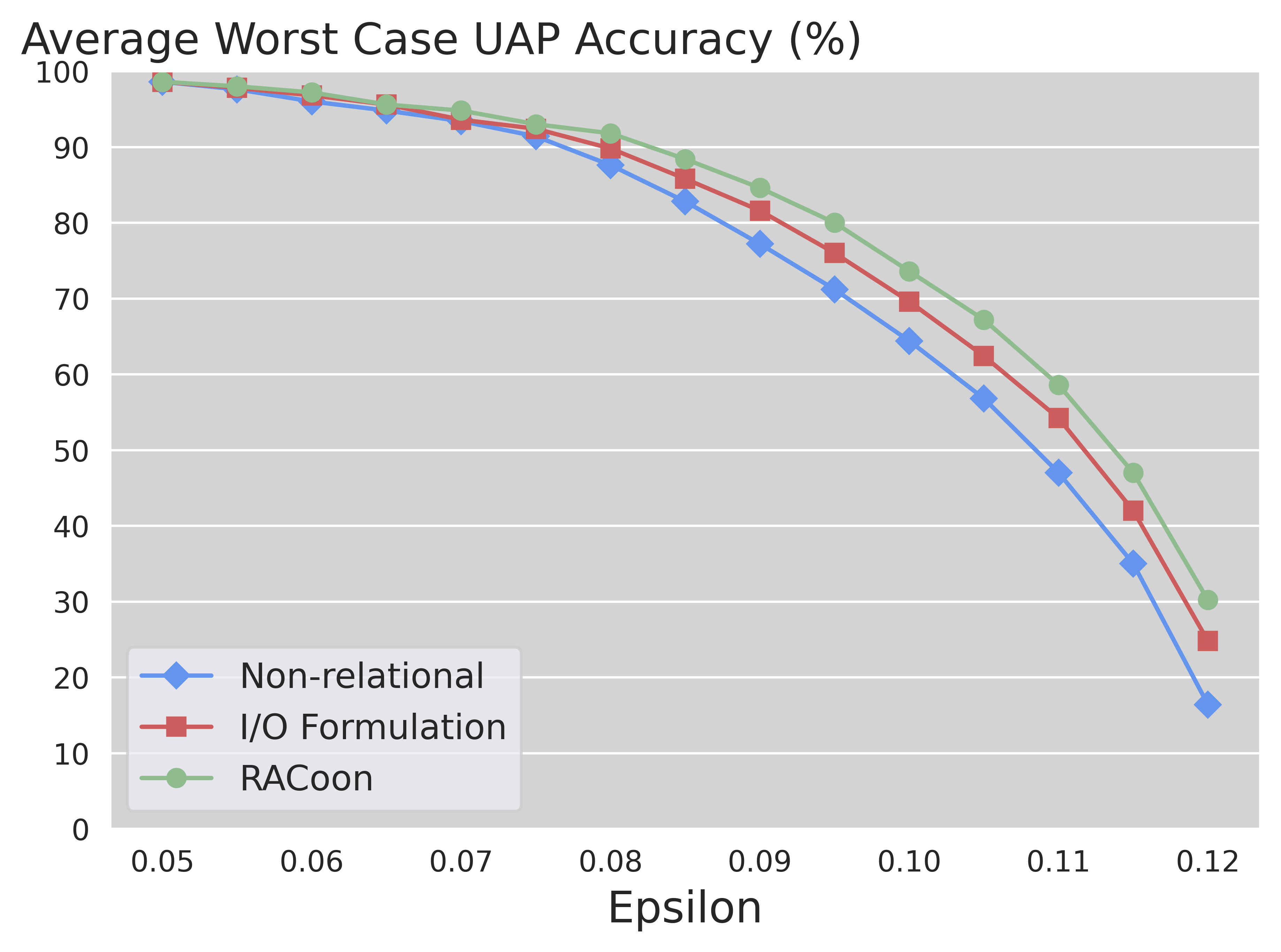}
\captionsetup{labelformat=empty}
\caption{(e) $k = 50$}
\end{minipage}
\addtocounter{figure}{-1}
\caption{Average worst-case UAP accuracy for different $k$ and $\epsilon$ values for ConvSmall PGD MNIST network.}
\label{fig:mnist_pgd_diff}
\end{figure}

\begin{figure}[htb]
\centering
\begin{minipage}[b]{.18\textwidth}
\includegraphics[width=\textwidth]{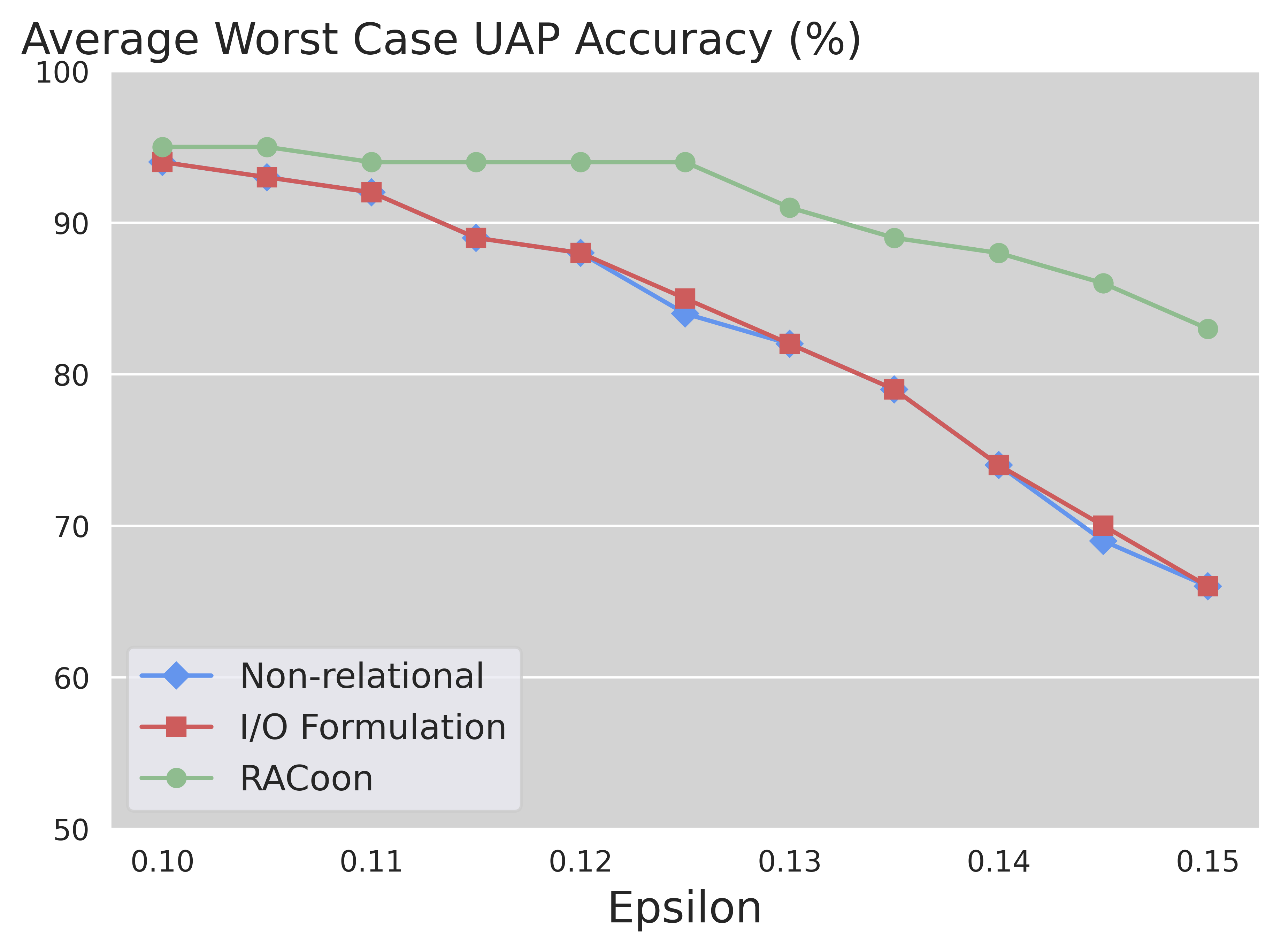}
\captionsetup{labelformat=empty}
\caption{(a) $k = 10$}
\end{minipage}
\addtocounter{figure}{-1}
\begin{minipage}[b]{.18\textwidth}
\includegraphics[width=\textwidth]{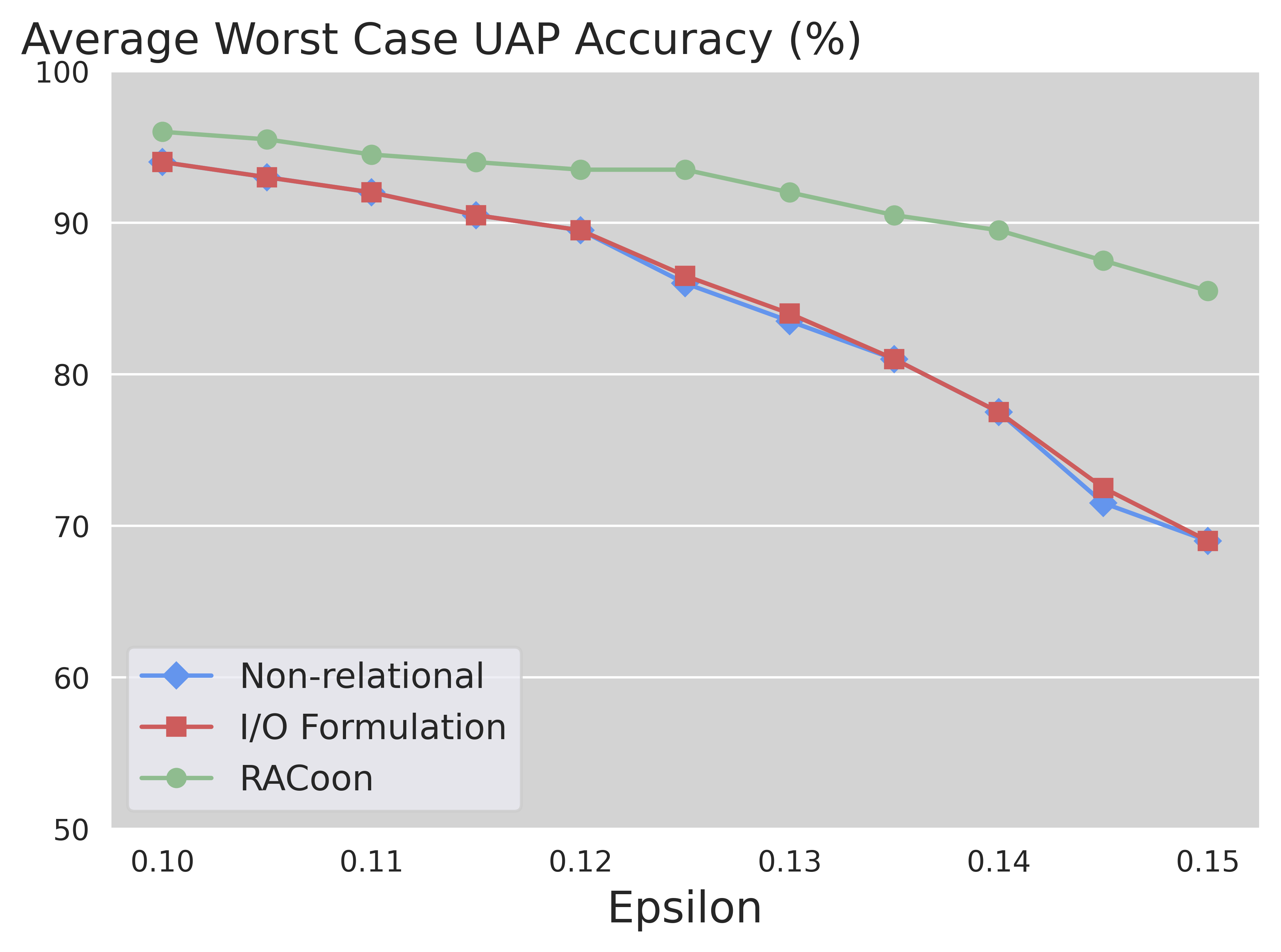}
\captionsetup{labelformat=empty}
\caption{(b) $k = 20$}
\end{minipage}
\addtocounter{figure}{-1}
\begin{minipage}[b]{.18\textwidth}
\includegraphics[width=\textwidth]{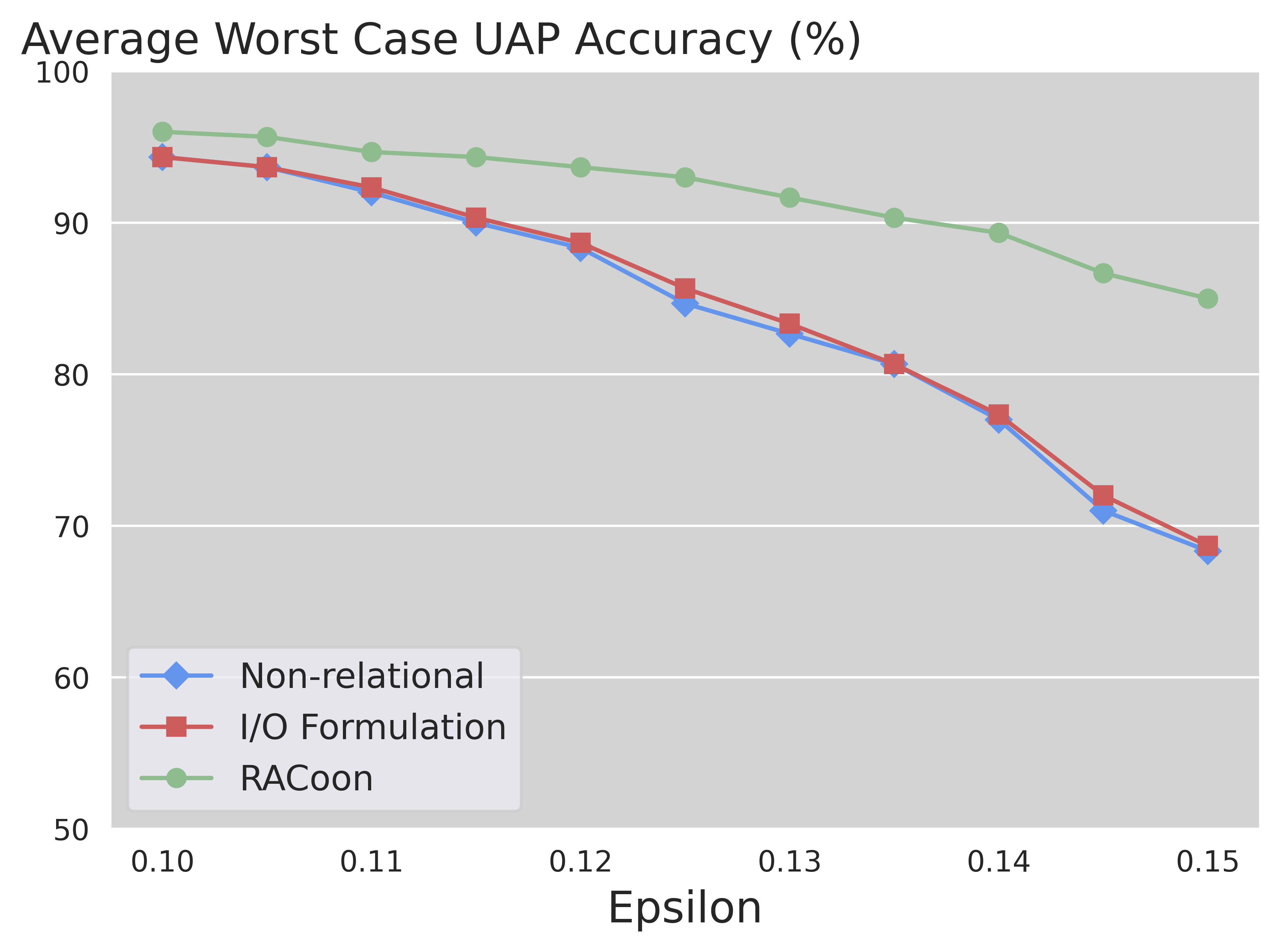}
\captionsetup{labelformat=empty}
\caption{(c) $k = 30$}
\end{minipage}
\addtocounter{figure}{-1}
\begin{minipage}[b]{.18\textwidth}
\includegraphics[width=\textwidth]{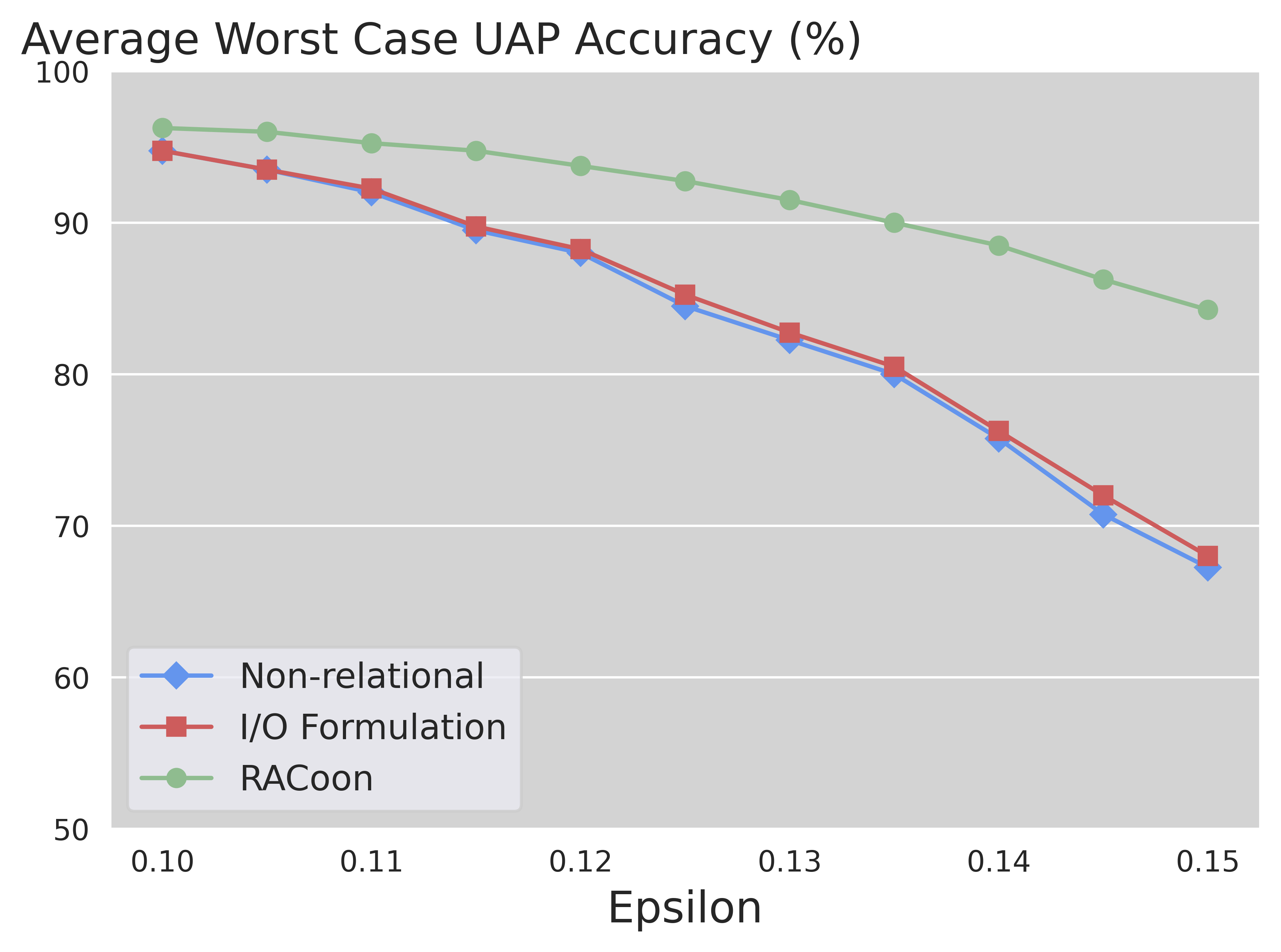}
\captionsetup{labelformat=empty}
\caption{(d) $k = 40$}
\end{minipage}
\addtocounter{figure}{-1}
\begin{minipage}[b]{.18\textwidth}
\includegraphics[width=\textwidth]{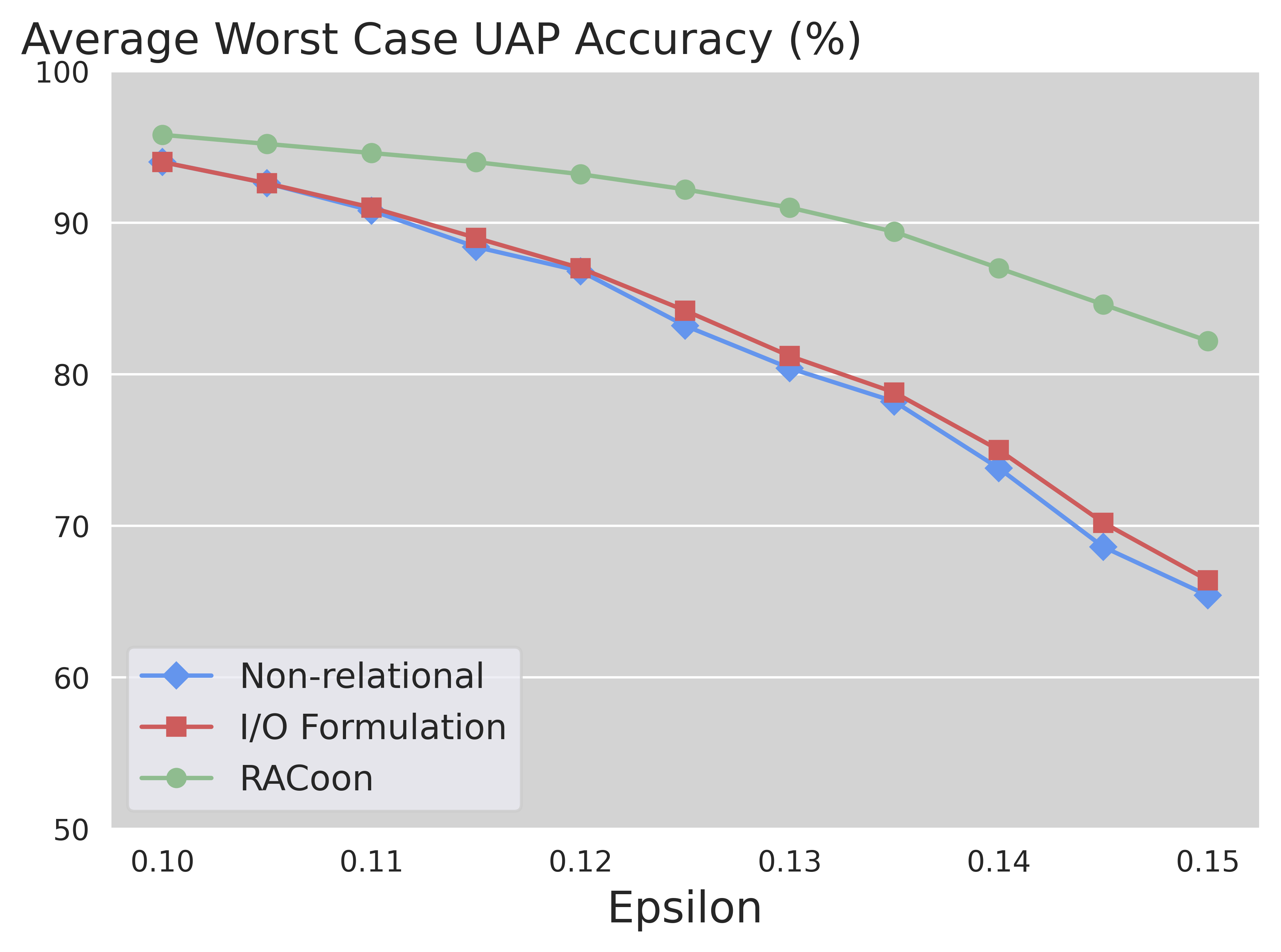}
\captionsetup{labelformat=empty}
\caption{(e) $k = 50$}
\end{minipage}
\addtocounter{figure}{-1}
\caption{Average worst-case UAP accuracy for different $k$ and $\epsilon$ values for ConvSmall COLT MNIST network.}
\label{fig:mnist_colt_diff}
\end{figure}

\begin{figure}[htb]
\centering
\begin{minipage}[b]{.18\textwidth}
\includegraphics[width=\textwidth]{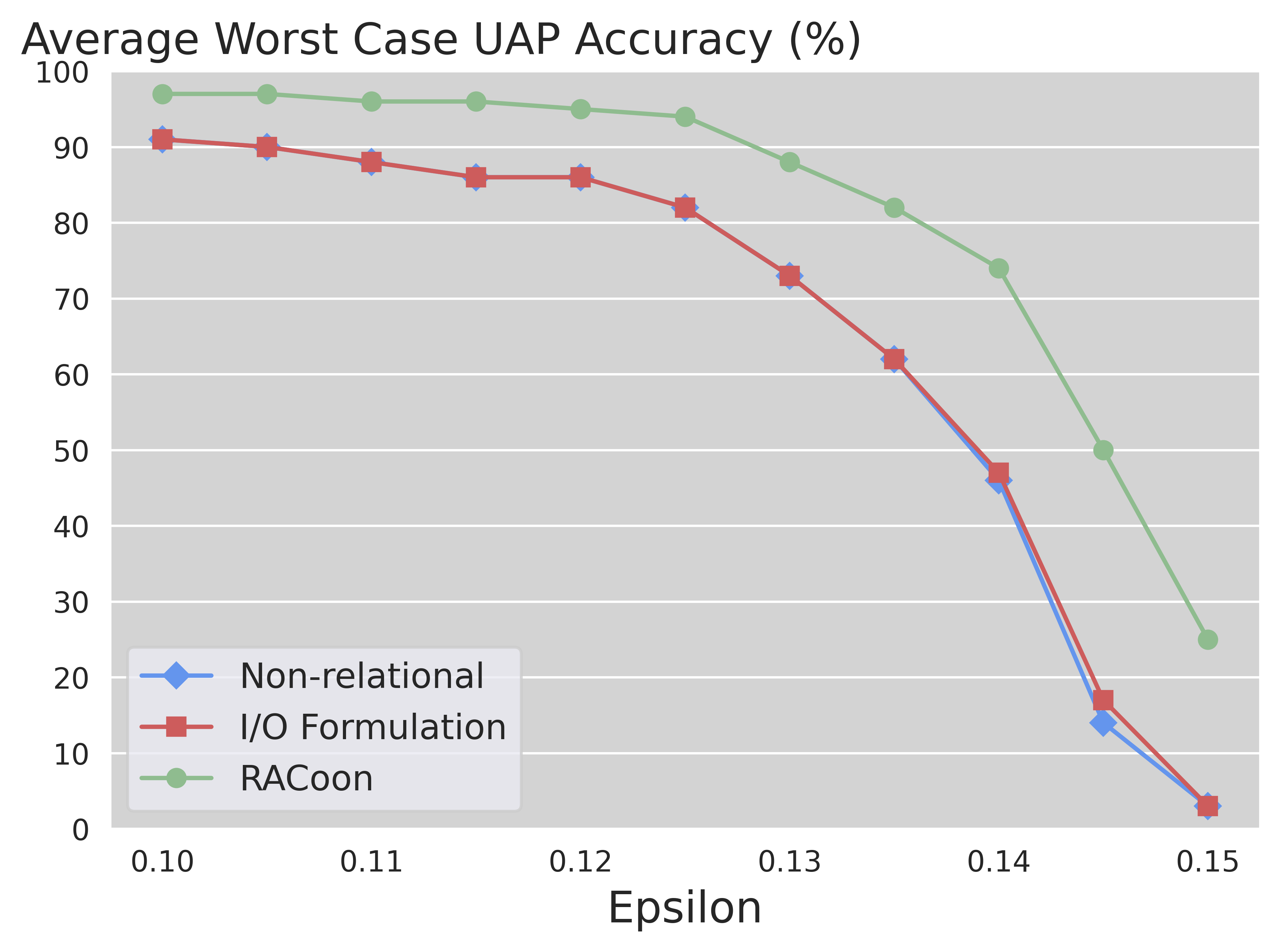}
\captionsetup{labelformat=empty}
\caption{(a) $k = 10$}
\end{minipage}
\addtocounter{figure}{-1}
\begin{minipage}[b]{.18\textwidth}
\includegraphics[width=\textwidth]{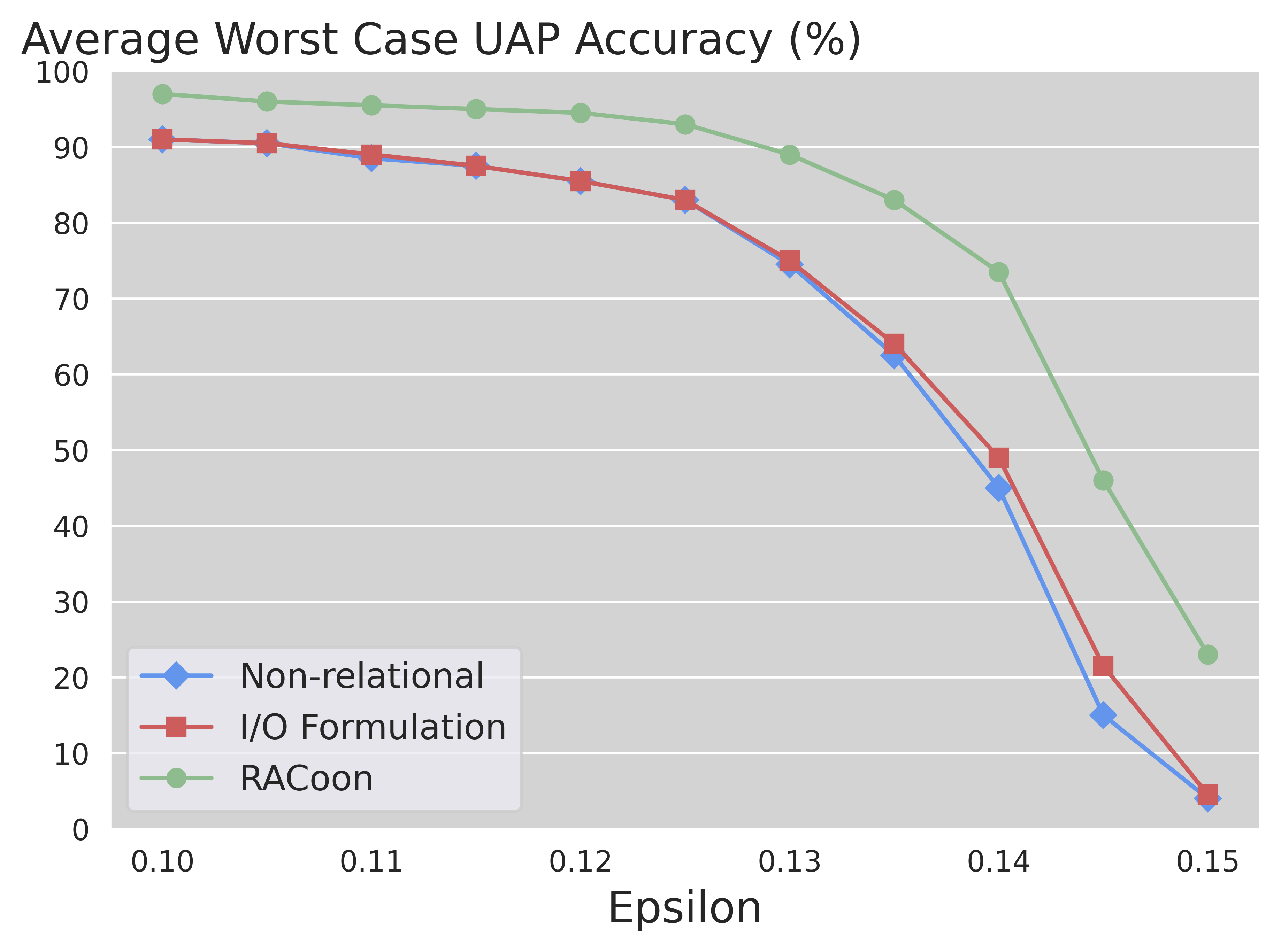}
\captionsetup{labelformat=empty}
\caption{(b) $k = 20$}
\end{minipage}
\addtocounter{figure}{-1}
\begin{minipage}[b]{.18\textwidth}
\includegraphics[width=\textwidth]{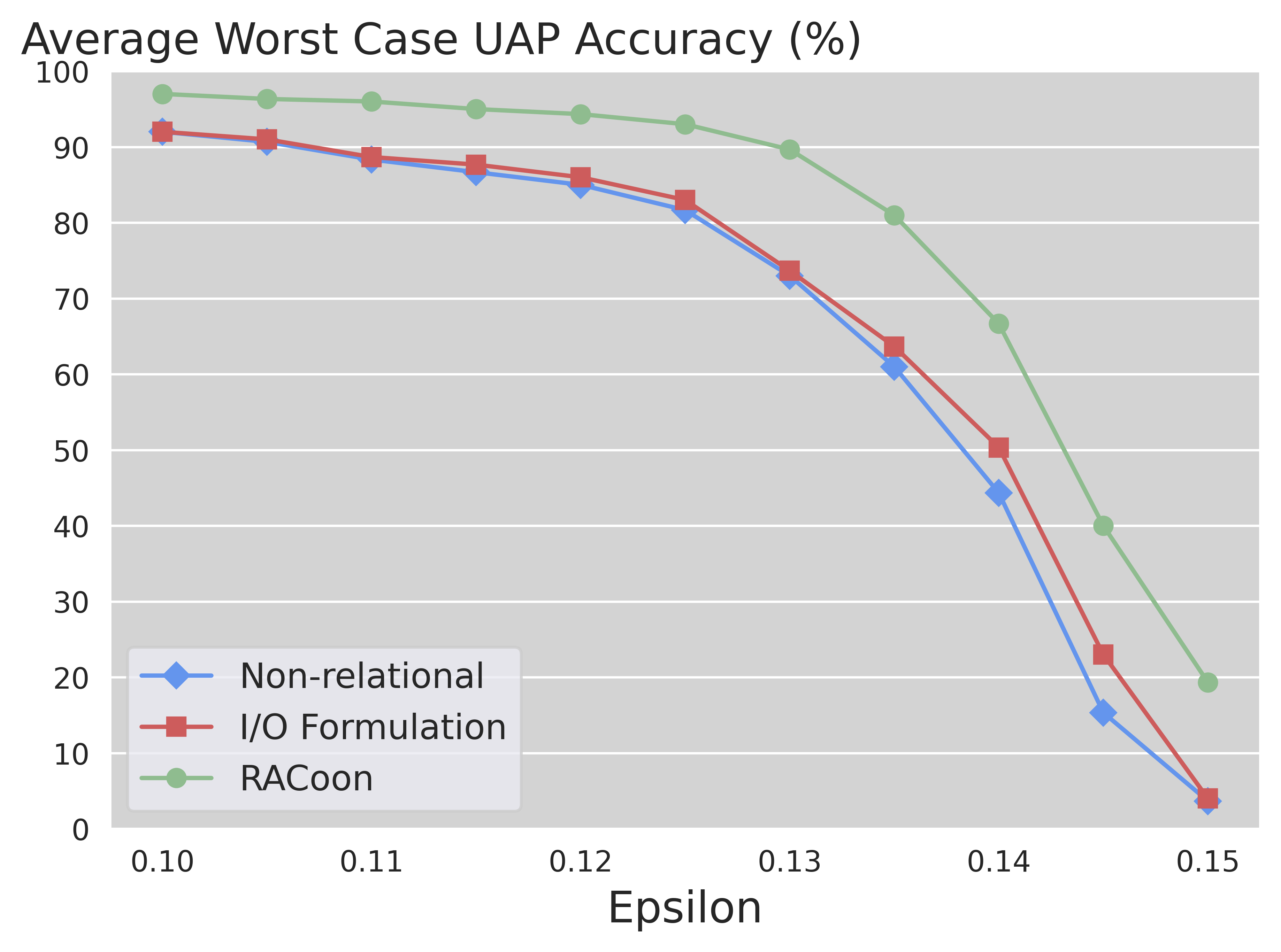}
\captionsetup{labelformat=empty}
\caption{(c) $k = 30$}
\end{minipage}
\addtocounter{figure}{-1}
\begin{minipage}[b]{.18\textwidth}
\includegraphics[width=\textwidth]{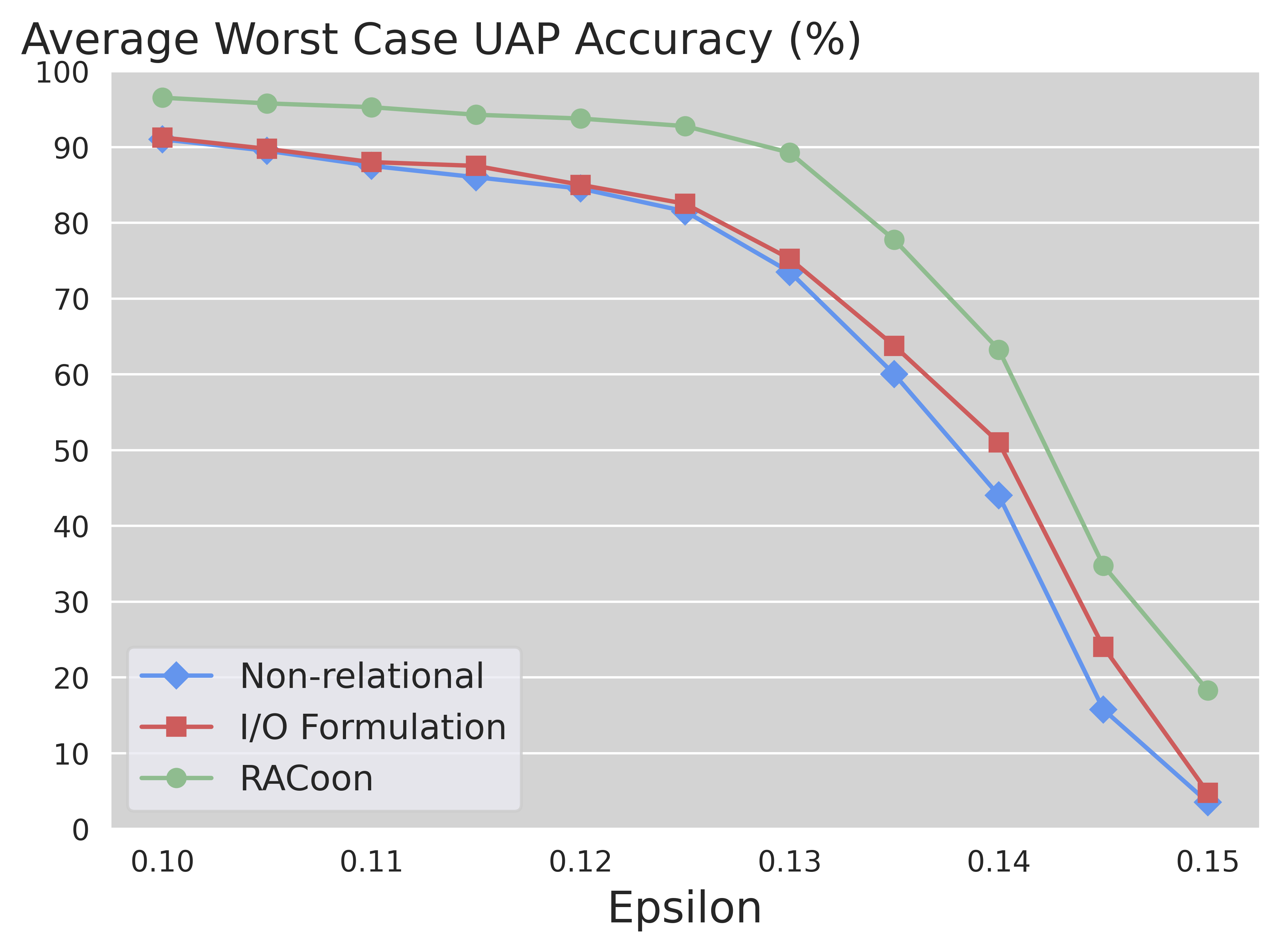}
\captionsetup{labelformat=empty}
\caption{(d) $k = 40$}
\end{minipage}
\addtocounter{figure}{-1}
\begin{minipage}[b]{.18\textwidth}
\includegraphics[width=\textwidth]{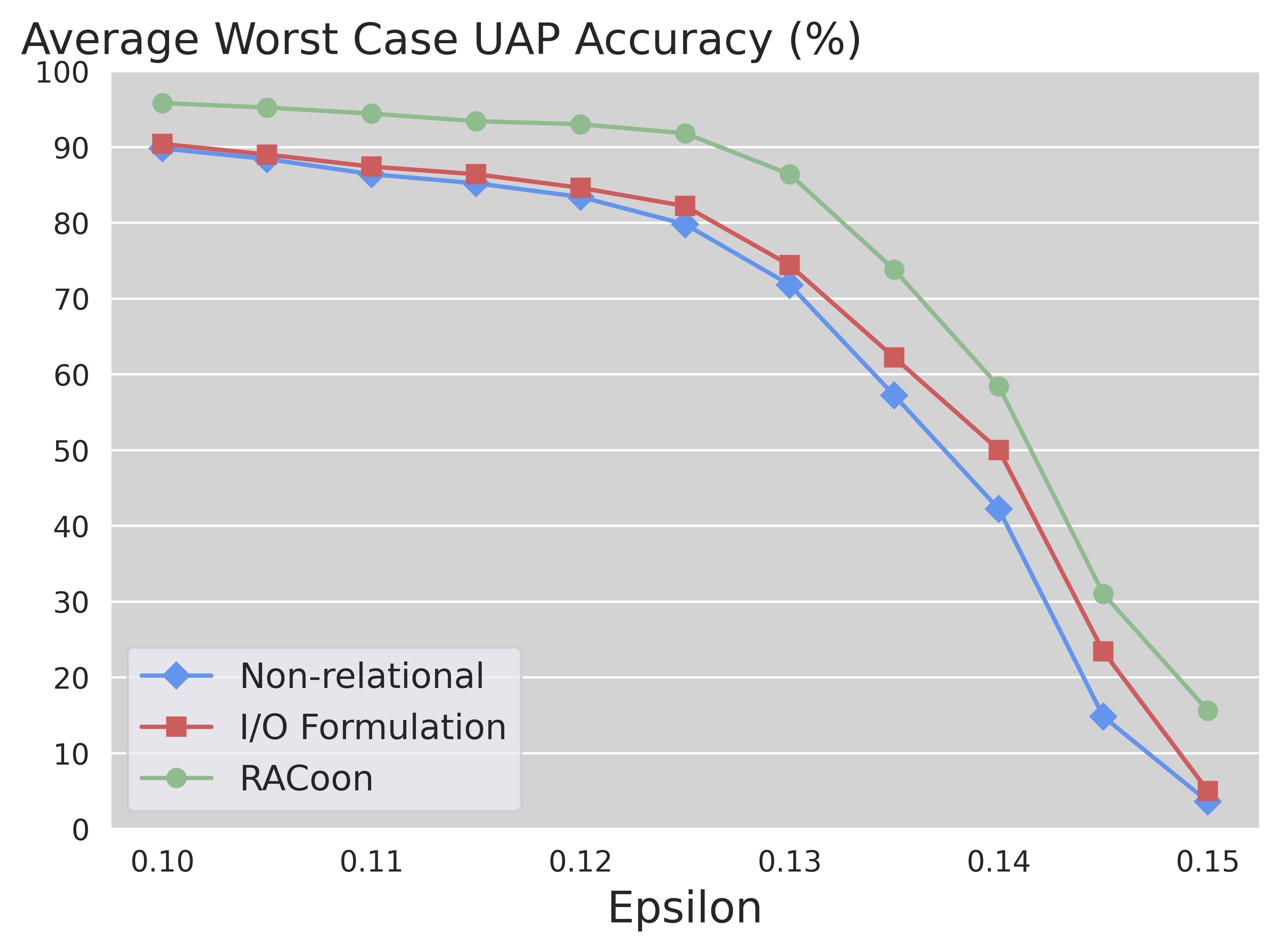}
\captionsetup{labelformat=empty}
\caption{(e) $k = 50$}
\end{minipage}
\addtocounter{figure}{-1}
\caption{Average worst-case UAP accuracy for different $k$ and $\epsilon$ values for IBPSmall MNIST network.}
\label{fig:mnist_ibp_small_diff}
\end{figure}

\begin{figure}[htb]
\centering
\begin{minipage}[b]{.18\textwidth}
\includegraphics[width=\textwidth]{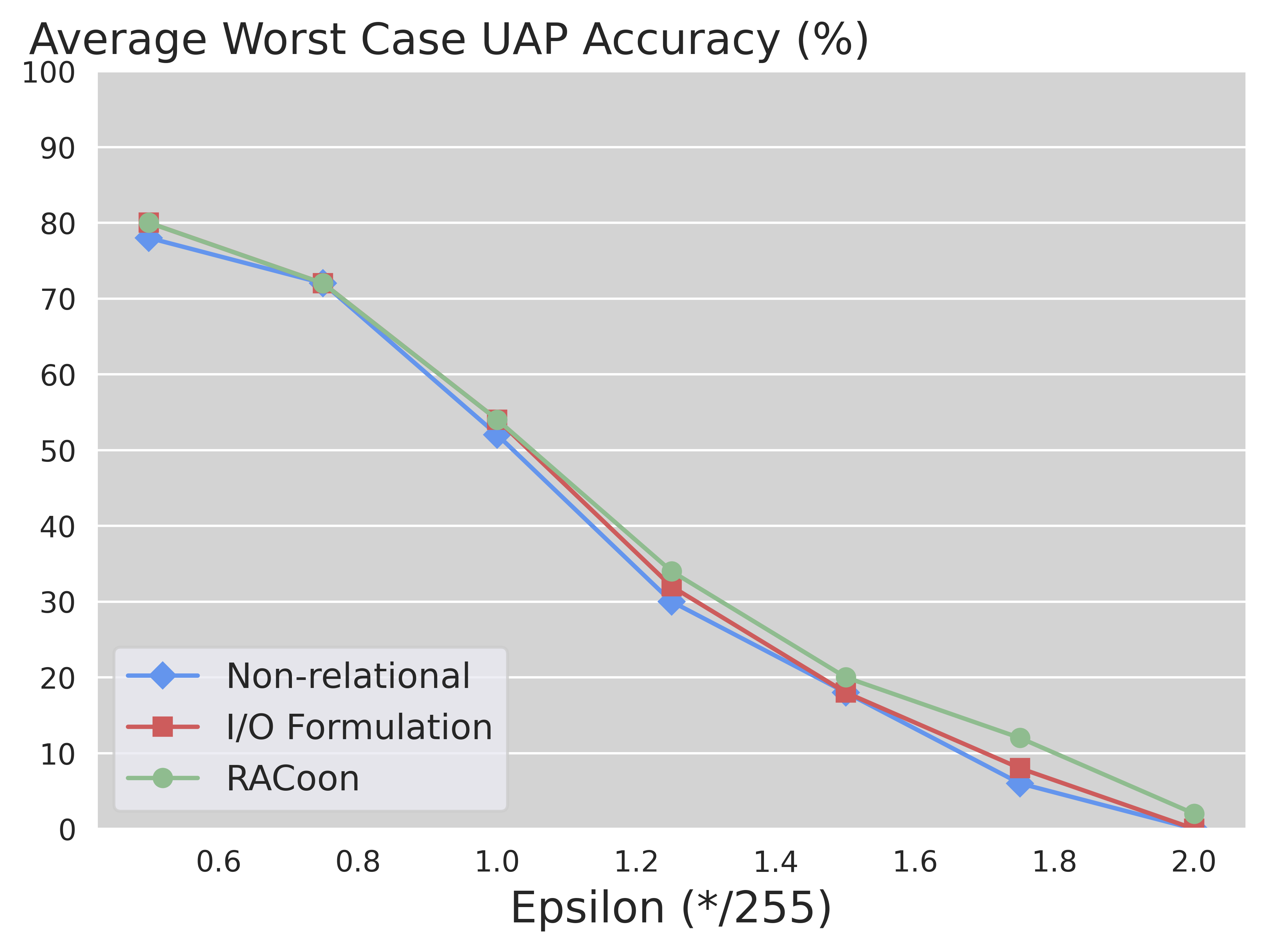}
\captionsetup{labelformat=empty}
\caption{(a) $k = 5$}
\end{minipage}
\addtocounter{figure}{-1}
\begin{minipage}[b]{.18\textwidth}
\includegraphics[width=\textwidth]{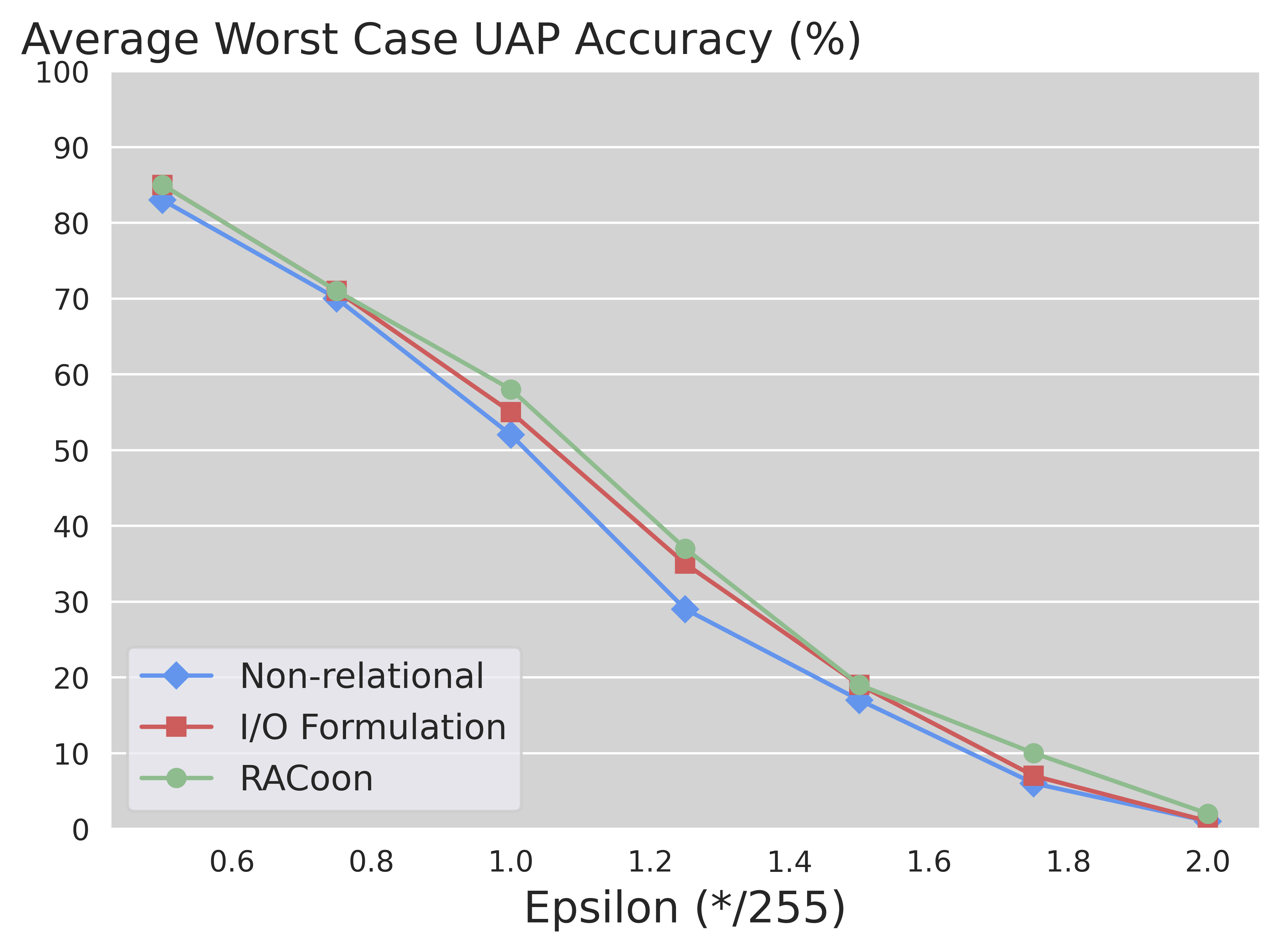}
\captionsetup{labelformat=empty}
\caption{(b) $k = 10$}
\end{minipage}
\addtocounter{figure}{-1}
\begin{minipage}[b]{.18\textwidth}
\includegraphics[width=\textwidth]{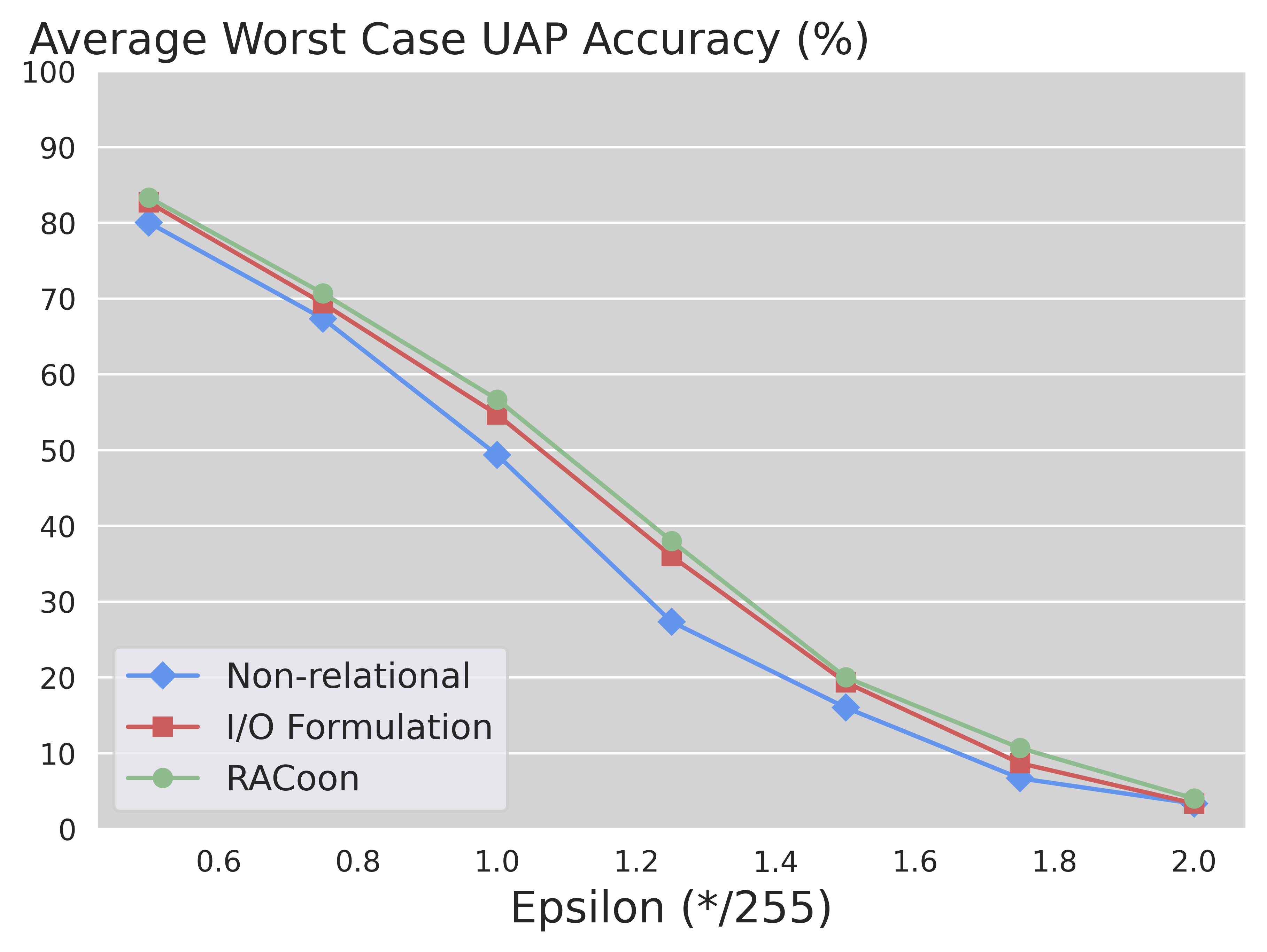}
\captionsetup{labelformat=empty}
\caption{(c) $k = 15$}
\end{minipage}
\addtocounter{figure}{-1}
\begin{minipage}[b]{.18\textwidth}
\includegraphics[width=\textwidth]{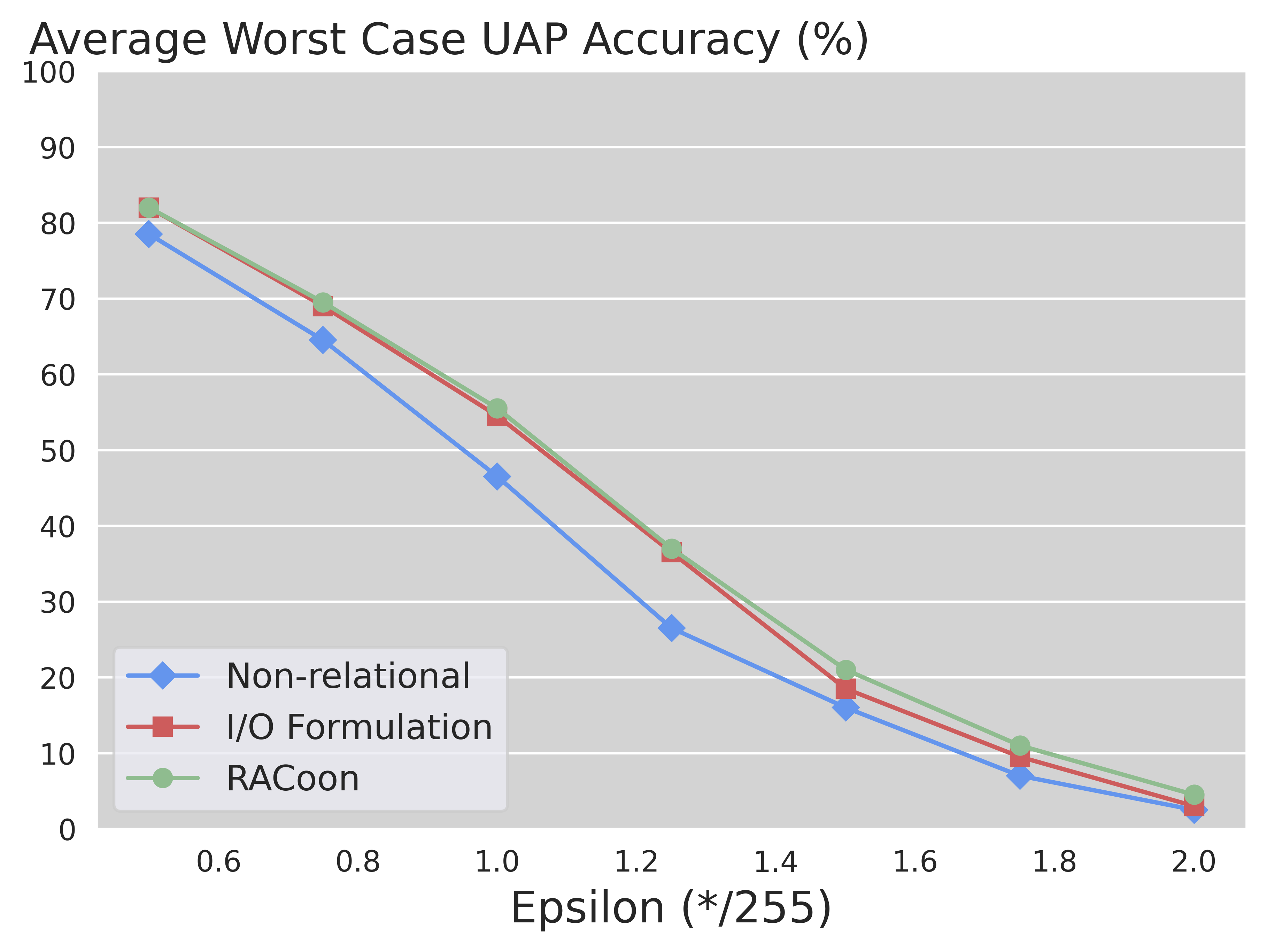}
\captionsetup{labelformat=empty}
\caption{(d) $k = 20$}
\end{minipage}
\addtocounter{figure}{-1}
\begin{minipage}[b]{.18\textwidth}
\includegraphics[width=\textwidth]{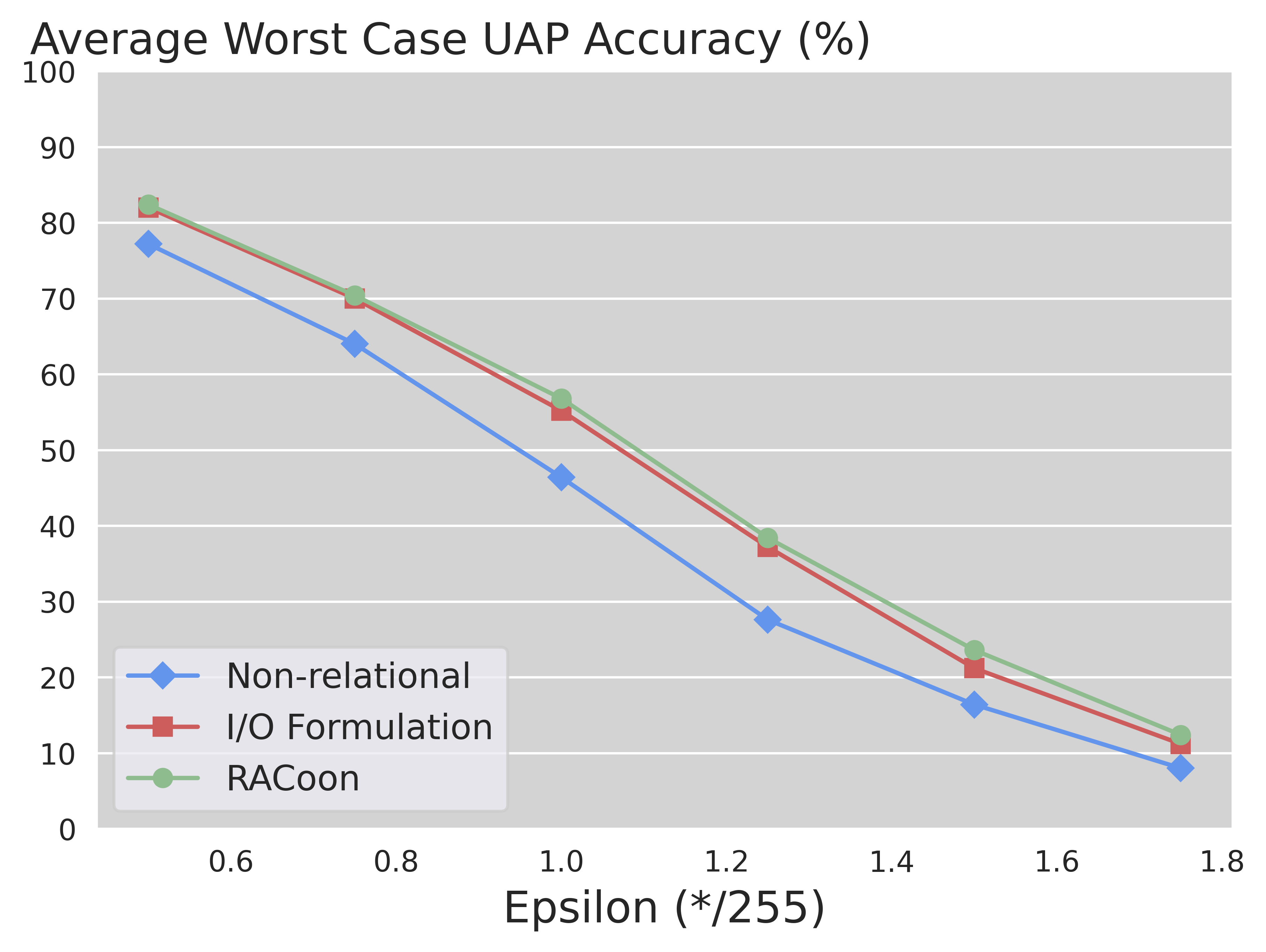}
\captionsetup{labelformat=empty}
\caption{(e) $k = 25$}
\end{minipage}
\addtocounter{figure}{-1}
\caption{Average worst-case UAP accuracy for different $k$ and $\epsilon$ values for ConvSmall Standard CIFAR10 network.}
\label{fig:cifa_point_diff}
\end{figure}

\begin{figure}[htb]
\centering
\begin{minipage}[b]{.18\textwidth}
\includegraphics[width=\textwidth]{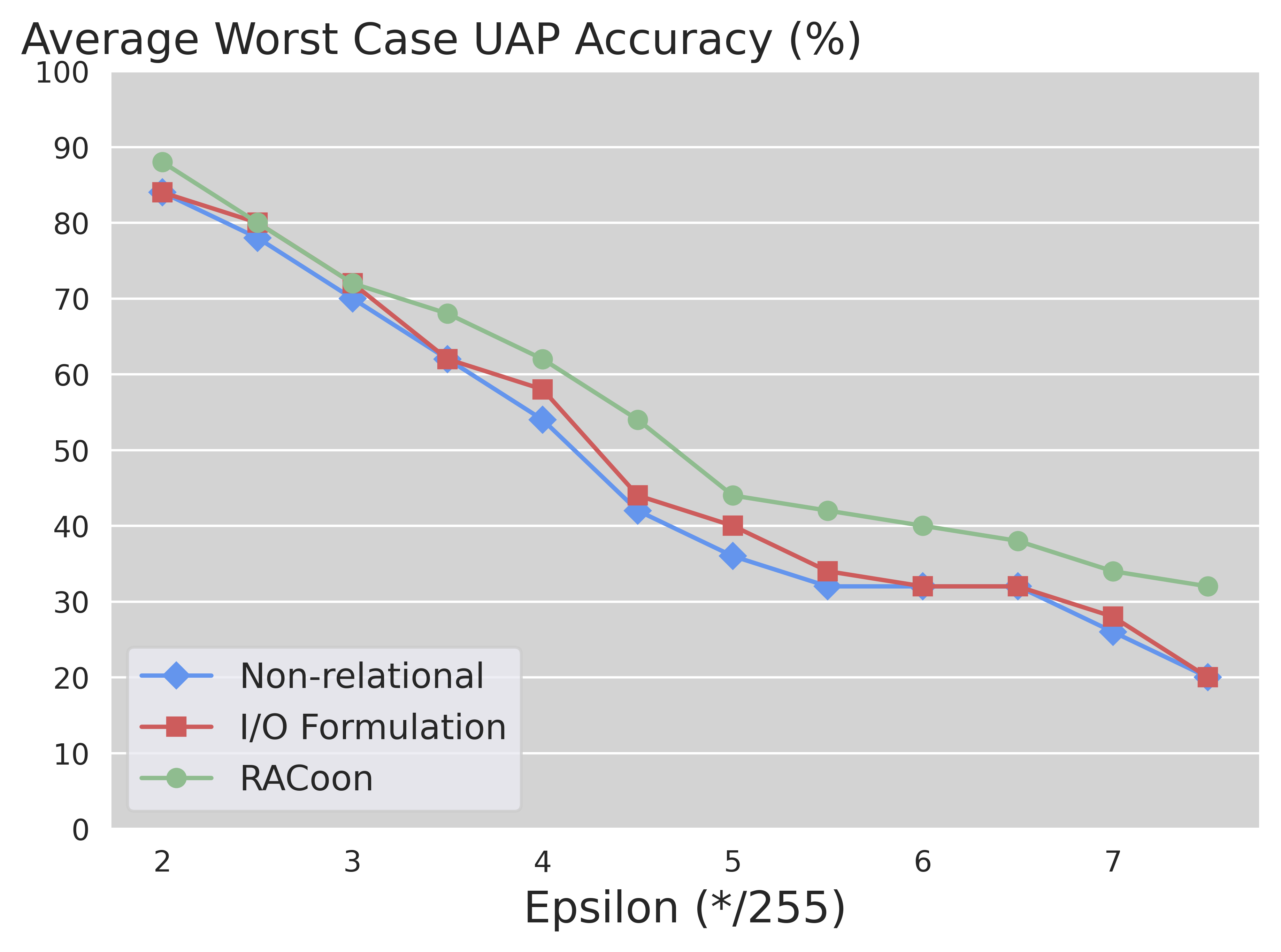}
\captionsetup{labelformat=empty}
\caption{(a) $k = 5$}
\end{minipage}
\addtocounter{figure}{-1}
\begin{minipage}[b]{.18\textwidth}
\includegraphics[width=\textwidth]{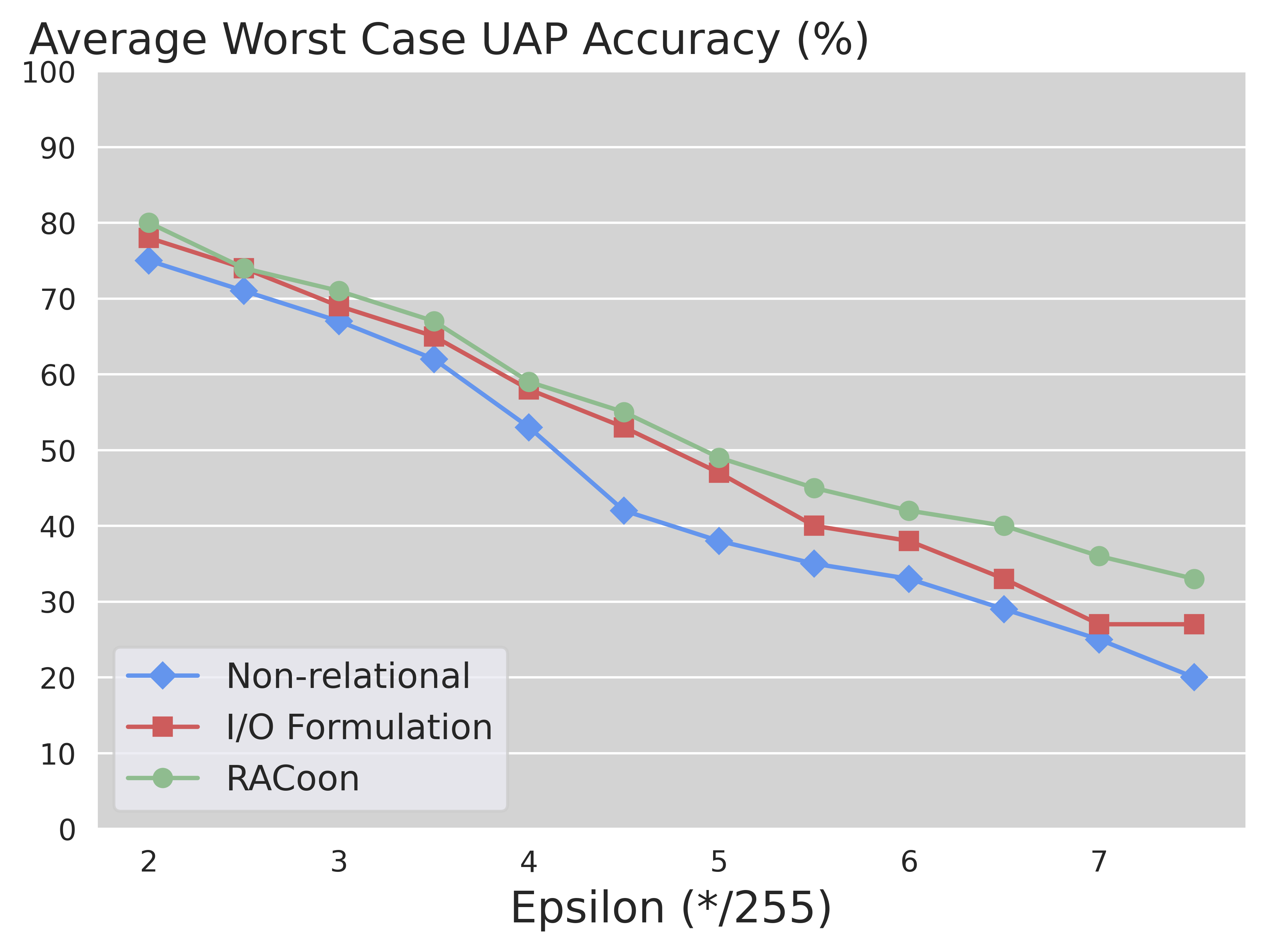}
\captionsetup{labelformat=empty}
\caption{(b) $k = 10$}
\end{minipage}
\addtocounter{figure}{-1}
\begin{minipage}[b]{.18\textwidth}
\includegraphics[width=\textwidth]{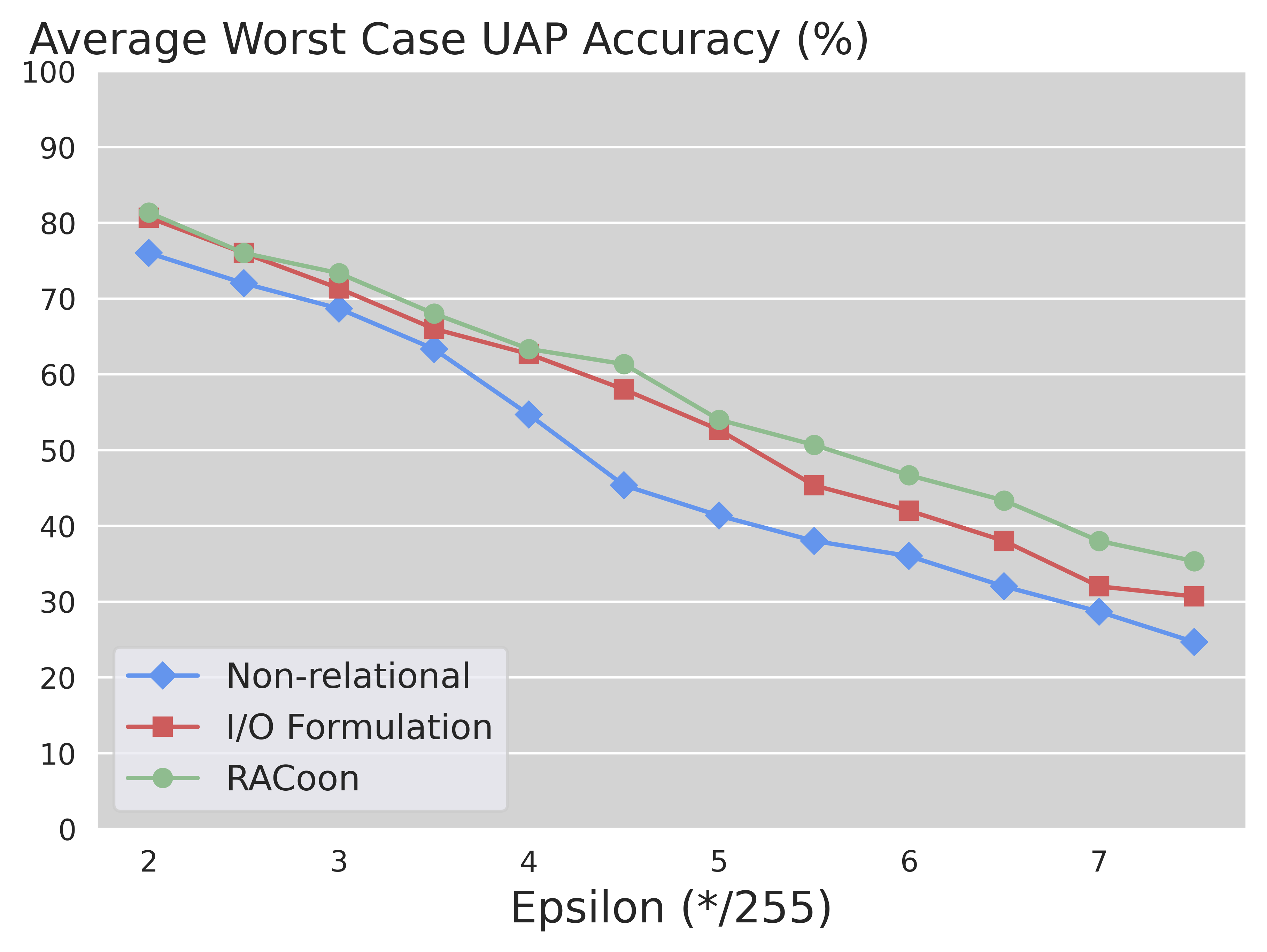}
\captionsetup{labelformat=empty}
\caption{(c) $k = 15$}
\end{minipage}
\addtocounter{figure}{-1}
\begin{minipage}[b]{.18\textwidth}
\includegraphics[width=\textwidth]{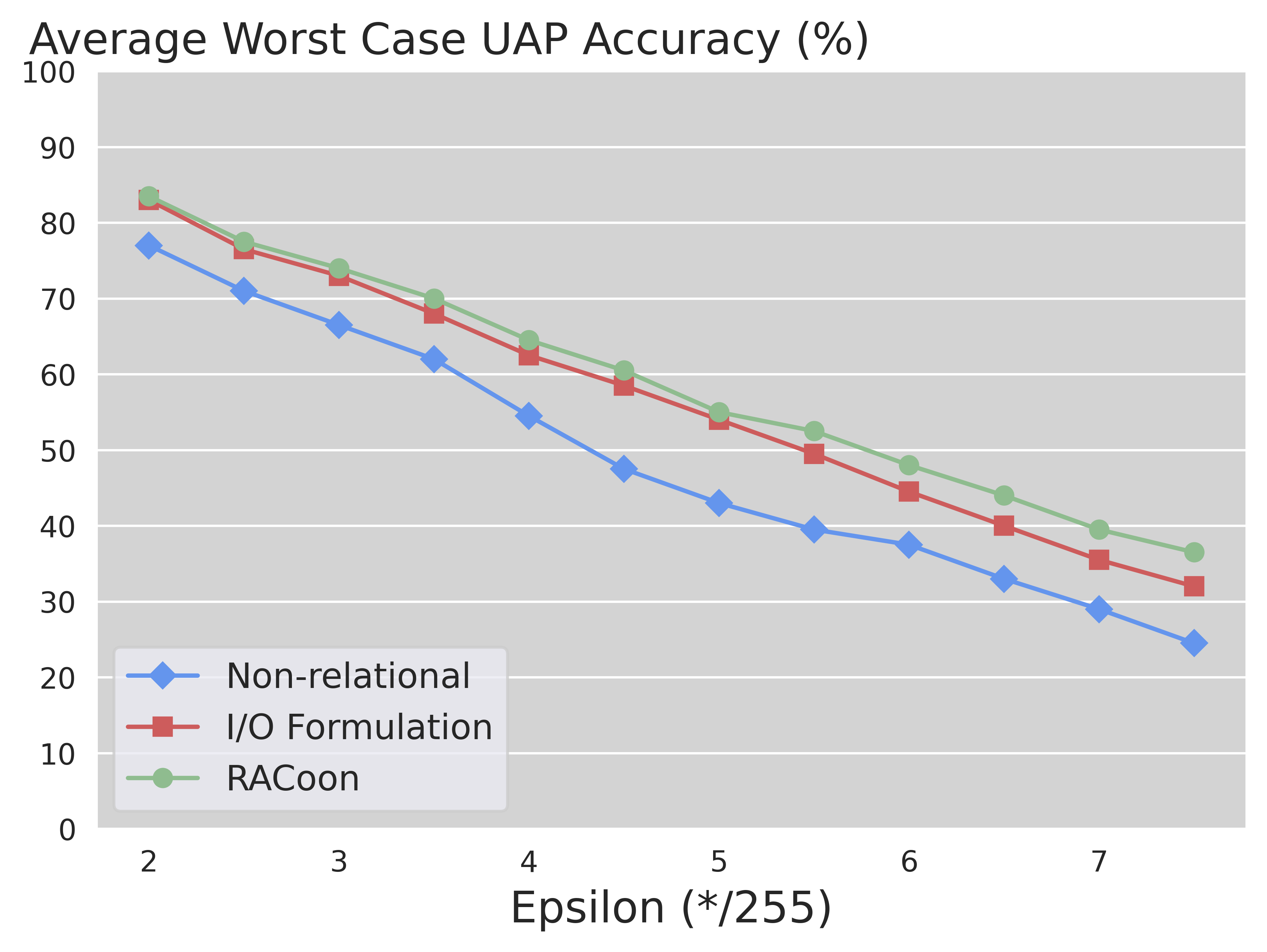}
\captionsetup{labelformat=empty}
\caption{(d) $k = 20$}
\end{minipage}
\addtocounter{figure}{-1}
\begin{minipage}[b]{.18\textwidth}
\includegraphics[width=\textwidth]{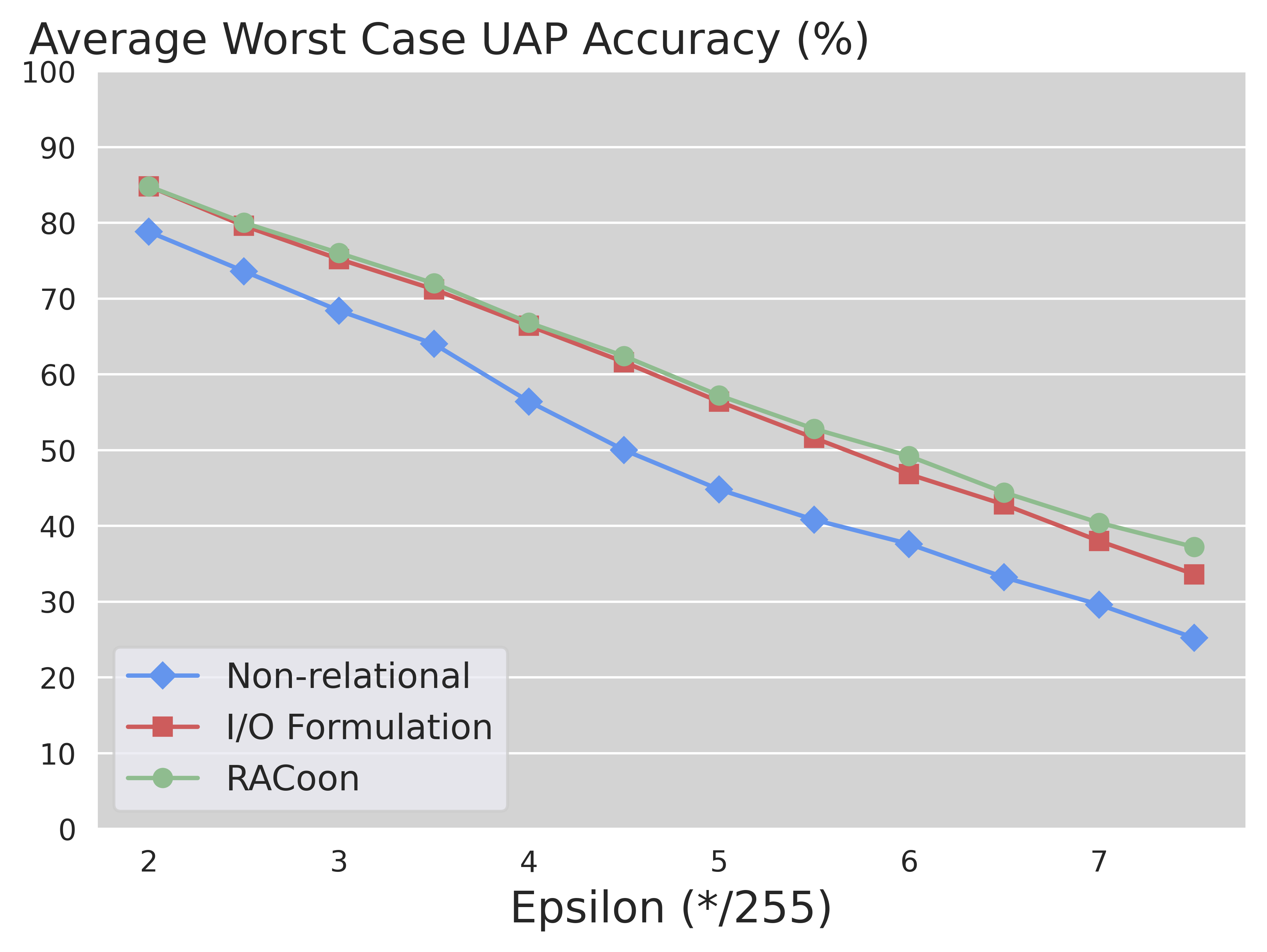}
\captionsetup{labelformat=empty}
\caption{(e) $k = 25$}
\end{minipage}
\addtocounter{figure}{-1}
\caption{Average worst-case UAP accuracy for different $k$ and $\epsilon$ values for ConvSmall DiffAI CIFAR10 network.}
\label{fig:cifar_pgd_diff}
\end{figure}

\begin{figure}[htb]
\centering
\begin{minipage}[b]{.18\textwidth}
\includegraphics[width=\textwidth]{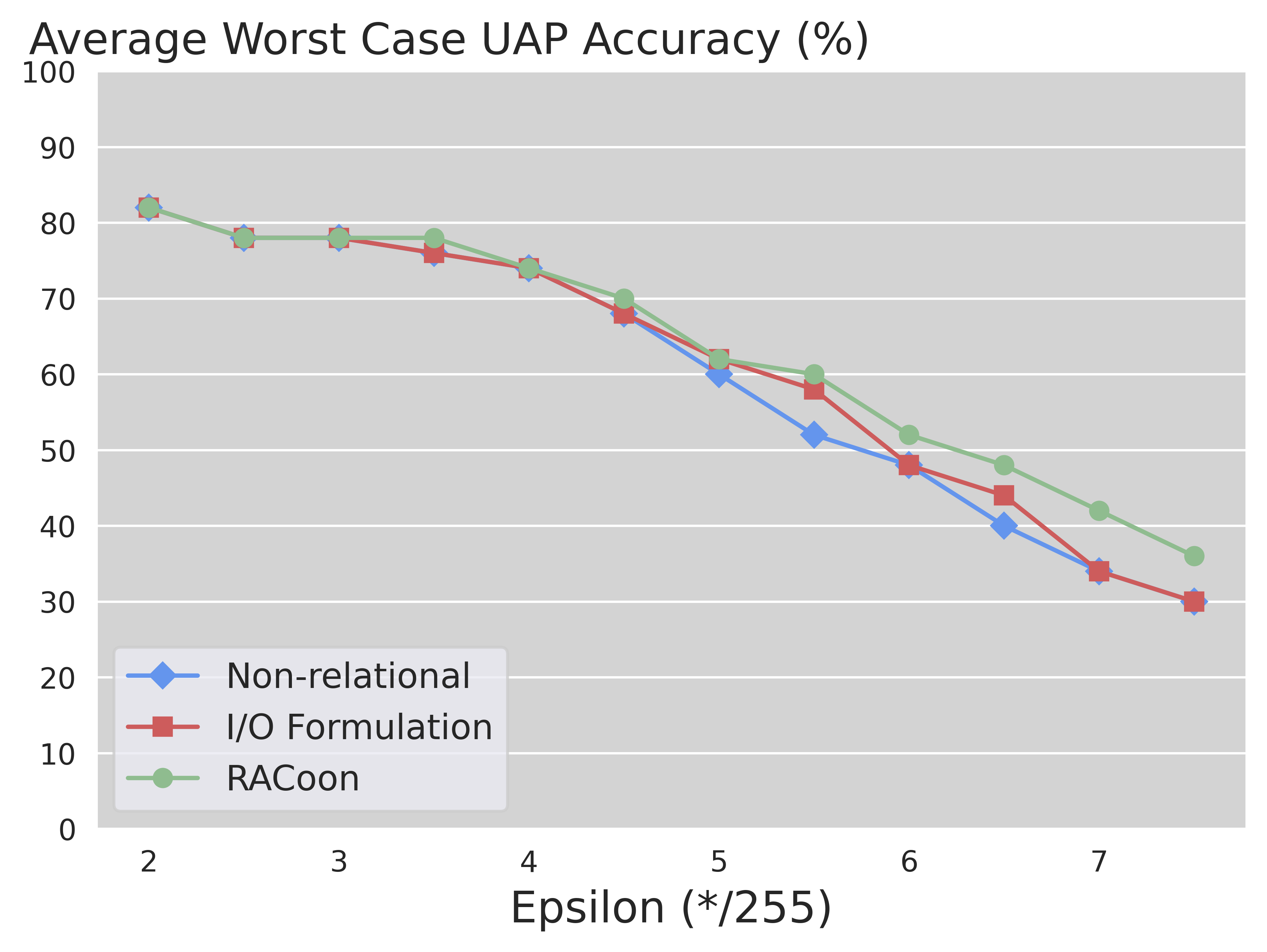}
\captionsetup{labelformat=empty}
\caption{(a) $k = 5$}
\end{minipage}
\addtocounter{figure}{-1}
\begin{minipage}[b]{.18\textwidth}
\includegraphics[width=\textwidth]{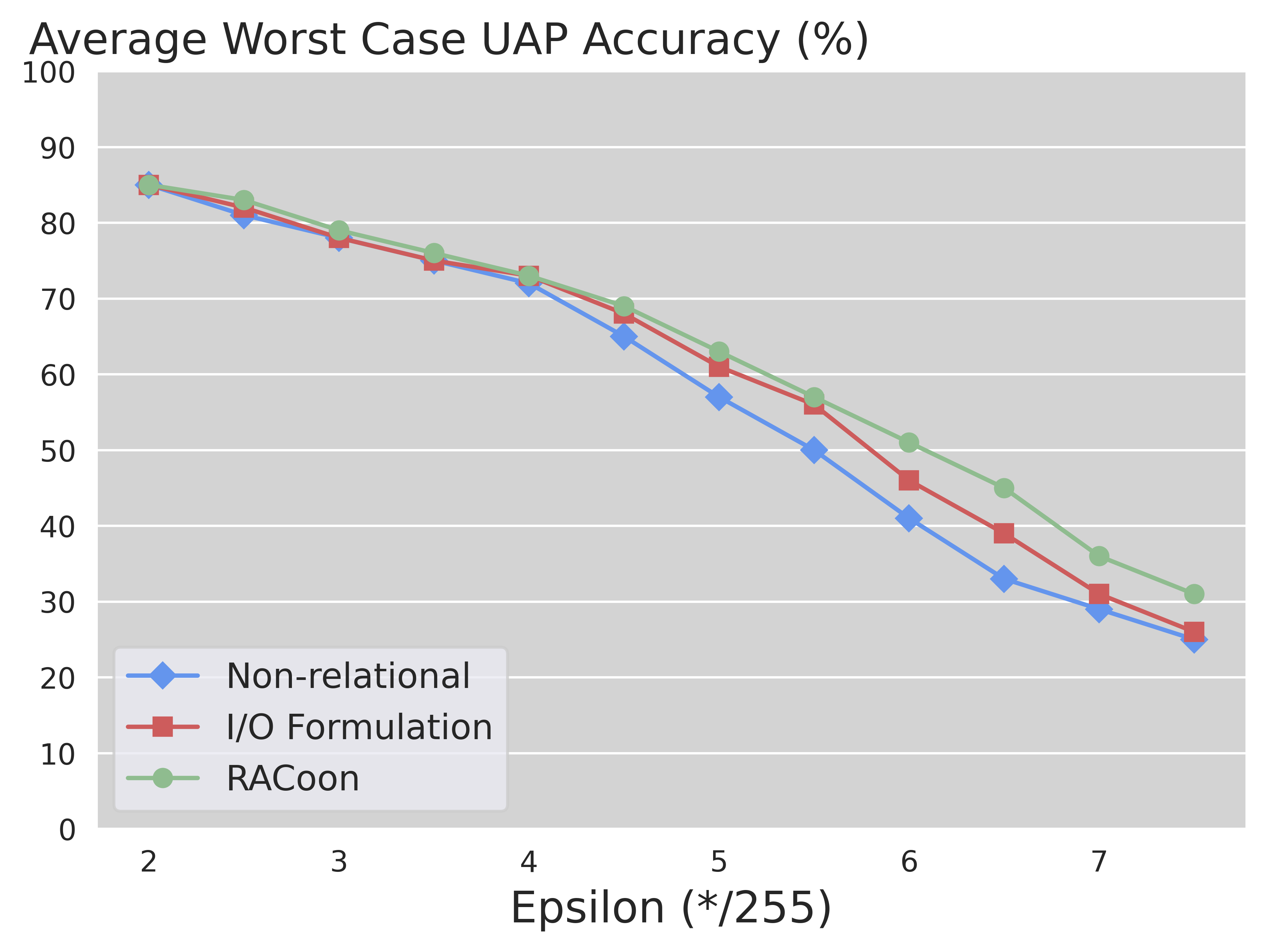}
\captionsetup{labelformat=empty}
\caption{(b) $k = 10$}
\end{minipage}
\addtocounter{figure}{-1}
\begin{minipage}[b]{.18\textwidth}
\includegraphics[width=\textwidth]{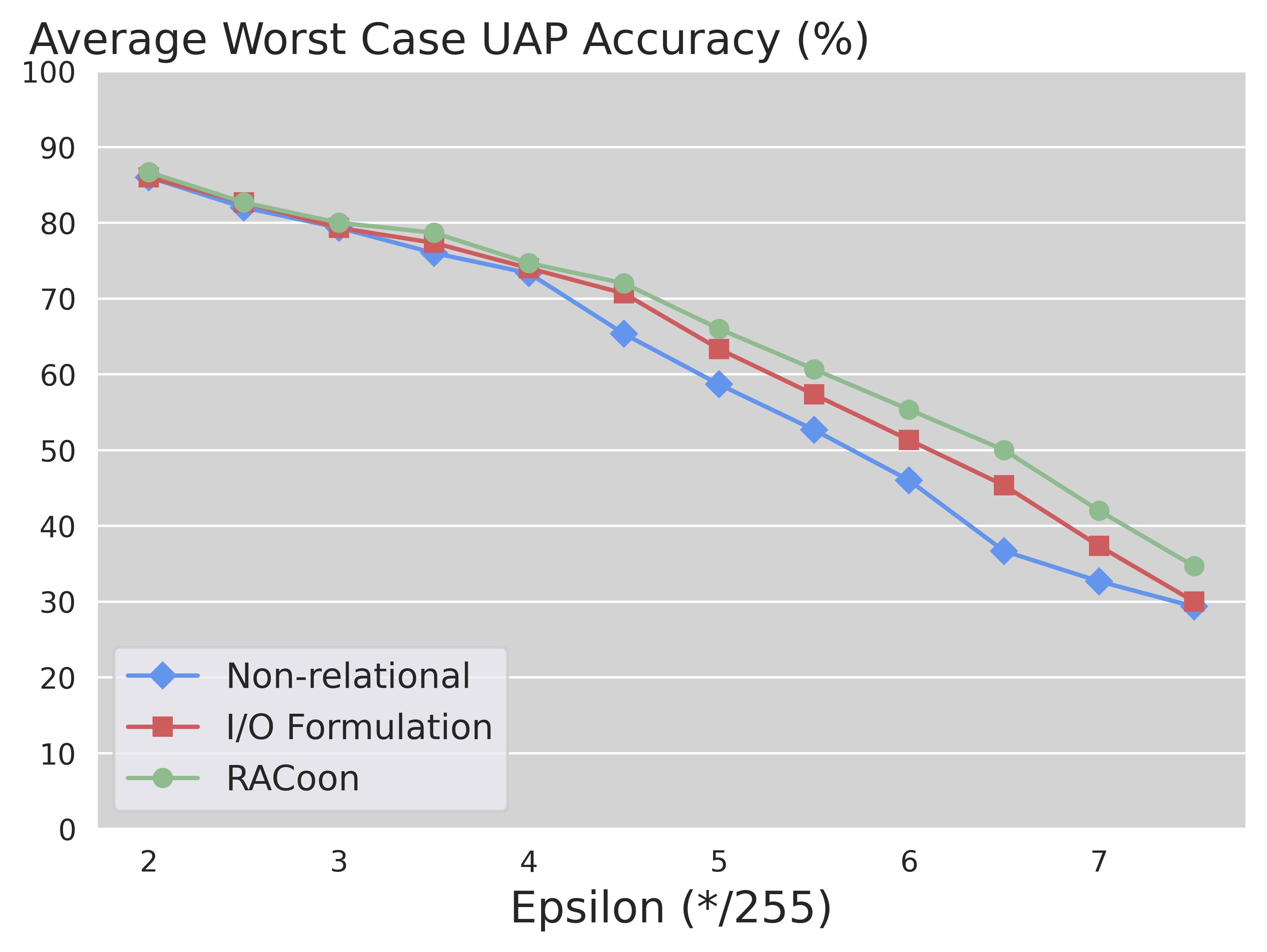}
\captionsetup{labelformat=empty}
\caption{(c) $k = 15$}
\end{minipage}
\addtocounter{figure}{-1}
\begin{minipage}[b]{.18\textwidth}
\includegraphics[width=\textwidth]{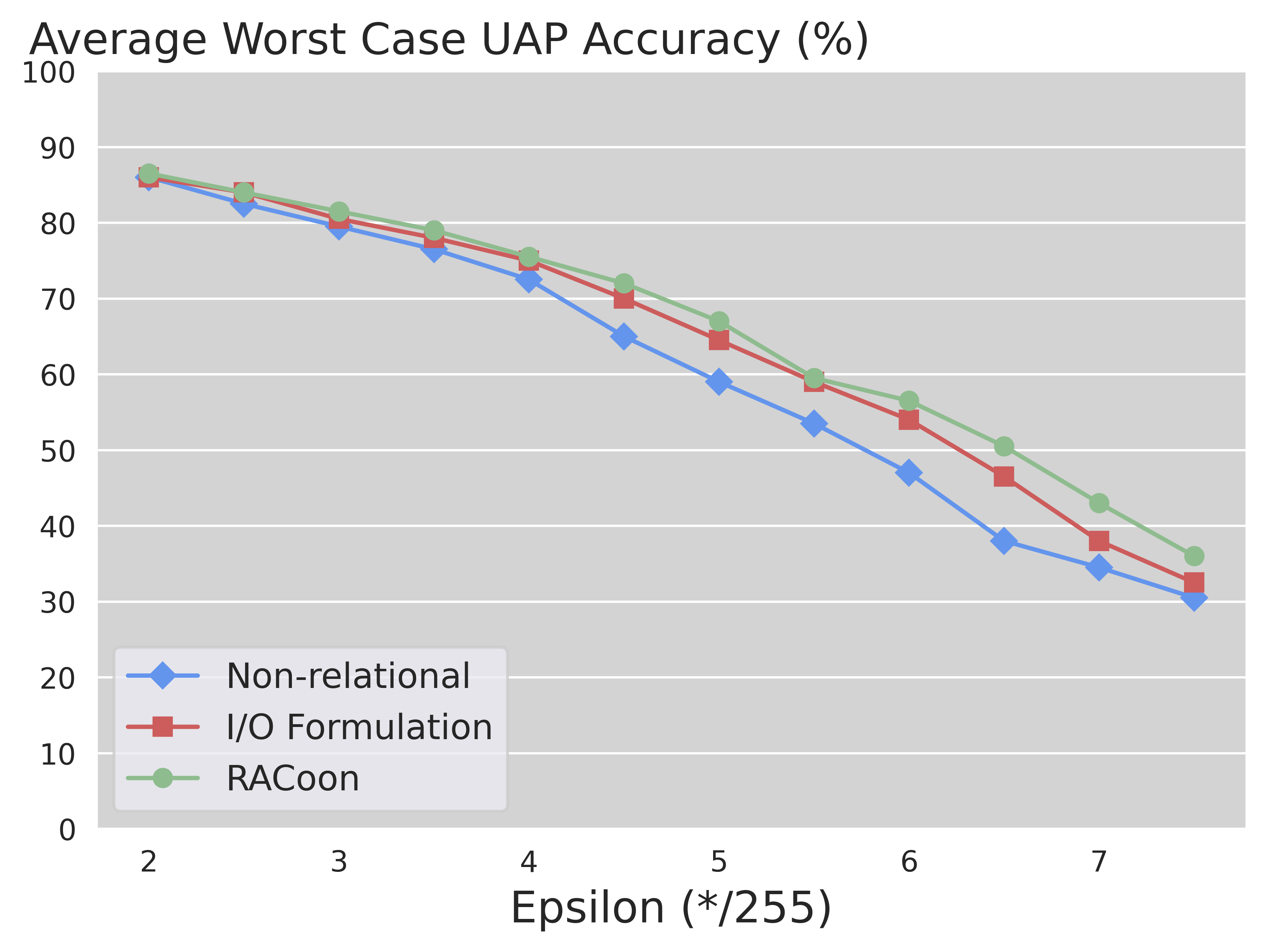}
\captionsetup{labelformat=empty}
\caption{(d) $k = 20$}
\end{minipage}
\addtocounter{figure}{-1}
\begin{minipage}[b]{.18\textwidth}
\includegraphics[width=\textwidth]{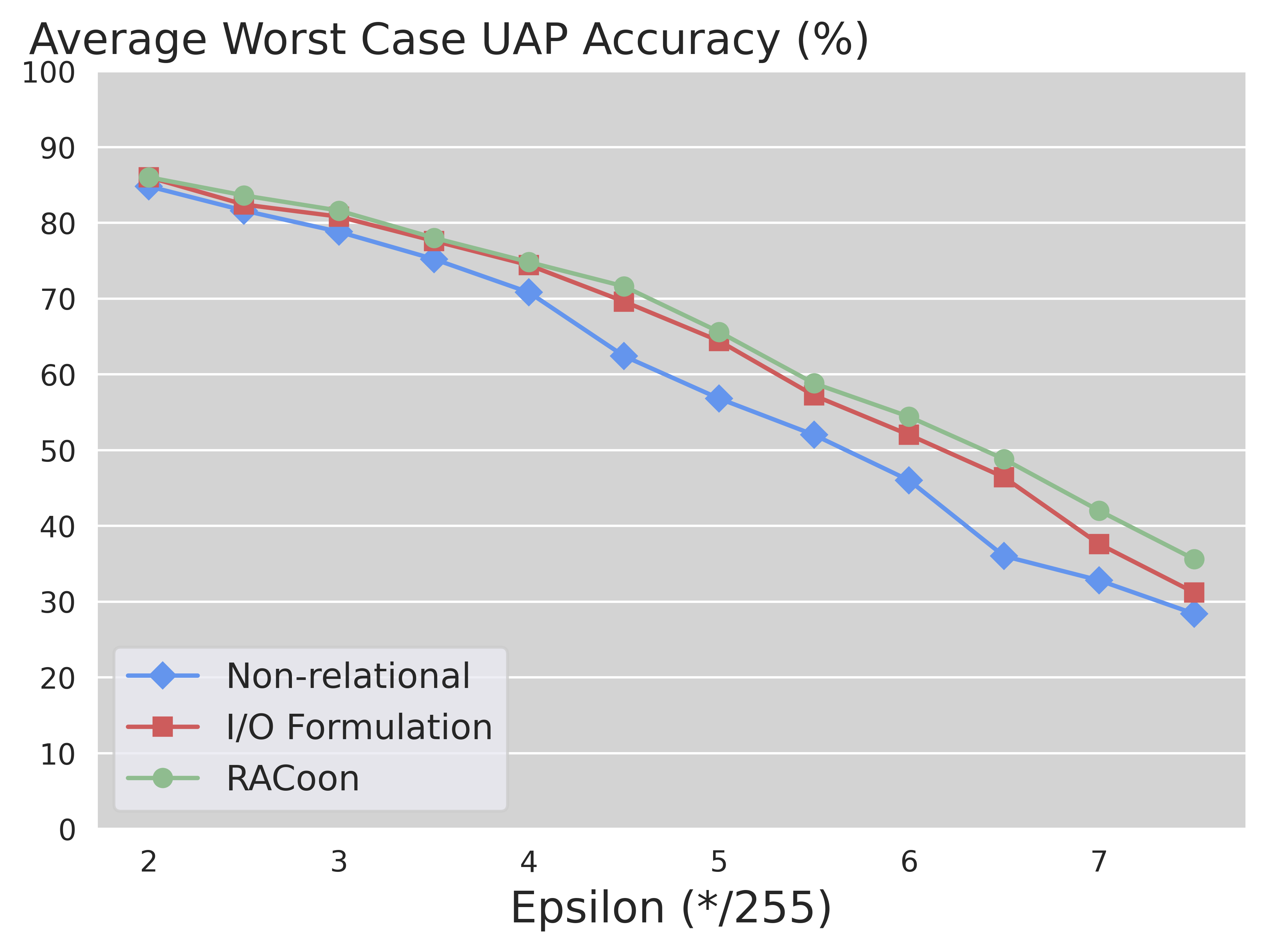}
\captionsetup{labelformat=empty}
\caption{(e) $k = 25$}
\end{minipage}
\addtocounter{figure}{-1}
\caption{Average worst-case UAP accuracy for different $k$ and $\epsilon$ values for ConvSmall COLT CIFAR10 network.}
\label{fig:cifar_colt_diff}
\end{figure}

\begin{figure}[htb]
\centering
\begin{minipage}[b]{.18\textwidth}
\includegraphics[width=\textwidth]{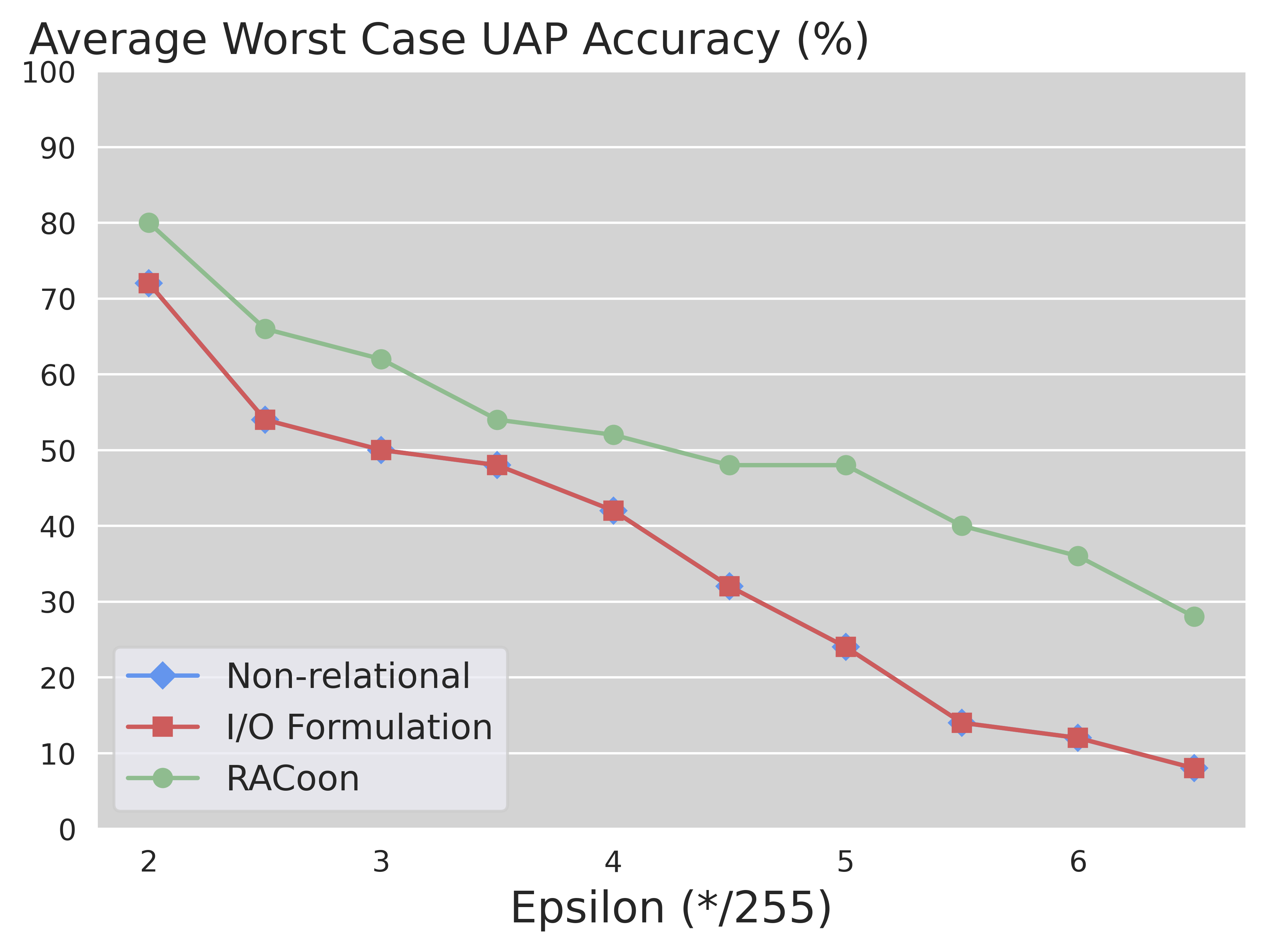}
\captionsetup{labelformat=empty}
\caption{(a) $k = 5$}
\end{minipage}
\addtocounter{figure}{-1}
\begin{minipage}[b]{.18\textwidth}
\includegraphics[width=\textwidth]{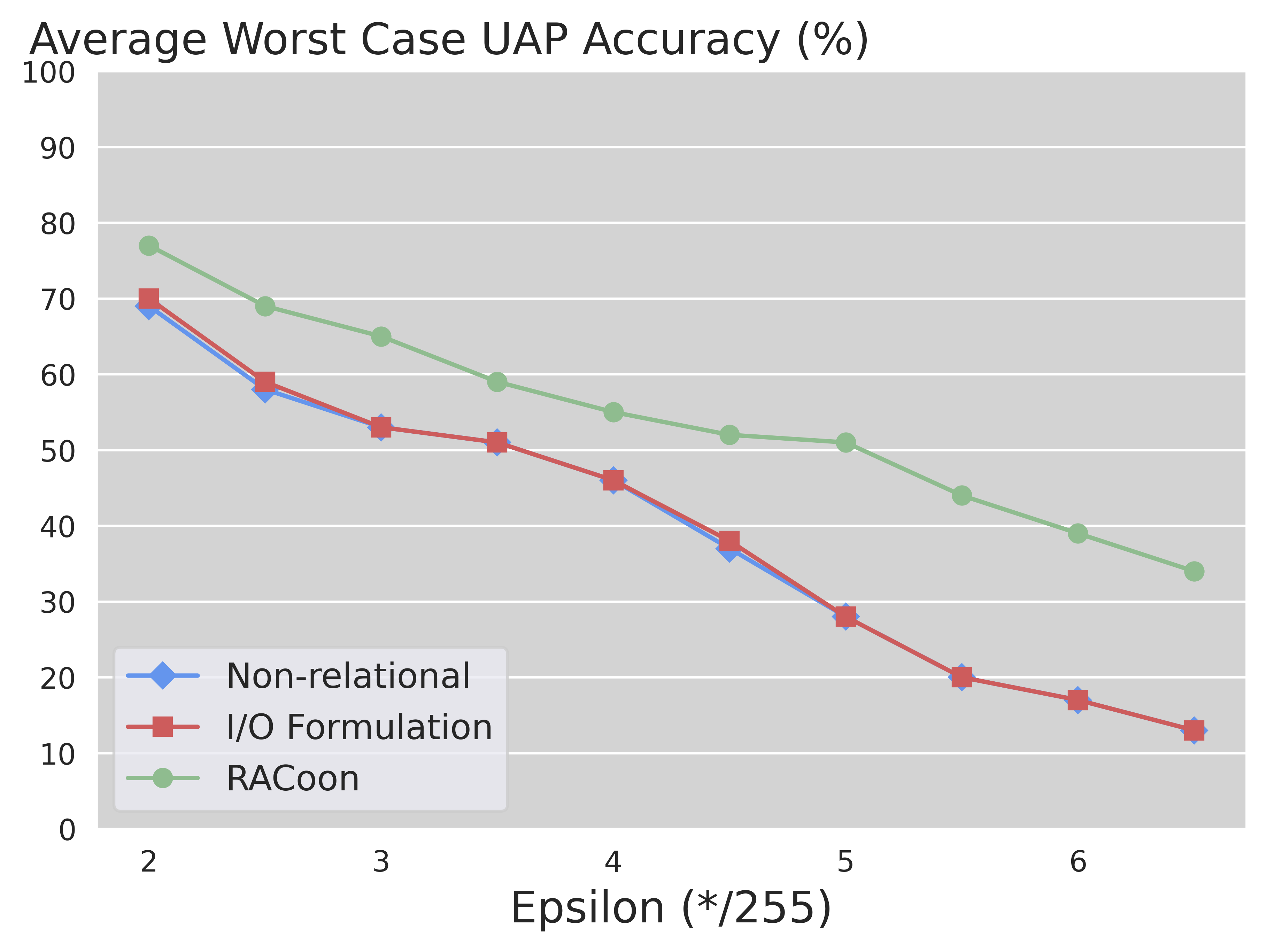}
\captionsetup{labelformat=empty}
\caption{(b) $k = 10$}
\end{minipage}
\addtocounter{figure}{-1}
\begin{minipage}[b]{.18\textwidth}
\includegraphics[width=\textwidth]{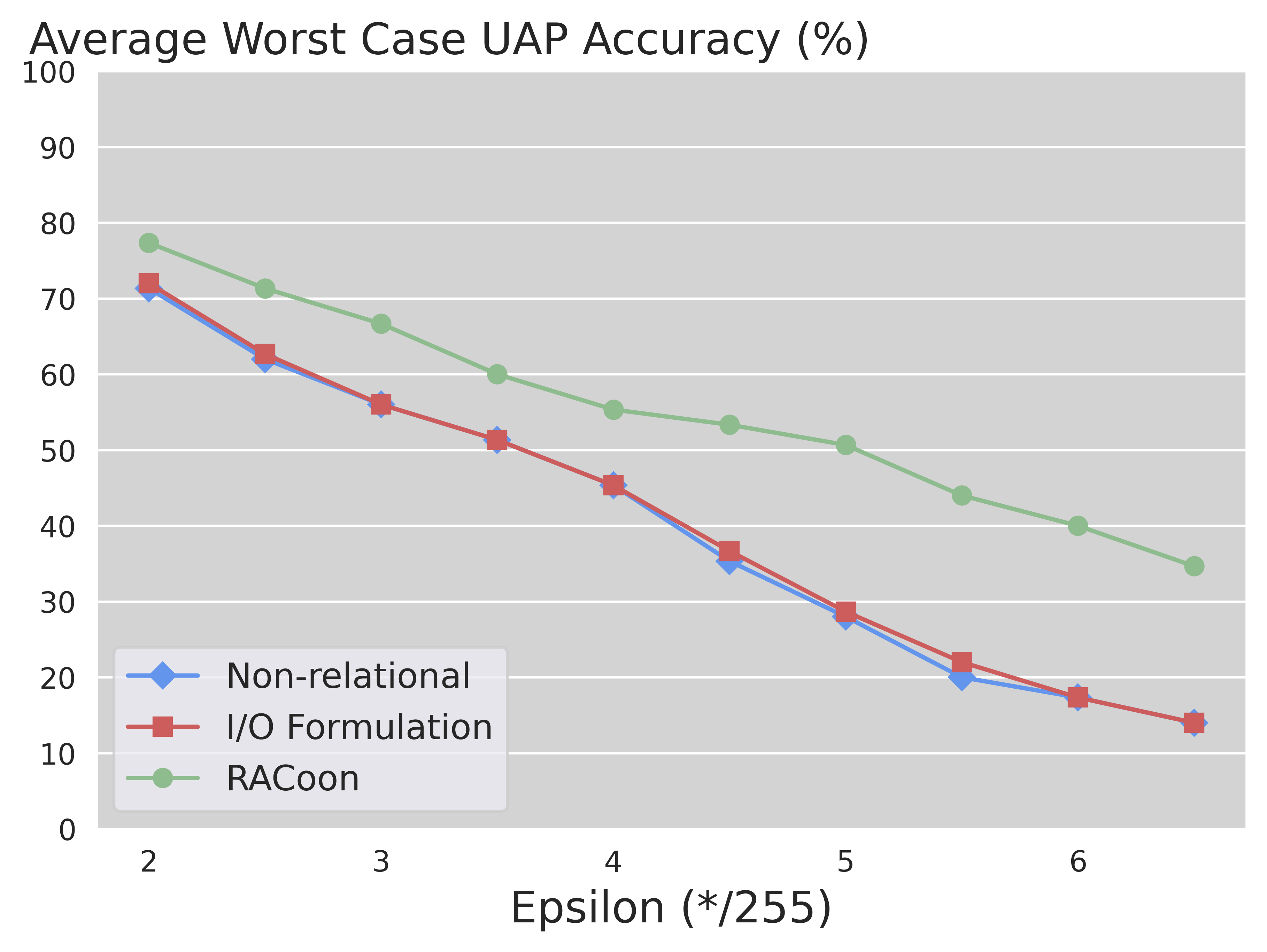}
\captionsetup{labelformat=empty}
\caption{(c) $k = 15$}
\end{minipage}
\addtocounter{figure}{-1}
\begin{minipage}[b]{.18\textwidth}
\includegraphics[width=\textwidth]{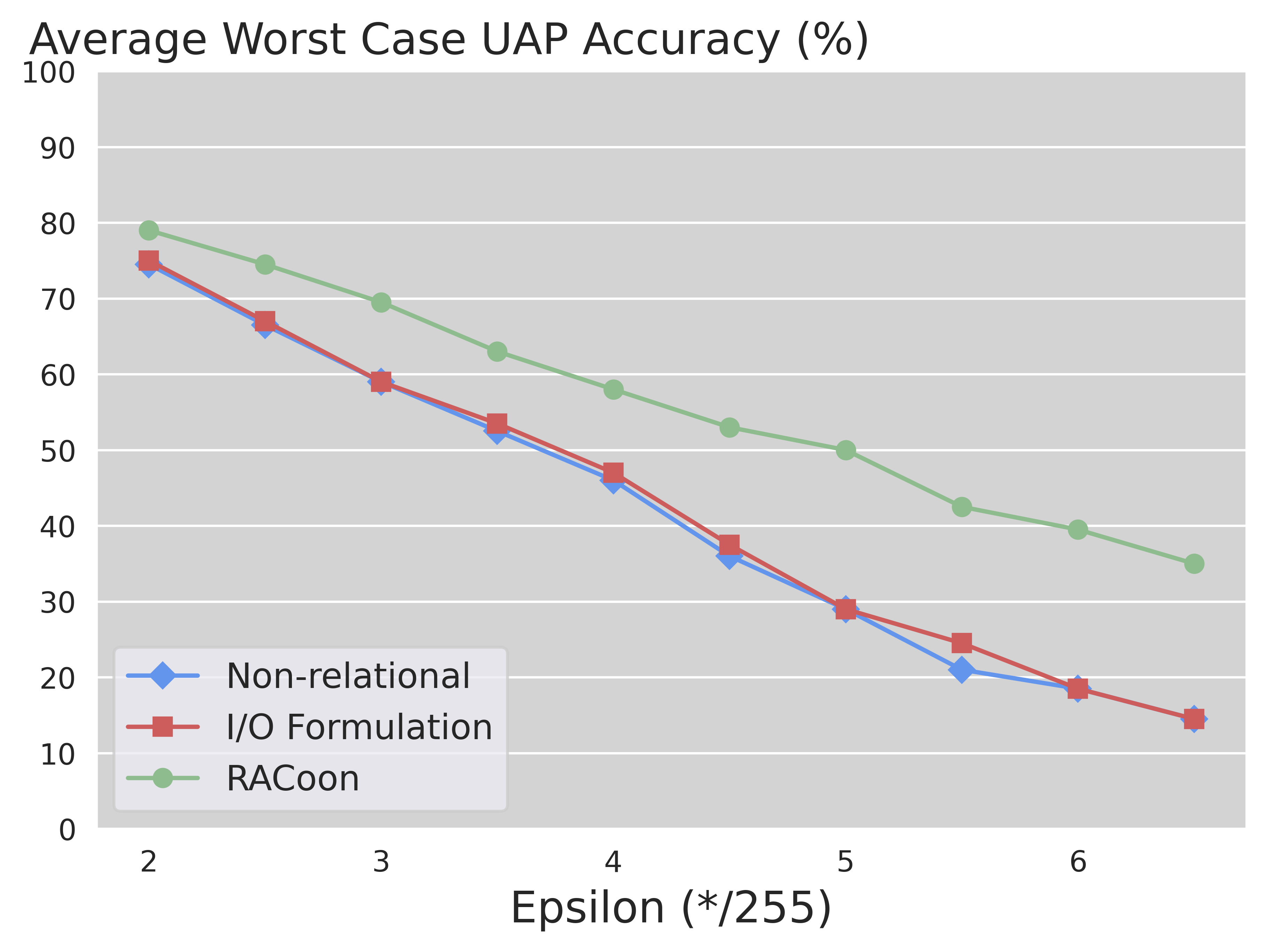}
\captionsetup{labelformat=empty}
\caption{(d) $k = 20$}
\end{minipage}
\addtocounter{figure}{-1}
\begin{minipage}[b]{.18\textwidth}
\includegraphics[width=\textwidth]{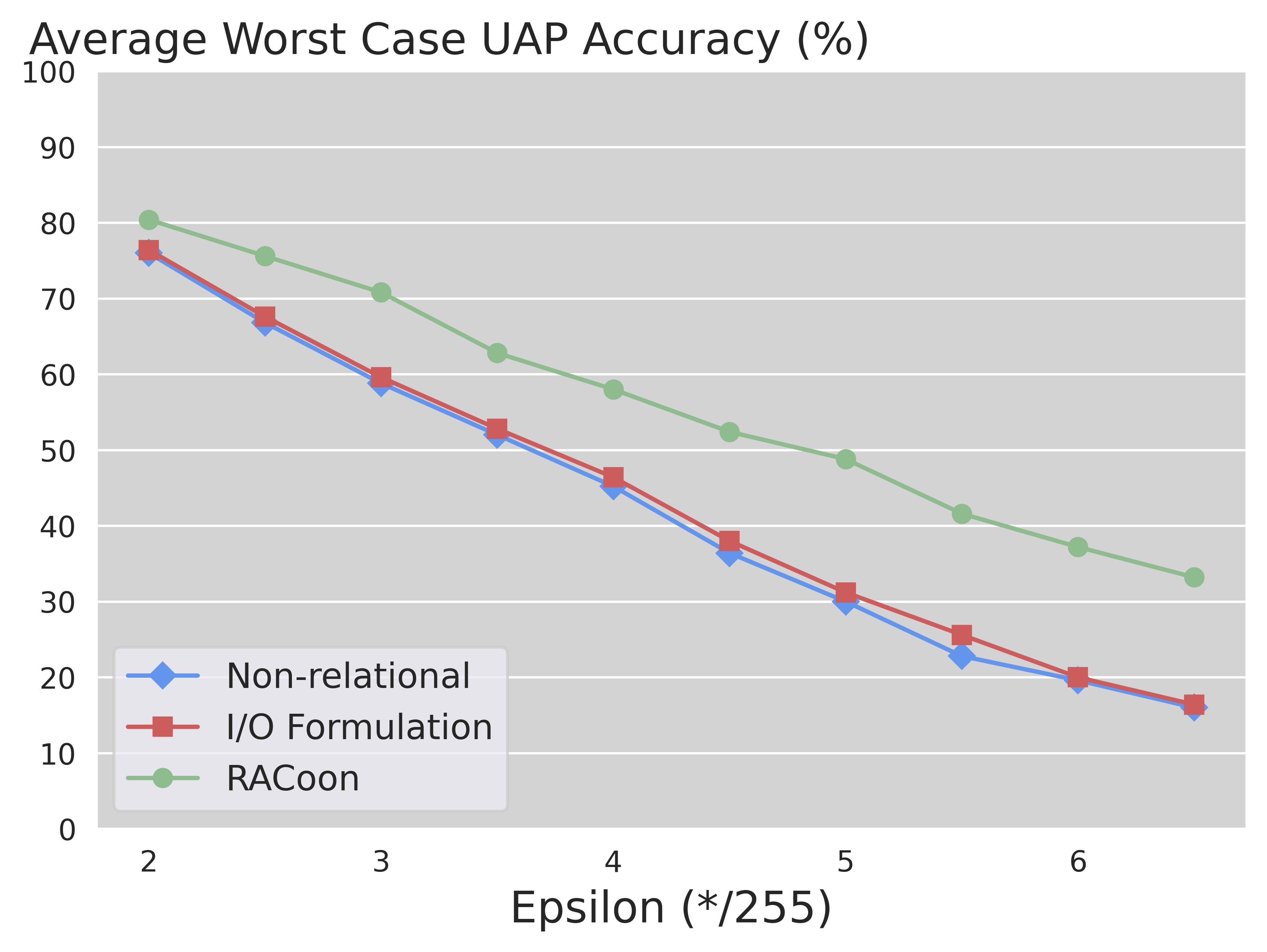}
\captionsetup{labelformat=empty}
\caption{(e) $k = 25$}
\end{minipage}
\addtocounter{figure}{-1}
\caption{Average worst-case UAP accuracy for different $k$ and $\epsilon$ values for IBPSmall CIFAR10 network.}
\label{fig:cifar_ibp_small_diff}
\end{figure}

\begin{table*}[htb]
\centering
\captionsetup{justification=centering}
\caption{\Tool Componentwise Runtime Analysis for $k=50$ for MNIST and $k=25$ for CIFAR10 networks}
\label{table:compareComponentsRuntimeLarger}
\resizebox{0.98\textwidth}{!}{
\begin{tabular}{@{}c c c c c c c c c c@{}}
\toprule
Dataset & Network & Training &  Perturbation & Non-relational  & I/O & Individual & Individual & Cross-Execution & \Tool  \\
\text{ } & Structure & Method &  Bound ($\epsilon$) & Verifier & Formulation & Refinement & Refinement with MILP & Refinement & verifier \\
\text{ } & \text{ } & \text{ } &  \text{ } & Avg. Time (sec.) & Avg. Time (sec.) & Avg. Time (sec.) & Avg. Time (sec.) & Avg. Time (sec.) & Avg. Time (sec.) \\
\midrule
\text{ } & ConvSmall & Standard & 0.08 & 0.01 & 31.76 \textcolor{red}{*}  & 0.40 & 24.11 & 2.25 & 4.16 \\
\text{ } & ConvSmall & PGD & 0.10 & 0.02 & 5.13 & 0.50 & 5.36 & 3.00 & 4.59 \\
\text{ } & IBPSmall & IBP & 0.13 & 0.01 & 3.62 & 0.39 & 3.38 & 1.53 & 2.27 \\
MNIST & ConvSmall & DiffAI & 0.13 & 0.02 & 18.32 \textcolor{red}{*} & 0.59 & 9.89 & 3.34 & 5.38 \\
\text{ } & ConvSmall & COLT & 0.15 & 0.04 & 2.53 & 0.43 & 3.41 & 1.07 & 1.64 \\
\text{ } & ConvBig & DiffAI & 0.20 & 2.14 & 5.23 & 9.70 & 11.07 & 9.36 & 14.30 \\
\midrule
\text{ } & ConvSmall & Standard & 1.0/255 & 0.04 & 91.74 \textcolor{red}{*}  & 0.86 & 101.17 & 4.54 & 19.39 \\
\text{ } & ConvSmall & PGD & 3.0/255 & 0.03 & 8.28 & 0.73 & 12.53 & 4.24 & 11.47 \\
\text{ } & IBPSmall & IBP & 6.0/255 & 0.02 & 2.76 & 0.56 & 3.92 & 3.32 & 6.76 \\
CIFAR10 & ConvSmall & DiffAI & 7.0/255 & 0.01 & 12.22 \textcolor{red}{*}  & 0.46 & 13.62 & 2.69 & 9.68 \\
\text{ } & ConvSmall & COLT & 7.0/255 & 0.07 & 15.74 & 0.95 & 25.82 & 5.20 & 24.54 \\
\text{ } & ConvBig & DiffAI & 3.0/255 & 2.01 & 13.42 & 7.27 & 22.89 & 16.45 & 21.92 \\
\midrule
\end{tabular}
}
\resizebox{0.98\textwidth}{!}{
\textcolor{red}{*} I/O formulation does not filter out executions verified by the non-relational verifier making MILP optimization for large $k$ expensive.

}
\end{table*}

\clearpage
\newpage
\section{Ablation study of the hyperparameters $k_0$ and $k_1$ on different MNIST networks}
\label{appenexp:ablation_with_k0_k1}

\begin{table*}[htb]
\centering
\captionsetup{justification=centering}
\caption{\Tool Average worst-case UAP accuracy on ConvSmall Standard MNIST net with different $k_0$ and $k_1$ on $\epsilon = 0.1$.}
\label{table:ablationAccStandard}
\resizebox{0.4\textwidth}{!}{
\begin{tabular}{@{}c| c c c c c c@{}}
\midrule
\diagbox[width=15mm, height=10mm]{ $k_1$ }{$k_0$} & 2 & 3 & 4 & 5 & 6 & 7  \\ 
\midrule
2  & 26.50 & 28.50 & 31.50 & 32.00 & 33.00 & 34.50 \\ 
3 & $\times$ & 28.50 & 31.50 & 32.00 & 33.00 & 34.50 \\ 
4 & $\times$& $\times$ & 32.00 & 33.00 & 33.00 & 34.50\\
\midrule

\end{tabular}
}
\end{table*}

\begin{table*}[htb]
\centering
\captionsetup{justification=centering}
\caption{\Tool Average runtime on ConvSmall Standard MNIST net with different $k_0$ and $k_1$ on $\epsilon = 0.1$.}
\label{table:ablationTimeStandard}
\resizebox{0.4\textwidth}{!}{
\begin{tabular}{@{}c | c c c c c c@{}}
\midrule
\diagbox[width=15mm, height=10mm]{ $k_1$ }{$k_0$} & 2 & 3 & 4 & 5 & 6 & 7  \\ 
\midrule
2  & 1.35 & 1.45 & 2.35 & 2.75 & 3.71 & 5.06 \\ 
3 & $\times$ & 1.53 & 2.12 & 3.79 & 4.78 & 8.68 \\ 
4 & $\times$& $\times$ & 2.16 & 3.94 & 6.09 & 8.15 \\ 
\midrule

\end{tabular}
}
\end{table*}

\begin{table*}[htb]
\centering
\captionsetup{justification=centering}
\caption{\Tool Average worst-case UAP accuracy on ConvSmall PGD MNIST net with different $k_0$ and $k_1$ on $\epsilon = 0.1$.}
\label{table:ablationAccPGD}
\resizebox{0.4\textwidth}{!}{
\begin{tabular}{@{}c | c c c c c c@{}}
\midrule
\diagbox[width=15mm, height=10mm]{ $k_1$ }{$k_0$} & 2 & 3 & 4 & 5 & 6 & 7  \\ 
\midrule
2  & 54.50 & 56.50 & 58.00 & 59.50 & 60.50 & 61.00  \\ 
3 & $\times$ & 56.50 & 58.50 & 60.00 & 61.50 & 62.00  \\ 
4 & $\times$& $\times$ & 58.50 & 60.00 & 61.50 & 62.00  \\
\midrule
\end{tabular}
}
\end{table*}

\begin{table*}[htb]
\centering
\captionsetup{justification=centering}
\caption{\Tool Average runtime on ConvSmall PGD MNIST net with different $k_0$ and $k_1$ on $\epsilon = 0.1$.}
\label{table:ablationTimePGD}
\resizebox{0.4\textwidth}{!}{
\begin{tabular}{@{}c | c c c c c c@{}}
\midrule
\diagbox[width=15mm, height=10mm]{ $k_1$ }{$k_0$} & 2 & 3 & 4 & 5 & 6 & 7  \\ 
\midrule
2  & 1.03 & 1.20 & 1.67 & 2.20 & 2.83 & 4.32 \\ 
3 & $\times$ & 1.44 & 2.24 & 3.71 & 5.13 & 7.39 \\ 
4 & $\times$& $\times$ & 2.22 & 4.28 & 6.67 & 8.41 \\ 
\midrule

\end{tabular}
}
\end{table*}

\begin{table*}[htb]
\centering
\captionsetup{justification=centering}
\caption{\Tool Average worst-case UAP accuracy on ConvSmall DiffAI MNIST net with different $k_0$ and $k_1$ on $\epsilon = 0.12$.}
\label{table:ablationAccDiffAI}
\resizebox{0.4\textwidth}{!}{
\begin{tabular}{@{}c | c c c c c c@{}}
\midrule
\diagbox[width=15mm, height=10mm]{ $k_1$ }{$k_0$} & 2 & 3 & 4 & 5 & 6 & 7  \\ 
\midrule
2  & 83.50 & 84.50 & 85.00 & 85.00 & 85.00 & 85.00 \\ 
3 & $\times$ & 84.50 & 85.00 & 85.00 & 85.00 & 85.00 \\ 
4 & $\times$& $\times$ & 85.00 & 85.00 & 85.00 & 85.00 \\ 
\midrule
\end{tabular}
}
\end{table*}

\begin{table*}[htb]
\centering
\captionsetup{justification=centering}
\caption{\Tool Average runtime on ConvSmall DiffAI MNIST net with different $k_0$ and $k_1$ on $\epsilon = 0.12$.}
\label{table:ablationTimeDiffAI}
\resizebox{0.4\textwidth}{!}{
\begin{tabular}{@{}c | c c c c c c@{}}
\midrule
\diagbox[width=15mm, height=10mm]{ $k_1$ }{$k_0$} & 2 & 3 & 4 & 5 & 6 & 7  \\ 
\midrule
2  & 0.95 & 1.29 & 2.27 & 2.70 & 1.90 & 1.89 \\ 
3 & $\times$ & 1.29 & 2.79 & 3.55 & 2.89 & 2.70 \\ 
4 & $\times$& $\times$ & 3.07 & 3.25 & 2.93 & 3.11 \\  
\midrule

\end{tabular}
}
\end{table*}

\begin{table*}[htb]
\centering
\captionsetup{justification=centering}
\caption{\Tool Average worst-case UAP accuracy on IBPSmall MNIST net with different $k_0$ and $k_1$ on $\epsilon = 0.15$.}
\label{table:ablationAccIBP}
\resizebox{0.4\textwidth}{!}{
\begin{tabular}{@{}c | c c c c c c@{}}
\midrule
\diagbox[width=15mm, height=10mm]{ $k_1$ }{$k_0$} & 2 & 3 & 4 & 5 & 6 & 7  \\ 
\midrule
2  & 55.50 & 59.00 & 63.00 & 65.50 & 68.00 & 69.50 \\ 
3 & $\times$ & 59.00 & 63.00 & 65.50 & 68.00 & 69.50 \\ 
4 & $\times$& $\times$ & 64.00 & 66.50 & 68.00 & 69.50 \\ 
\midrule
\end{tabular}
}
\end{table*}

\begin{table*}[htb]
\centering
\captionsetup{justification=centering}
\caption{\Tool Average runtime on IBPSmall MNIST net with different $k_0$ and $k_1$ on $\epsilon = 0.15$.}
\label{table:ablationTimeIBP}
\resizebox{0.4\textwidth}{!}{
\begin{tabular}{@{}c | c c c c c c@{}}
\midrule
\diagbox[width=15mm, height=10mm]{ $k_1$ }{$k_0$} & 2 & 3 & 4 & 5 & 6 & 7  \\ 
\midrule
2  & 0.91 & 0.95 & 0.91 & 0.98 & 1.57 & 1.90 \\ 
3 & $\times$ & 0.89 & 0.87 & 0.95 & 1.25 & 1.76 \\ 
4 & $\times$& $\times$ & 0.94 & 1.16 & 1.50 & 1.67 \\ 
\midrule

\end{tabular}
}
\end{table*}
\end{document}